\documentclass[journal]{IEEEtran}
 
\def\thisismainpaper{0}   

\usepackage[utf8]{inputenc}
\usepackage{mathtools}
\usepackage{amssymb}
\usepackage{graphicx}

\usepackage[rightcaption]{sidecap}
\usepackage{graphicx} 
\graphicspath{ {images/} }

\usepackage[rightcaption]{sidecap}
\usepackage{graphicx}
\graphicspath{ {images/} }
\usepackage{amsmath}  
\usepackage[linesnumbered,ruled,vlined]{algorithm2e}
\usepackage[noend]{algorithmic}
\usepackage[]{algorithm2e}
\usepackage{algorithmic}
\usepackage{mathtools}

\usepackage{amsmath}
\usepackage{tabulary}
\usepackage[utf8]{inputenc}
\usepackage[english]{babel}
\usepackage{amsthm}
\usepackage{comment}
\usepackage{array}
\usepackage{subfigure}
\usepackage{graphicx}
\usepackage{authblk}
\usepackage{textcomp}
\usepackage{setspace}
\usepackage{cleveref}
\newtheorem{theorem}{Theorem}[section] 
\newtheorem{corollary}{Corollary}[theorem]
\newtheorem{lemma}{Lemma}[section]
\newtheorem{definition}{Definition}[section]
\makeatletter
\newcommand{\nosemic}{\renewcommand{\@endalgocfline}{\relax}}
\newcommand{\dosemic}{\renewcommand{\@endalgocfline}{\algocf@endline}}
\let\oldnl\nl
\newcommand{\nonl}{\renewcommand{\nl}{\let\nl\oldnl}}
\makeatother
\def\cost{\mbox{cost}}
\def\dist{\mbox{dist}}
\def\opt{\mbox{opt}}
\def\approx{\mbox{approx}}
\def\wmin{w_{\min}}
\def\argmax{\mathop{\arg\max}}
\def\argmin{\mathop{\arg\min}}
\def\dVC{d_{\mbox{\tiny VC}}}
\usepackage{color}

\newcommand{\rev}[1]{\textcolor{black}{#1}}

\newcommand{\appendixnumbering}[1] {\setcounter{definition}{0}
\setcounter{lemma}{0}
\setcounter{theorem}{0}
\setcounter{corollary}{0}
\setcounter{subsection}{0}
\renewcommand{\thedefinition}{#1.\arabic{definition}}
\renewcommand{\thelemma}{#1.\arabic{lemma}}
\renewcommand{\thetheorem}{#1.\arabic{theorem}}
\renewcommand{\thecorollary}{#1.\arabic{corollary}}
\renewcommand{\thesubsection}{#1.\arabic{subsection}.}
\renewcommand{\thesubsectiondis}{#1.\arabic{subsection}.}}

\allowdisplaybreaks

\begin{document}

\IEEEoverridecommandlockouts
\title{Robust Coreset Construction for Distributed Machine Learning}


\author[Hanlin Lu et al. ]
       {Hanlin Lu, \emph{Student Member, IEEE,} Ming-Ju Li, Ting He, \emph{Senior Member, IEEE,}
       Shiqiang Wang, \emph{Member, IEEE}, Vijaykrishnan Narayanan, \emph{Fellow, IEEE,} and Kevin S Chan, \emph{Senior Member, IEEE}
       \thanks{\scriptsize This paper has been accepted by IEEE Journal on Selected Areas in Communications. }
       \thanks{\scriptsize H. Lu, M. Li, T. He, and V. Narayanan are with Pennsylvania State University, University Park, PA 16802 USA (email: \{hzl263, mxl592, tzh58, vxn9\}@psu.edu). \\
       \indent S. Wang is with IBM T. J. Watson Center, Yorktown Heights, NY 10598 USA (email: wangshiq@us.ibm.com).\\
       \indent K. Chan is with Army Research Laboratory, Adelphi, MD 20783 USA (email: kevin.s.chan.civ@mail.mil). }
       \thanks{\scriptsize This research was partly sponsored by the U.S. Army Research Laboratory and the U.K. Ministry of Defence under Agreement Number W911NF-16-3-0001. Narayanan and Hanlin were partly supported by NSF 1317560. The views and conclusions contained in this document are those of the authors and should not be interpreted as representing the official policies, either expressed or implied, of the U.S. Army Research Laboratory, the U.S. Government, the U.K. Ministry of Defence or the U.K. Government. The U.S. and U.K. Governments are authorized to reproduce and distribute reprints for Government purposes notwithstanding any copyright notation hereon.\looseness=-1 }
       \vspace{-0.2in}
       }
       
\maketitle

\begin{abstract}
\rev{Coreset, which is a summary of the original dataset in the form of a small weighted set in the same sample space, provides a promising approach to enable machine learning over distributed data. Although viewed as a proxy of the original dataset, each coreset is only designed to approximate the cost function of a specific machine learning problem, and thus different coresets are often required to solve different machine learning problems, increasing the communication overhead. 
We resolve this dilemma by developing robust coreset construction algorithms that can support a variety of machine learning problems. Motivated by empirical evidence that suitably-weighted $k$-clustering centers provide a robust coreset, we harden the observation by establishing theoretical conditions under which the coreset provides a guaranteed approximation for a broad range of machine learning problems, and developing both centralized and distributed algorithms to generate coresets satisfying the conditions. The robustness of the proposed algorithms is verified through extensive experiments on diverse datasets with respect to both supervised and unsupervised learning problems.  
}
\end{abstract}


\section{ Introduction}

\rev{Sensor-driven distributed intelligent systems are becoming ubiquitous in a variety of applications such as precision agriculture, machine health monitoring, environmental tracking, traffic management, and infrastructure security. While the surge in distributed data generation powered by various sensors 
has enabled novel features based on machine learning techniques, 
there are many challenges towards collecting such distributed data. Various resource and application constraints make it challenging to gather voluminous data from distributed data sources, including limitation on network connectivity and bandwidth, limitation on power consumption, and the need to preserve the privacy of raw data.  Consequently, there is an increasing need for techniques to efficiently apply machine learning on distributed datasets.}

\rev{The current distributed machine learning approaches can be broadly classified as follows: (1) those that globally aggregate the outputs of local models; (2) those that construct global models from individual models derived from local data, and (3) those that share representative local data with a global aggregator. An example of the first approach involves independent computations at individual nodes and sharing the outputs of local models~\cite{Peteiro-Barral12PAI}. These independent outputs are aggregated using methods such as majority voting at a global aggregator. In contrast to the first approach, 
the second approach shares the models created from local data 
\cite{McMahan16AISTATS,Wang18INFOCOM,konevcny2016federated}. The individual models are combined to create a global model using techniques such as weighted average. The third approach \cite{Balcan13NIPS,Kannan14COLT,Barger16SDM}, which is the focus of this work, is to share summaries of local data towards the creation of a shared global model.  Our focus on this approach is driven by its promise to train multiple models using the same data summary, amortizing the communication overhead between the edge nodes and the global aggregator across machine learning problems. }

\begin{figure}[t!]
   \centerline{\includegraphics[width=.8\linewidth]{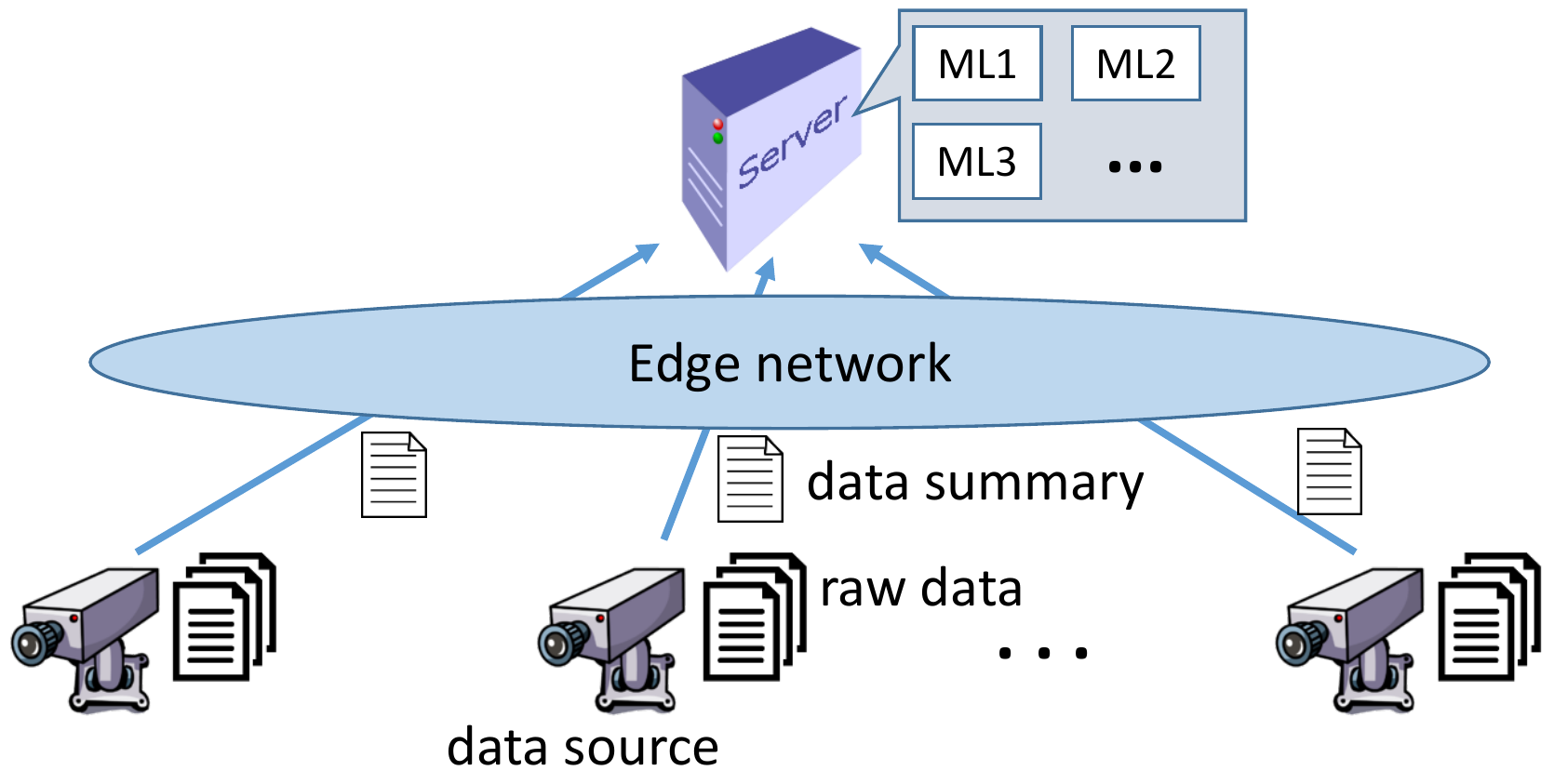}}
     \vspace{-.5em}
    \caption{Application scenario (ML $i$: machine learning model $i$).}
    \label{fig:systemmodel}
    \vspace{-1em}
\end{figure}

\rev{We are particularly interested in the edge-based learning scenario illustrated in Figure~\ref{fig:systemmodel}~\cite{Wang18INFOCOM}, where data sources report local summaries to an edge server, which then computes various machine learning models from these summaries. Specifically, our approach utilizes a \emph{coreset} for data summarization~\cite{Badoiu02STOC}. Creation of a coreset involves identifying a smaller weighted dataset that is representative of the more voluminous original dataset, while guaranteeing provable approximations. While there exist other data summarization approaches such as sketches, coresets are more convenient to use due to their ability to preserve the sample space of the original dataset. Although various algorithms have been developed to construct coresets with guaranteed approximation to the original dataset in terms of training machine learning models (see Section~\ref{subsec:Coreset Construction Algorithms}), existing coreset construction algorithms are tailor-made, which means that we have to collect different coresets to solve different problems. Our work attempts to identify whether a single coreset can be used to amortize the communication overhead over a broad set of machine learning problems.}\looseness=-1

\rev{In our preliminary study \cite{lu19GLOBECOM}, we observed, in a centralized setting with a single data source, that a particular type of coreset generated by $k$-means/median clustering gave a good approximation in training simple machine learning models. 
However, in practice, data are often distributed across multiple nodes, requiring a distributed way to construct the coreset with a low communication overhead. 
In addition, many modern machine learning models are based on neural networks, which were not studied in \cite{lu19GLOBECOM}. 
In this work, we deepen the study by including neural networks into the models of interest and extending the coreset construction algorithm to the distributed setting. 
We show that by carefully allocating a given coreset size among different data sources as well as between local centers and randomized samples, the proposed algorithm can automatically tune to the dataset at hand and thus achieves robustness to both the machine learning problem on top and the underlying data distribution. Through extensive experiments on diverse datasets, we verify the effectiveness of the proposed algorithms in supporting the learning of both unsupervised and supervised models 
with up to thousands of parameters, while only transferring a tiny fraction of the raw data. \looseness=-1
}


\subsection{Related Work}
Distributed learning is considered one of the most promising lines of research for large-scale learning \cite{Peteiro-Barral12PAI}, particularly for naturally distributed data. The main challenge in distributed learning is to incorporate information from each distributed dataset, without the high overhead of collecting all the data.

Traditionally, this is achieved by collecting the outputs of learned models or the models themselves \cite{Chan93AAAI}. The first approach (i.e., collecting outputs) is more popular among earlier works. For example, \cite{Kittler98IEEE} proposed various heuristic decision rules (e.g., majority vote) to combine outputs of local classifiers, and \cite{Wolpert92NN} proposed to train a global classifier using labeled outputs of local classifiers. The solution in \cite{Wolpert92NN} was modified in \cite{Tsoumakas02ECAI} to improve efficiency for large-scale distributed data, and
extended in \cite{Chan93AAAI} to include various ways of composing the global training set.  The idea was later used to build a descriptive model from distributed data \cite{Guo02LNCS}. To further improve the accuracy, a distributed-pasting-votes framework was proposed in \cite{Chawla04JMLR} to learn sets of classifiers (ensembles).

The second approach (i.e., collecting models) is more useful when we want to learn not just one answer, but the rule to give answers. For example, the distributed boosting framework in \cite{Lazarevic02DPD} requires nodes to share locally trained classifiers, and the federated learning framework in \cite{McMahan16AISTATS,smith2017federated} requires nodes to report locally learned models to a single node, which then aggregates the models and broadcasts the result to others.

Meanwhile, research on data summarization has inspired a third approach: collecting data summaries. Data summaries, e.g., coresets, sketches, projections \cite{Phillips16CoRR,Munteanu18KI,wang2017sketching}, are derived datasets that are much smaller than the original dataset, and can hence be transferred to a central location with a low communication overhead. This approach has been adopted in recent works, e.g., \cite{Balcan13NIPS,Kannan14COLT,Barger16SDM,makarychev2019performance}. 
We are particularly interested in a specific type of data summary, \emph{coreset}, as it can be used as a proxy of the original dataset. See Section~\ref{subsec:Coreset Construction Algorithms} for a detailed review of  related works on coreset.  


\subsection{Summary of Contributions}
 

We are the first to explore using coreset to support diverse machine learning problems on distributed data. Specifically: 

\begin{enumerate}
\item \rev{We empirically show that although existing coreset construction algorithms are designed for specific machine learning problems, an algorithm based on $k$-means clustering yields good performance for different problems.} \looseness=-1
\item We harden the above observation by proving that the optimal $k$-clustering (including $k$-means/median) gives a coreset that provides a guaranteed approximation for any machine learning problem with a sufficiently continuous cost function (\Cref{thm:1}). We further prove that the same holds for the coreset given by a suboptimal $k$-clustering algorithm, as long as it satisfies certain assumptions (\Cref{thm:2}).
\item \rev{We adapt an existing algorithm designed to support distributed $k$-clustering to construct a robust coreset over distributed data with a very low communication overhead.}
\item Our evaluations on diverse machine learning problems \rev{and datasets} verify that $k$-clustering (especially $k$-means) \rev{and its distributed approximations} provide coresets good for learning a variety of machine learning models, \rev{including neural networks with thousands of parameters, at a tiny fraction of the communication overhead}.  \looseness=-1 
\end{enumerate}

\textbf{Roadmap.} Section~\ref{sec:Background} reviews the background on coreset. Section~\ref{sec:Robust Coreset Construction} presents our main theoretical results on the universal performance guarantee of $k$-clustering-based coreset. Section~\ref{sec:Application in Distributed Setting} presents our distributed coreset construction algorithm.  Section~\ref{sec:Performance Evaluation} evaluates the proposed algorithm. Section~\ref{sec:Conclusion} concludes the paper. 
\if\thisismainpaper1
Supporting proofs, analysis, and evaluations are provided in \cite{coreset19:report}.
\fi

\section{Background}\label{sec:Background}

\subsection{Coreset and Machine Learning}\label{subsec:Coreset and Machine Learning}

Many machine learning problems can be cast as a cost (or loss) minimization problem. Given a dataset in $d$-dimensional space $P\subseteq \mathbb{R}^d$, a generic machine learning problem over $P$ can be characterized by a {solution space} $\mathcal{X}$, a \emph{per-point cost function} $\cost(p,x)$ ($p\in P$, $x\in\mathcal{X}$), and an \emph{overall cost function} $\cost(P,x)$ ($x\in\mathcal{X}$) that aggregates the per-point costs over $P$. For generality, we consider $P$ to be a weighted set, where each $p\in P$ has weight $w_p$. Let $\wmin \coloneqq \min_{p\in P} w_p$ denote the minimum weight. For an unweighted dataset, we have $w_p\equiv 1$. The machine learning problem is then to solve 
\begin{align}
x^* = \argmin_{x\in\mathcal{X}} \cost(P,x)    
\end{align}
for the optimal model parameter $x^*$. 

\emph{Example:} Let $\dist(p,x):= \| p - x \|_2$ denote the Euclidean distance between points $p$ and $x$. The \emph{minimum enclosing ball (MEB)} problem \cite{Badoiu02STOC} aims at minimizing the maximum distance between any data point and a center, i.e., $\cost(p,x)=\dist(p,x)$, $\cost(P,x) = \max_{p\in P}\cost(p,x)$, and $\mathcal{X}=\mathbb{R}^d$. 
The \emph{$k$-means clustering} problem aims at minimizing the weighted sum of the squared distance between each data point and the nearest center in a set of $k$ centers, i.e., $\cost(p,x) = \min_{x_i\in x}\dist(p,x_i)^2$, $\cost(P,x) = \sum_{p\in P}w_p \cost(p,x)$, and $\mathcal{X} = \{x:=\{x_i\}_{i=1}^k:\: x_i\in\mathbb{R}^d\}$. 

Typically, the overall cost is defined as: (i) \emph{sum cost}, i.e., $\cost(P,x)= \sum_{p\in P}w_p \cost(p,x)$ (e.g., $k$-means), or (ii) \emph{maximum cost}, i.e., $\cost(P,x) = \max_{p\in P}\cost(p,x)$ (e.g., MEB).

A coreset is a small weighted dataset in the same space as the original dataset that approximates the original dataset in terms of cost, formally defined below.

\begin{definition}[\cite{Feldman11STOC}]\label{def:epsilon-coreset}
A weighted set $S\subseteq \mathbb{R}^d$ with weights $u_q$ ($q\in S$) is an $\epsilon$-coreset for P with respect to (w.r.t.) $\cost(P,x)$ ($x\in \mathcal{X}$) if $\forall x\in \mathcal{X}$,
\begin{align}
(1-\epsilon)\cost(P,x) \leq \cost(S,x) \leq (1+\epsilon)\cost(P,x),
\end{align}
where $\cost(S,x)$ is defined in the same way as $\cost(P,x)$, i.e., $\cost(S,x) = \sum_{q\in S}u_q\cost(q,x)$ for sum cost, and $\cost(S,x) = \max_{q\in S}\cost(q,x)$ for maximum cost.
\end{definition}

From Definition~\ref{def:epsilon-coreset}, it is clear that the quality of a coreset depends on the cost function it needs to approximate, and hence the machine learning problem it supports.

\subsection{Coreset Construction Algorithms}\label{subsec:Coreset Construction Algorithms}

Because of the dependence on the cost function (Definition~\ref{def:epsilon-coreset}), existing coreset construction algorithms are tailor-made for specific machine learning problems.
Here we briefly summarize common approaches for coreset construction and representative algorithms, and refer to \cite{Phillips16CoRR,Munteanu18KI} for detailed surveys. 

\emph{1) Farthest point algorithms:} Originally proposed for MEB \cite{Badoiu02STOC,Badoiu03SODA}, these algorithms iteratively add to the coreset a point far enough or farthest from the current center, and stop when the enclosing ball of the coreset, expanded by $1+\epsilon$, includes all data points. This coreset has been used to compute $\epsilon$-approximation to several clustering problems, including $k$-center clustering, 1-cylinder clustering, and $k$-flat clustering \cite{Badoiu02STOC,Har-Peled02SCG}. As \emph{support vector machine (SVM)} training can be formulated as MEB problems \cite{Tsang05JMLR}, similar algorithms have been used to support SVM \cite{Tsang05JMLR,Har-Peled07IJCAI}. Variations have been used for dimensionality reduction \cite{Feldman16NIPS} and probabilistic MEB \cite{Feldman14SOCG}. These algorithms are considered as variations of the Frank-Wolfe algorithm \cite{Clarkson10TALG}. 

\emph{2) Random sampling algorithms:} These algorithms construct a coreset by sampling from the original dataset. The basic version, uniform sampling, usually requires a large coreset size to achieve a good approximation. Advanced versions use  \emph{sensitivity sampling} \cite{Langberg10SODA}, where each data point is sampled with a probability proportional to its contribution to the overall cost. Proposed for numerical integration \cite{Langberg10SODA}, the idea was extended into a framework supporting projective clustering problems that include $k$-median/means and \emph{principle component analysis (PCA)} as special cases \cite{Feldman11STOC}. The framework has been used to generate coresets for other problems, e.g., dictionary learning \cite{Feldman13JMIV} and dependency networks \cite{Molina18AAAI} , and is further generalized in \cite{Braverman16CoRR}. Although the framework can instantiate algorithms for different machine learning problems by plugging in different cost functions, the resulting coreset only guarantees approximation for the specific problem defined by the plugged-in cost function. 

\emph{3) Geometric decomposition algorithms:} These algorithms divide the sample space or input dataset into partitions, and then selecting points to represent each partition. Specific instances have been developed for weighted facility problems \cite{Feldman06FOCS}, Euclidean graph problems \cite{Frahling05STOC}, $k$-means/median \cite{Barger16SDM,Har-Peled04STOC}.

While there are a few works not fully covered by the above approaches, e.g., SVD-based algorithms in \cite{Feldman13SODA,Boutsidis13IT}, the above represents the key approaches used by existing coreset construction algorithms. Using a generic merge-and-reduce approach in \cite{Feldman11NIPS}, all these algorithms can be used to construct coresets of distributed datasets. Of course, the resulting coresets are still tailor-made for specific problems. 
In contrast, we seek coreset construction algorithms which can construct coresets that simultaneously support multiple machine learning problems.

\section{Robust Coreset Construction}\label{sec:Robust Coreset Construction}

Our main result is that selecting \emph{representative points} using clustering techniques yields a coreset that achieves a good approximation for a broad set of machine learning problems.
We will start with a centralized setting in this section, where the raw data reside at a single data source (that needs to compute and report the coreset to a server as illustrated in Figure~\ref{fig:systemmodel}), and leave the distributed setting where the raw data are distributed across multiple data sources to Section~\ref{sec:Application in Distributed Setting}. 

\subsection{Motivating Experiment}\label{subsec:Motivating Experiment}

\begin{figure}

\centering
\begin{minipage}{.24\textwidth}
\centerline{
  \includegraphics[width=\textwidth]{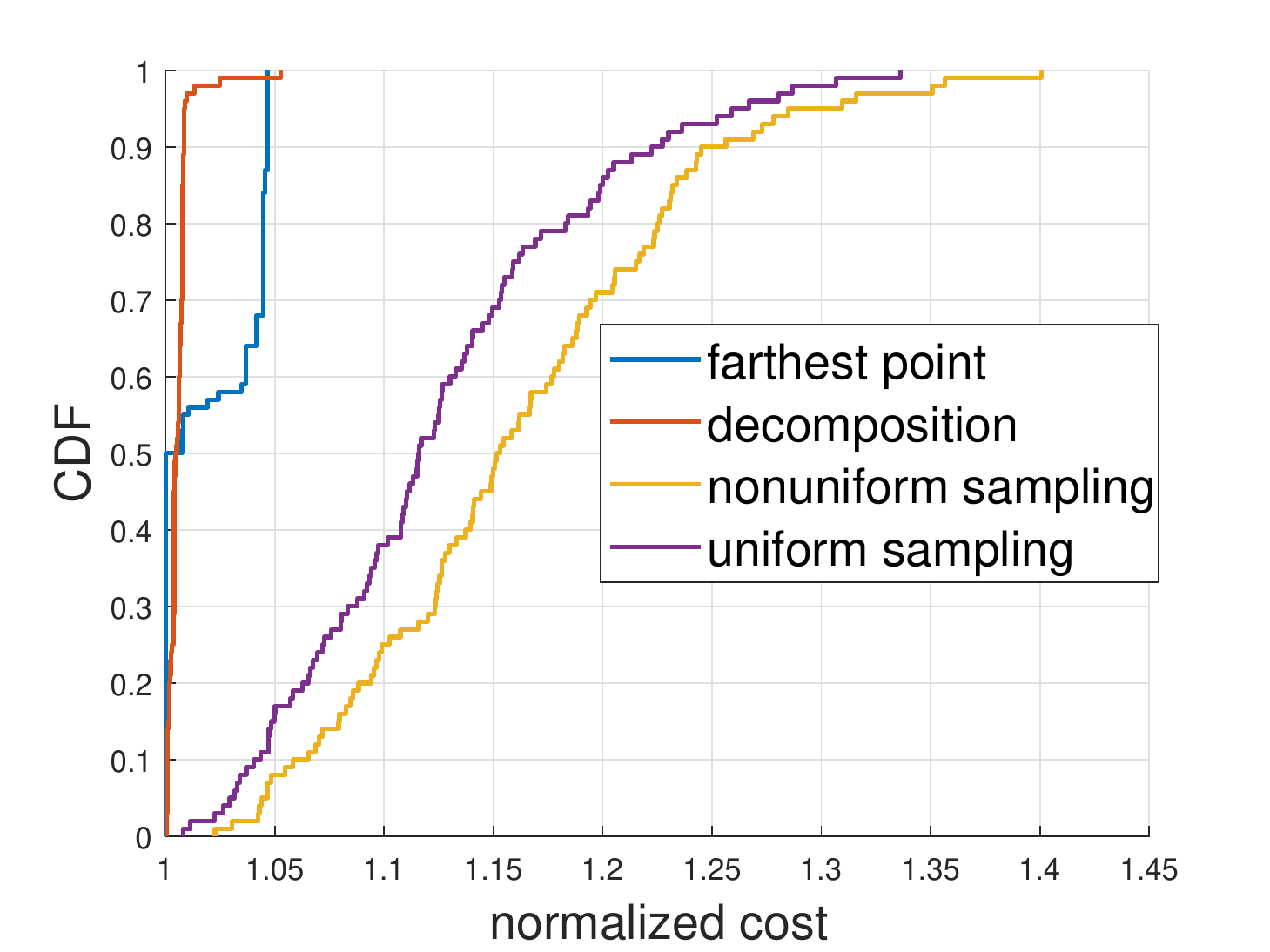}}
  \centerline{(a) MEB}
\end{minipage}\hfill
\begin{minipage}{.24\textwidth}
\centerline{
  \includegraphics[width=\textwidth]{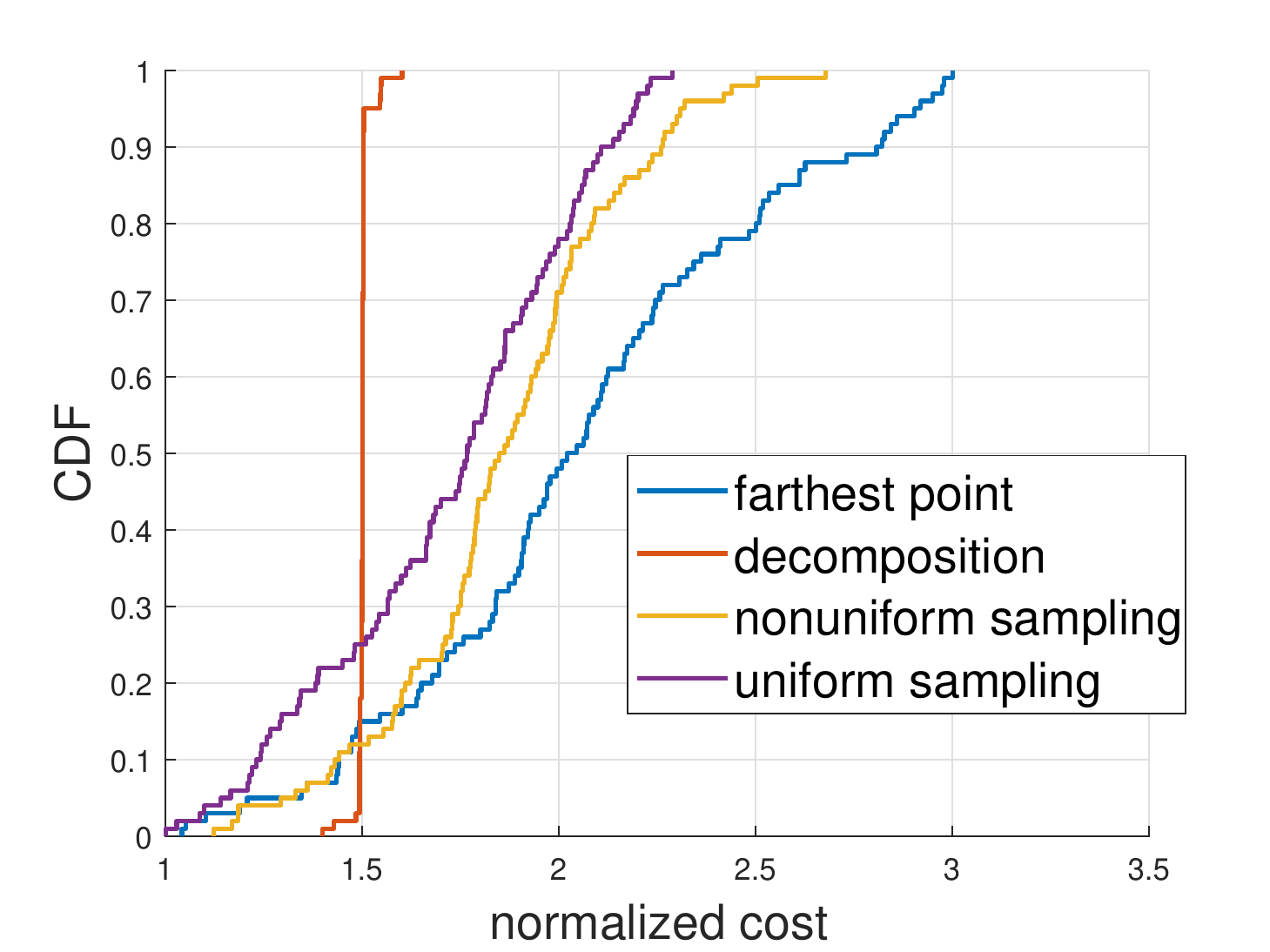}}
  \centerline{(b) $k$-means ($k=2$) }
\end{minipage}
\caption{Comparison of coreset construction algorithms (coreset size: 8). }  \label{fig:motivating experiment}
\end{figure}

We start with an initial experiment that compares selected algorithms representing the three approaches  in Section~\ref{subsec:Coreset Construction Algorithms}: (i) the algorithm in \cite{Badoiu03SODA} (`farthest point') representing farthest point algorithms, (ii) the framework in \cite{Feldman11STOC} instantiated for $k$-means (`nonuniform sampling') and the uniform sampling algorithm (`uniform sampling'), both representing random sampling algorithms, and (iii) the algorithm in \cite{Barger16SDM} (`decomposition') representing geometric decomposition algorithms. As the algorithm in \cite{Badoiu03SODA} was designed for MEB and the algorithms in  \cite{Feldman11STOC,Barger16SDM} were designed for $k$-means, we evaluate both MEB and $k$-means. 

The evaluation is based on a synthetic dataset containing $4000$ points uniformly distributed in $[1, 50]^3$; evaluations on real datasets will be presented later (Section~\ref{sec:Performance Evaluation}). All the algorithms are tuned to have the same average coreset size. We evaluate the performance of a coreset $S$ by the \emph{normalized cost}, defined as $\cost(P,x_S)/\cost(P,x^*)$, where $x^*$ is the model learned from the original dataset $P$, and $x_S$ is the model learned from the coreset. The smaller the normalized cost, the better the performance. As these coreset construction algorithms are randomized, we plot the CDF of the normalized costs computed over $100$ Monte Carlo runs in Figure~\ref{fig:motivating experiment}.


Not surprisingly, a coreset construction algorithm designed for one problem can perform poorly for another, e.g., the farthest point algorithm \cite{Badoiu03SODA} designed for MEB performs poorly for $k$-means. Interestingly, the decomposition algorithm \cite{Barger16SDM}, although designed for $k$-means, also performs well for MEB. This observation triggers a question: Is the superior performance of the decomposition algorithm \cite{Barger16SDM} just a coincidence, or is there something fundamental?

\subsection{The $k$-clustering Problem}

At the core, the decomposition algorithm in \cite{Barger16SDM} constructs a $k$-point coreset by partitioning the dataset into $k$ clusters using $k$-means clustering, and then using the cluster centers as the coreset points. To analyze its performance in supporting a general machine learning problem, 
we introduce a few definitions.

Given a weighted dataset $P\subseteq \mathbb{R}^d$ with weight $w_p$ ($p\in P$), and a set $Q = \{ q_1, ..., q_k\}$ of $k\geq 1$ points in $\mathbb{R}^d$ (referred to as \emph{centers}), the cost of clustering $P$ into $Q$ is defined as
\begin{equation}\label{eq:k-clustering cost}
c(P, Q) = \sum_{p \in P}w_p (\min_{q\in Q} \dist(p, q))^z,
\end{equation}
for a constant $z>0$. 
The \emph{$k$-clustering} problem is to find the set of $k$ centers that minimizes (\ref{eq:k-clustering cost}). For $z = 1$, this is the $k$-median problem. For $z = 2$, this is the $k$-means problem. We will use the solution to the $k$-clustering problem to construct coresets, based on which we can solve general machine learning problems. We use $c(P,\cdot)$ to denote the cost function of this auxiliary problem and $\cost(P,\cdot)$ to denote the cost function of a general machine learning problem of interest. 

We denote by $\mu(P)$ the optimal center for $1$-clustering of $P$. It is known that for $z=2$, $\mu(P)$ is the sample mean:
\begin{equation}
\mu(P) = \frac{1}{\sum_{p \in P} w_p}\sum_{p \in P}w_p \cdot p.
\end{equation}

We denote by $\opt(P,k)$ the optimal cost for $k$-clustering of $P$. 
It is known that $k$-means and $k$-median are both NP-hard problems \cite{Aloise09ML,Megiddo84JC}, for which efficient heuristics exist (e.g., Lloyd's algorithm and variations) \cite{Arthur07SODA}. Let $\approx(P, k)$ denote the cost of a generic $k$-clustering algorithm, which always satisfies $\approx(P, k)\geq \opt(P, k)$. 

Each set of $k$ centers $Q=\{q_i\}_{i=1}^k$ induces a partition of $P$ into $\{P_1,\ldots,P_k\}$, where $P_i$ is the subset of points in $P$ whose closest center in $Q$ is $q_i$ (ties broken arbitrarily). For ease of presentation, we use\footnote{Throughout the paper, for $k\in \mathbb{Z}^+$, $[k]:=\{1,\ldots,k\}$. } $\{P_i\}_{i\in [k]}$ to denote the partition induced by the optimal $k$-clustering, and $\{\widetilde{P}_i\}_{i\in [k]}$ to denote the partition induced by a suboptimal $k$-clustering.

\subsection{Coreset by Optimal $k$-clustering}

We will show that the superior performance of the algorithm in \cite{Barger16SDM} observed in Section~\ref{subsec:Motivating Experiment} is not a coincidence; instead, it is a fundamental property of any coreset computed by $k$-clustering, as long as the cost function of the machine learning problem satisfies certain continuity conditions. 

\emph{Sketch of analysis:} At a high level, our analysis is based on the following observations:
\begin{enumerate}
\item If doubling the number of centers only reduces the optimal $k$-clustering cost by a little, then using two centers instead of one in any cluster gives little reduction to its clustering cost (Lemma~\ref{lem:1}).
\item If selecting two centers in a cluster $P_i$ gives little reduction to its clustering cost, then all the points in $P_i$ must be close to its center $\mu(P_i)$ (Lemma~\ref{lem:2}), as otherwise selecting an outlier as the second center would have reduced the cost substantially. 
\item If each data point is represented by a coreset point with a similar per-point cost, then the coreset gives a good approximation of the overall cost (Lemmas~\ref{lem:3} and \ref{lem:4}).
\end{enumerate}
Therefore, for any machine learning problem with a sufficiently continuous cost function, if the condition in item (1) is satisfied, then the per-point cost of each $k$-clustering center will closely approximate the per-point costs of all the points in its cluster, and hence the set of $k$-clustering centers will give a good coreset (Theorem~\ref{thm:1}).

\emph{Complete analysis:} We now present the precise statements, supported by proofs in
\if\thisismainpaper1
 Appendix~A in \cite{coreset19:report}.
\else
 Appendix~A.
\fi

\begin{lemma}\label{lem:1}
For any $\epsilon' > 0$, if $\opt(P, k) - \opt(P, 2k) \leq \epsilon'$, then $\opt(P_i, 1) - \opt(P_i, 2) \leq \epsilon'$ ($\forall i\in [k]$), where $\{ P_i \}^k_{i=1}$ is the partition of $P$ generated by the optimal $k$-clustering.
\end{lemma}

\begin{lemma}\label{lem:2}
If $\opt(P_i, 1) - \opt(P_i, 2) \leq \epsilon'$, then $\dist(p, \mu(P_i)) \leq (\frac{\epsilon'}{\wmin})^{\frac{1}{z}}$, $\forall p \in P_i$.
\end{lemma}

\begin{lemma}\label{lem:3}
For any machine learning problem with cost function $\cost(P, x) = \sum_{p \in P} w_p \cost(p, x)$, if $\exists$ a partition $\{ P_i \}^k_{i=1}$ of $P$ such that $\forall x \in \mathcal{X}$, $i \in [k]$, and $p \in P_i$,
\begin{align}\label{eq:lem3}
&\hspace{-1em} (1-\epsilon)\cost(p, x)\leq \cost(\mu(P_i), x) \leq (1+\epsilon)\cost(p, x),
\end{align}
then $S = \{ \mu(P_i) \}^k_{i=1}$  with weight $u_{\mu(P_i)} =  \sum_{p\in P_i}w_p$ is an $\epsilon$-coreset for P w.r.t. $\cost(P, x)$.
\end{lemma}

\begin{lemma}\label{lem:4}
For any machine learning problem with cost function $\cost(P, x) = \max_{p \in P} \cost(p, x)$, if $\exists$ a partition $\{ P_i \}^k_{i=1}$ of $P$ such that (\ref{eq:lem3}) holds for any $x \in \mathcal{X}$, $i \in [k]$, and $p \in P_i$,
then $S = \{ \mu(P_i) \}^k_{i=1}$ (with arbitrary weights) is an $\epsilon$-coreset for $P$ w.r.t. $\cost(P, x)$.
\end{lemma}

We now prove the main theorem based on Lemmas~\ref{lem:1}--\ref{lem:4}. \looseness=-1

\begin{theorem}\label{thm:1}
If $\opt(P, k) - \opt(P, 2k) \leq \wmin (\frac{\epsilon}{\rho})^z$, then the optimal $k$-clustering of $P$ gives an $\epsilon$-coreset for $P$ w.r.t. both the sum cost and the maximum cost for any per-point cost function satisfying (i) $\cost(p, x) \geq 1$, and (ii) $\cost(p,x)$ is $\rho$-Lipschitz-continuous in $p$, $\forall x \in \mathcal{X}$.
\end{theorem}
\begin{proof}
By \Cref{lem:1}, $\opt(P, k) - \opt(P, 2k) \leq \epsilon'$ implies $\opt(P_i, 1) - \opt(P_i, 2) \leq \epsilon'$, $\forall$ cluster $P_i$ generated by the optimal $k$-clustering.
By \Cref{lem:2}, this in turn implies that $\dist(p, \mu(P_i)) \leq (\frac{\epsilon'}{\wmin})^{\frac{1}{z}}$, $\forall p \in P_i$. Because $\cost(p, x)$ is $\rho$-Lipschitz-continuous in $p$ for all $x\in \mathcal{X}$, we have
\begin{align}
\mid \cost(p, x) - \cost(\mu(P_i), x) \mid \leq \rho (\frac{\epsilon'}{\wmin})^{\frac{1}{z}}, \forall x \in \mathcal{X}, p \in P_i.
\end{align}
Moreover, as $\cost(p,x) \geq 1$,
\[
\begin{aligned}
\frac{\mid \cost(p, x) - \cost(\mu(P_i), x) \mid}{\cost(p, x)} \leq \rho (\frac{\epsilon'}{\wmin})^{\frac{1}{z}} = \epsilon
\end{aligned}
\]
for $\epsilon' = \wmin (\frac{\epsilon}{\rho})^z$. By \Cref{lem:3}, $k$-clustering gives an $\epsilon$-coreset for $P$ w.r.t. the sum cost; by \Cref{lem:4}, $k$-clustering gives an $\epsilon$-coreset for $P$ w.r.t. the maximum cost.
\end{proof}

Often in practice, the coreset size must satisfy some upper bound specified by the maximum communication overhead. In this case, we can rephrase \Cref{thm:1} to characterize the quality of approximation as a function of the coreset size.

\begin{corollary}\label{coro:1}
Given a maximum coreset size $k\in \mathbb{Z}^+$ (positive integers), for any cost function satisfying the conditions in Theorem~\ref{thm:1}, the optimal $k$-clustering gives an $\epsilon$-coreset for $P$ w.r.t. this cost function, where
\begin{align}
\epsilon = \rho\left({\frac{\opt(P,k)-\opt(P,2k)}{\wmin}} \right)^{\frac{1}{z}}. \label{eq:epsilon, opt}
\end{align}
\end{corollary}
\begin{proof}
This is a direct implication of \Cref{thm:1}, as setting $\epsilon$ by (\ref{eq:epsilon, opt}) satisfies the condition in \Cref{thm:1}.
\end{proof}

\emph{Remark:} 
Condition (i) in \Cref{thm:1} is easily satisfied by any machine learning problem with nonnegative per-point costs, as we can add `$+1$' to the cost function without changing the optimal solution. Even without this condition, a similar proof will show that the coreset $S$ given by $k$-clustering approximates the original dataset $P$ in that $|\cost(P,x)-\cost(S,x)| \leq \widetilde{\epsilon}$ ($\forall x\in \mathcal{X}$), where $\widetilde{\epsilon} = \epsilon \sum_{p\in P}w_p$ for the sum cost, and $\widetilde{\epsilon} = \epsilon$ for the maximum cost. 

Condition (ii) is satisfied by many machine learning problems with distance-based cost functions. For example, for MEB, $\cost(p,x)= \dist(p,x)$, where $x\in \mathbb{R}^d$ denotes the center of the enclosing ball. For any data points $p,p'\in \mathbb{R}^d$, by the triangle inequality, we have:
\begin{align}
|\dist(p,x)-\dist(p',x)| \leq \dist(p,p'). 
\end{align}
Hence, its cost function is 1-Lipschitz-continuous (i.e., $\rho=1$). 
\if\thisismainpaper1
See Appendix~B in \cite{coreset19:report} for more examples. 
\else
See Appendix~B for more examples. 
\fi
In Section~\ref{sec:Performance Evaluation}, we will stress-test our coreset when this condition is violated. 

From the proof of Theorem~\ref{thm:1}, it is easy to see that the theorem holds as long as the distance between each data point and its nearest $k$-clustering center is bounded by $\epsilon/\rho$, i.e., $\dist(p, \mu(P_i))\leq \epsilon/\rho$ for all $i\in [k]$ and $p\in P_i$. This implies that Corollary~\ref{coro:1} actually holds for 
\begin{align}\label{eq:epsilon - opt, new}
    \epsilon = \rho \left(\max_{i\in [k]} \max_{p\in P_i}\: \dist(p, \mu(P_i)) \right), 
\end{align}
which can be much smaller than (\ref{eq:epsilon, opt}) for large datasets.

\subsection{Coreset by Suboptimal $k$-clustering}\label{subsec:Coreset Given by Suboptimal $k$-clustering}

While \Cref{thm:1} and \Cref{coro:1} suggest that the optimal $k$-clustering gives a good coreset, the $k$-clustering problem is NP-hard \cite{Aloise09ML,Megiddo84JC}. 
The question is whether similar performance guarantee holds for the coreset computed by an efficient but suboptimal $k$-clustering algorithm. To this end, we introduce a few assumptions on the $k$-clustering algorithm: 

\emph{\bf Assumption 1 (local optimality):} Given the partition $\{\widetilde{P}_i\}_{i=1}^k$ generated by the algorithm, the center it selects in each $\widetilde{P}_i$ is $\mu(\widetilde{P}_i)$, i.e., the optimal 1-clustering center for $\widetilde{P}_i$.

\emph{\bf Assumption 2 (self-consistency):} For any $P$ and any $k\geq 1$, the cost of the algorithm satisfies
\begin{align}\label{eq:self-consistency}
\approx(P,2k) \leq \sum_{i=1}^k \approx(\widetilde{P}_i,2).
\end{align}

\emph{\bf Assumption 3 (greedy dominance):} For any $P$, the $2$-clustering cost of the algorithm satisfies
\begin{align}\label{eq:greedy dominance}
\approx(P, 2) \leq c(P, \{\mu(P), p^*\}),
\end{align}
where $c(P, Q)$ is defined in (\ref{eq:k-clustering cost}), and $p^* \coloneqq \argmax_{p\in P} w_p \cdot\dist(p,\mu(P))^z$ is the point with the highest 1-clustering cost. 

These are mild assumptions that should be satisfied or approximately satisfied by any good $k$-clustering algorithm. Assumption~1 is easy to satisfy, as computing the 1-mean is easy (i.e., sample mean), and there exists an algorithm \cite{Cohen16STOC} that can compute the 1-median to arbitrary precision in nearly linear time. Assumption~2 means that applying the algorithm for $2k$-clustering of $P$ should perform no worse than first using the algorithm to partition $P$ into $k$ clusters, and then computing $2$-clustering of each cluster. Assumption~3 means that for $k=2$, the algorithm should perform no worse than a greedy heuristic that starts with the 1-clustering center, and then adds the point with the highest clustering cost as the second center. 
We will discuss how to ensure these assumptions for the proposed algorithm in Section~\ref{subsec:Coreset Construction Algorithm}.

We show that for any $k$-clustering algorithm satisfying these assumptions, statements analogous to \Cref{lem:1} and \Cref{lem:2} can be made (proofs in
\if\thisismainpaper1
 Appendix~A in \cite{coreset19:report}).
\else
 Appendix~A).
\fi 
Let $\{ \widetilde{P}_i \}^k_{i=1}$ denote the partition of $P$ generated by the $k$-clustering algorithm.

\begin{lemma}\label{lem:1 suboptimal}
For any $\epsilon' > 0$, if $\approx(P, k) - \approx(P, 2k)$ $\leq \epsilon'$, then $\approx(\widetilde{P}_i, 1) - \approx(\widetilde{P}_i, 2) \leq \epsilon'$ for any $i\in [k]$.
\end{lemma}

\begin{lemma}\label{lem:2 suboptimal}
If $\approx(\widetilde{P}_i, 1)-\approx(\widetilde{P}_i, 2)\leq \epsilon'$, then $\dist(p,\mu(\widetilde{P}_i)) \leq ({\epsilon'\over \wmin})^{1\over z}$, $\forall p\in \widetilde{P}_i$.
\end{lemma}

\begin{theorem}\label{thm:2}
If $\approx(P,k)-\approx(P,2k) \leq \wmin({\epsilon\over \rho})^z$, where $\approx(P,k)$ is the cost of a (possibly suboptimal) $k$-clustering algorithm satisfying Assumptions~1--3, then the centers computed by the algorithm for $k$-clustering of $P$ give an $\epsilon$-coreset for $P$ w.r.t. both the sum cost and the maximum cost for any per-point cost function satisfying (i--ii) in \Cref{thm:1}.
\end{theorem}
\begin{proof}
The proof follows the same steps as that of \Cref{thm:1}, except that \Cref{lem:1} is replaced by \Cref{lem:1 suboptimal}, and \Cref{lem:2} is replaced by \Cref{lem:2 suboptimal}. Note that Lemmas~\ref{lem:3} and \ref{lem:4} hold for any partition of $P$, which in this case is $\{\widetilde{P}_i\}_{i=1}^k$ generated by the $k$-clustering algorithm. 
\end{proof}

Similar to \Cref{coro:1}, we can rephrase \Cref{thm:2} to characterize the quality of a coreset   of a specified size.

\begin{corollary}\label{coro:3}
Given a maximum coreset size $k\in \mathbb{Z}^+$, for any cost function satisfying the conditions in \Cref{thm:1} and any $k$-clustering algorithm satisfying Assumptions~1--3, the centers computed by the algorithm for $k$-clustering of $P$ give an $\epsilon$-coreset for $P$ w.r.t. the given cost function, where
\begin{align}
\epsilon = \rho \left({\approx(P,k) - \approx(P,2k)\over \wmin}\right)^{1\over z}.
\end{align}
\end{corollary}

\subsection{Coreset Construction Algorithm}\label{subsec:Coreset Construction Algorithm}

Based on \Cref{thm:2}, we propose a centralized $k$-clustering-based coreset construction algorithm, called \emph{Robust Coreset Construction (RCC)} (Algorithm~\ref{Alg:k-clustering coreset}), which uses a $k$-clustering algorithm as subroutine in lines~\ref{RCC:1} and \ref{RCC:2}. 
If the coreset size $k$ is predetermined, we can directly start from line~\ref{RCC:2}. The constant $z=1$ if the adopted clustering algorithm is for $k$-median, or $z=2$ if it is for $k$-means. 

\begin{algorithm}[tb]
\caption{ Robust Coreset Construction $(P,\epsilon,\rho)$}
\label{Alg:k-clustering coreset}
\SetKwInOut{Input}{input}\SetKwInOut{Output}{output}
\Input{A weighted set $P$ with minimum weight $\wmin$, approximation error $\epsilon>0$, Lipschitz constant $\rho$ }
\Output{An $\epsilon$-coreset $S$ for $P$ w.r.t. a cost function satisfying \Cref{thm:2} }
\ForEach{$k=1,\ldots,|P|$}
{\If{$\approx(P,k)-\approx(P,2k)\leq \wmin({\epsilon\over \rho})^z$ \label{RCC:1}}
{break\;}
}
$(\{\mu(\widetilde{P}_i)\}_{i=1}^k, \{\widetilde{P}_i\}_{i=1}^k)\leftarrow k\mbox{-clustering}(P,k)$\; \label{RCC:2}
$S\leftarrow \{\mu(\widetilde{P}_i)\}_{i=1}^k$, where $\mu(\widetilde{P}_i)$ has weight $\sum_{p\in\widetilde{P}_i}w_p$\;
return $S$\;
\end{algorithm}
\normalsize
\setlength{\textfloatsep}{.5em}

\emph{The $k$-clustering subroutine:} Algorithm~\ref{Alg:k-clustering coreset} can use any $k$-clustering algorithm as subroutine, although our performance guarantee holds only if the algorithm satisfies Assumptions~1--3. 
We note that these assumptions are easy to satisfy 
if $z=2$
. Consider the standard $k$-means algorithm (i.e., Lloyd's algorithm), which iteratively assigns each point to the nearest center and updates the centers to the means of the clusters. Clearly, this algorithm satisfies Assumption~1.  Moreover, with the following initialization, it also satisfies Assumptions~2 and 3. For $(2k)$-clustering of $P$: 
\begin{enumerate}
\item if $k=1$, then use the mean $\mu(P)$ and the point $p^*$ with the highest clustering cost as defined in (\ref{eq:greedy dominance}) as the initial centers, which helps to satisfy Assumption~3;  
\item if $k>1$, then first compute $k$-clustering of $P$, and then compute $2$-clustering of each of the $k$ clusters 
(both by calling the same algorithm recursively)
; finally, use the union of the $2$-clustering centers as the initial centers, 
which helps to satisfy Assumption~2.
\end{enumerate}
Any odd number of initial centers are chosen randomly. 
Since iterations can only reduce the cost, Lloyd's algorithm with this initialization satisfies Assumptions 1--3. 

In theory, the above initialization plus a Lloyd-style algorithm can satisfy Assumptions~1--3 for an arbitrary $z>0$, given a subroutine to compute the optimal $1$-clustering center $\mu(P)$. For $z=1$, there is an algorithm to compute $\mu(P)$ to an arbitrary precision in nearly linear time~\cite{Cohen16STOC}. 


\section{Distributed Coreset Construction}\label{sec:Application in Distributed Setting}

In a distributed setting, the entire dataset $P$ is distributed across $n$ ($n>1$) nodes 
(i.e., data sources)
$v_1,\ldots,v_n$, where each $v_j$ has a subset $P_j\subseteq P$. 
We have shown in Section~\ref{sec:Robust Coreset Construction} that the $k$-clustering centers of $P$ form a robust coreset. However, computing the global $k$-clustering centers of a distributed dataset is highly non-trivial. Note that a naive solution that only includes local centers in the global coreset may select nearly identical points at different nodes if the local datasets are similar, which is non-optimal and inefficient.

\textbf{An existing algorithm:}
The state-of-the-art solution to the distributed $k$-clustering problem is based on an algorithm called \emph{Communication-aware Distributed Coreset Construction (CDCC)} \cite{Balcan13NIPS}. In our context, this solution works as follows: 
\begin{enumerate}
\item the server allocates a given sample size $t$ among the nodes, such that the sample size $t_j$ allocated to node $v_j$ is proportional to the local $k$-clustering cost $c(P_j, B_j)$ reported by $v_j$; 
\item each node $v_j$ generates and reports a local coreset $D_j$, consisting of the local  centers $B_j$ and $t_j$ points sampled i.i.d. from $P_j$, where each $p\in P_j$ has a sampling probability proportional to the cost of clustering $p$ to the nearest center in $B_j$;
\item the server computes a set of $k$-clustering centers $Q_D$ from the coreset $D=\bigcup_{j=1}^n D_i$.
\end{enumerate}
It is shown in \cite{Balcan13NIPS} that if $t=O({1\over \epsilon^2}(kd + \log{1\over \delta}))$ for $k$-median and $t=O({1\over \epsilon^4}(kd+\log{1\over \delta})+nk\log{nk\over \delta})$ for $k$-means, then with probability at least $1-\delta$, $D$ is an $\epsilon$-coreset for $P$ w.r.t. the cost function of $k$-median/means. According to Definition~\ref{def:epsilon-coreset}, this implies that if $Q_P$ is the set of optimal $k$-clustering centers for $P$, then $c(P,Q_D)/c(P,Q_P) \leq(1+\epsilon)/(1-\epsilon)$. 

\textbf{Adaptation for coreset construction:} 
First, we skip step (3) (i.e., computation of $Q_D$) and directly use $D=\bigcup_{j=1}^n D_j$ as the coreset. This is because the coreset of a coreset cannot have a better quality than the original coreset \cite{Feldman11NIPS}. 

Moreover, in CDCC, the number of local centers $k$ is a given parameter as it is only designed to support $k$-clustering. Since our goal is to support a variety of machine learning problems, the number of local centers $k_j$ at each $v_j$ becomes a design parameter that can vary across nodes. Given a global coreset size $N$, we will show that the approximation error of the constructed coreset depends on $(k_j)_{j=1}^n$ through ${1\over \sqrt{N-\sum_{j=1}^n k_j}}\sum_{j=1}^n \approx{(P_j,k_j)}$ (see \Cref{thm:DRCC_sum-cost_kmedian}). Thus, we set $(k_j)_{j=1}^n$ to minimize this error, and obtain the remaining $t = N-\sum_{j=1}^n k_j$ points by sampling. 

\begin{algorithm}[tb]
\caption{ Distributed Robust Coreset Construction $((P_j)_{j=1}^n,N,K)$}
\label{Alg:distributed RCC}
\SetKwInOut{Input}{input}\SetKwInOut{Output}{output}
\SetKwFor{EachNode}{each $v_j$ ($j\in [n]$):}{}{}
\SetKwFor{TheServer}{the server:}{}{}
\Input{A distributed dataset $(P_j)_{j=1}^n$, global coreset size $N$, maximum number of local centers $K$}
\Output{A coreset $D=\bigcup_{j=1}^n (S_j\cup B_j^{k_j})$ for $P=\bigcup_{j=1}^n P_j$ }
\BlankLine
\EachNode{}{
\dosemic compute local approximate $k$-clustering centers $B_j^k$ on $P_j$ for $k=1,\ldots,K$\;\label{DRCC: compute local centers}
report $(c(P_j,B_j^k))_{k=1}^K$ to the server\; \label{DRCC: report local costs}
}
\BlankLine
\TheServer{}{
\dosemic find $(k_j)_{j=1}^n$ that minimizes ${1\over \sqrt{N-\sum_{j=1}^n k_j}}\sum_{j=1}^n c(P_j, B_j^{k_j})$ s.t. $k_j\in [K]$ and $\sum_{j=1}^n k_j\leq N$\; \label{DRCC: find k_j}
randomly allocate $t=N-\sum_{j=1}^n k_j$ points i.i.d. among $v_1,\ldots,v_n$, where each point belongs to $v_j$ with probability ${c(P_j,B_j^{k_j})\over \sum_{j=1}^n c(P_j,B_j^{k_j})}$\; \label{DRCC: allocate t}
communicate $(k_j, t_j, {C\over t})$ to each $v_j$ ($j\in [n]$), where $t_j$ is the number of points allocated to $v_j$ and $C=\sum_{l=1}^n c(P_l,B_l^{k_l})$\; \label{DRCC: communicate configuration}
}
\BlankLine
\EachNode{}{
\dosemic sample a set $S_j$ of $t_j$ points i.i.d. from $P_j$, where each sample equals $p\in P_j$ with probability ${m_p\over c(P_j,B_j^{k_j})}$ for $m_p=c(\{p\},B_j^{k_j})$\; \label{DRCC: sample S_j}
set the weight of each $q\in S_j$ to $u_q = {C w_q \over t m_q}$\; \label{DRCC: set sample weight}
set the weight of each $b\in B_j^{k_j}$ to $u_b = \sum_{p\in P_b}w_p - \sum_{q\in P_b\cap S_j} u_q$, where $P_b$ is the set of points in $P_j$ whose closest center in $B_j^{k_j}$ is $b$\; \label{DRCC: set center weight}
report each point $q\in S_j\cup B_j^{k_j}$ and its weight $u_q$ to the server\; \label{DRCC: report S_j and B_j}
}
\end{algorithm}
\normalsize
\setlength{\textfloatsep}{.5em}

Combining these ideas yields a distributed coreset construction algorithm called \emph{Distributed Robust Coreset Construction (DRCC)} (Algorithm~\ref{Alg:distributed RCC}). The algorithm works in three steps: (1) each node reports its local $k$-clustering cost for a range of $k$ (lines~\ref{DRCC: compute local centers}-\ref{DRCC: report local costs}), (2) the server uses the reported costs to configure the number of local centers $k_j$ and the number of random samples $t_j$ at each node $v_j$ (lines~\ref{DRCC: find k_j}-\ref{DRCC: communicate configuration}), and (3) each node independently constructs a local coreset using a combination of samples and local centers (lines~\ref{DRCC: sample S_j}-\ref{DRCC: report S_j and B_j}).   
DRCC generalizes CDCC in that: (i) it allows the input dataset to be weighted ($w_p$: weight of input point $p$; $u_q$: weight of coreset point $q$); (ii) it allows the number of local centers to be different for different nodes. In the special case of $k_j\equiv k$ for all $j\in [n]$ and $w_p\equiv 1$ for all $p\in P$, DRCC is reduced to CDCC. 

\textbf{Communication overhead:} DRCC has a communication overhead of $O(Kn)$, measured by the total number of scalars reported by the nodes besides the coreset itself. 
In practice, $K$ should be a small constant to allow efficient computation of local $k$-clustering for $k\in [K]$. 
This overhead is much smaller than the $O((n-1)Nd)$ overhead of the merge-and-reduce approach\footnote{
Specifically, the merge-and-reduce approach works by applying a centralized coreset construction algorithm (e.g., Algorithm~\ref{Alg:k-clustering coreset}) repeatedly to combine local coresets into a global coreset. Given $n$ local coresets computed by each of the $n$ nodes, each containing $N$ points in $\mathbb{R}^d$, the centralized algorithm needs to be applied $n-1$ times, each time requiring a local coreset to be transmitted to the location of another local coreset. This results in a total communication overhead of $(n-1)Nd$.   
} in \cite{Feldman11NIPS}. 

\textbf{Quality of coreset:} Regarding the quality of the coreset, we have proved the following result in 
\if\thisismainpaper1
Appendix~A in \cite{coreset19:report}.
\else
Appendix~A.
\fi
Given a general per-point cost function $\cost(p,x)$, define $f_x(p) \coloneqq w_p(\cost(p,x)-\cost(b_p,x) + \rho \dist(p,b_p))$, where $b_p$ is the center in $B_j^{k_j}$ closest to $p\in P_j$. Let $\dim(F,P)$ denote the \emph{dimension of the function space} $F\coloneqq \{f_x(p):\: x\in\mathcal{X}\}$ \cite{Balcan13NIPS}. 

\begin{theorem}\label{thm:DRCC_sum-cost_kmedian}
If $\cost(p,x)$ is $\rho$-Lipschitz-continuous in $p$ for any $x\in\mathcal{X}$, then $\exists t = O({1\over \epsilon^2}(\dim(F,P)+\log{1\over \delta})$ such that with probability at least $1-\delta$, the coreset $D$ constructed by DRCC based on local $k$-median clustering, which contains $k_j$ local centers from $v_j$ ($j\in [n]$) and $t$ random samples, satisfies
\begin{align}\label{eq:DRCC approx error}
\Big| \sum_{p\in P} w_p \cost(p,x)  - \sum_{q\in D} u_q \cost(q,x)\Big| 
\leq  {2 \epsilon \rho } \sum_{j=1}^n c(P_j, B_j^{k_j}) 
\end{align}
for all $x\in\mathcal{X}$.
\end{theorem}

Here, the parameter $\dim(F,P)$ is a property of the machine learning problem under consideration, which intuitively measures the degree of freedom in the solution $x\in \mathcal{X}$, e.g., $\dim(F,P)=O(kd)$ for $k$-means/median in $d$-dimensional space \cite{Feldman11STOC}. 
\if\thisismainpaper1
See Appendix~C in \cite{coreset19:report} for more discussions. 
\else
See Appendix~C for more discussions. 
\fi

Due to the relationship between $t$ and $\epsilon$ given in Theorem~\ref{thm:DRCC_sum-cost_kmedian}, the bound on the right-hand side of (\ref{eq:DRCC approx error}) depends on parameters $t$ and $(k_j)_{j=1}^n$ through ${1\over \sqrt{t}}\sum_{j=1}^n c(P_j, B_j^{k_j})$. 
Specifically, given a total coreset size $N$, the right-hand side of (\ref{eq:DRCC approx error}) is 
\begin{align}
    O\left({\rho \sqrt{\dim(F,P)+\log{1\over \delta}}\over \sqrt{N-\sum_{j=1}^n k_j}}  \cdot \sum_{j=1}^n c(P_j, B_j^{k_j})\right). \label{eq:rewrite DRCC approx error}
\end{align}
This error bound tells us that the approximation error decreases with the coreset size $N$ at roughly $O(1/\sqrt{N})$. The error, however, may not be monotone with the numbers of local centers $k_j$'s, as increasing their values decreases both  $\sqrt{N-\sum_{j=1}^n k_j}$ and $\sum_{j=1}^n c(P_j, B_j^{k_j})$. Thus, we select $(k_j)_{j=1}^n$ to minimize the error bound in line~\ref{DRCC: find k_j} of Algorithm~\ref{Alg:distributed RCC}. 
As the server needs to know $(c(P_j,B_j^k))_{k=1}^K$ ($\forall j\in [n]$) to solve this minimization over $k_j\in [K]$, the choice of the parameter $K$ faces a tradeoff: a larger $K$ yields a larger solution space and possibly a better configuration of $(k_j)_{j=1}^n$ to minimize the approximation error, but incurs a higher communication (and computation) overhead at the nodes. The optimal $K$ will depend on the desirable tradeoff and the specific dataset.

\emph{Remark:} The performance bound in \Cref{thm:DRCC_sum-cost_kmedian} is on the absolute error, instead of the relative error as guaranteed by an $\epsilon$-coreset. 
Nevertheless, if $\exists \beta>0$ and $(k_j)_{j=1}^n$ such that $\cost(P,x) = \sum_{p\in P}w_p\cost(p,x) \geq \beta \sum_{j=1}^n c(P_j, B_j^{k_j})$ for all $x\in\mathcal{X}$, then (\ref{eq:DRCC approx error}) implies that $D$ is an $\epsilon$-coreset for $P$ w.r.t. $\cost(P,x)$ with probability at least $1-\delta$ if $t=O({\rho^2\over \epsilon^2}(\dim(F,P)+\log{1\over \delta}))$, i.e., the total coreset size $N=O({\rho^2\over \epsilon^2}(\dim(F,P)+\log{1\over \delta})+\sum_{j=1}^n k_j)$. In the special case where $\cost(P,x)$ is the $k$-median clustering cost and $k_j\equiv k$, we have $\rho=1$ 
\if\thisismainpaper1
(Appendix~B in \cite{coreset19:report})
\else
(Appendix~B)
\fi
and $\dim(F,P) =O(kd)$ \cite{Feldman11STOC}, and thus the size of an $\epsilon$-coreset is $O({1\over \epsilon^2}(kd+\log{1\over \delta})+kn)$, which generalizes the result in \cite{Balcan13NIPS} to weighted datasets. 

\section{Performance Evaluation}\label{sec:Performance Evaluation}

We evaluate the proposed coreset construction algorithms and their benchmarks on a variety of machine learning problems and datasets, and compare the cost of each model learned on a coreset with the cost of the model learned on the original dataset. We first perform evaluations in a centralized setting to compare different approaches to construct coresets, and then evaluate different ways of applying the most promising approach in a distributed setting.

\emph{\bf Coreset construction algorithms:} In the centralized setting, we evaluate RCC based on $k$-median clustering (`RCC-kmedian') and RCC based on $k$-means clustering (`RCC-kmeans'), together with benchmarks including the algorithm in \cite{Badoiu03SODA} (`farthest point'), the framework in \cite{Feldman11STOC} instantiated for $k$-means (`nonuniform sampling'), and uniform sampling (`uniform sampling'). We note that the algorithm in \cite{Barger16SDM} (`decomposition' in Figure~\ref{fig:motivating experiment}) is essentially RCC based on $k$-means clustering, hence omitted.

\begin{table}[tb]
\footnotesize
\renewcommand{\arraystretch}{1.3}
\caption{Parameters of Datasets} \label{tab:dataset}
\centering
\begin{tabular}{r|c|c|c}
  \hline
dataset & size ($|P|$) & dimension ($d$) & \#distinct labels ($L$)   \\
  \hline
  Fisher's iris & 150 & 5 & 3  \\
  \hline
  Facebook & 500 & 19 & 4 \\
  \hline
  Pendigits & 7494 & 17 & 10   \\
  \hline
 MNIST & 70000 & 401 & 10  \\
  \hline
  HAR & 10299 & 562 & 6 \\
  \hline
\end{tabular}
\end{table}
\normalsize

\begin{table*}[tb]
\small
\renewcommand{\arraystretch}{1.3}
\caption{Machine Learning Cost Functions} \label{tab:cost}
\centering
\begin{tabular}{r|c|c}
  \hline
  problem & overall cost function\footnotemark & $\rho$  \\
  \hline
  MEB & $\max_{p\in P}\dist(x,p)$ & $1$   \\
  \hline
  $k$-means & $\sum_{p\in P}w_p\cdot \min_{q_i\in x} \dist(q_i,p)^2$ & $2\Delta$   \\
  \hline
  PCA & $\sum_{p\in P}w_p\cdot \dist(p, x p)^2$ & $2\Delta(l+1)$   \\
  \hline
  SVM & $\sum_{p\in P} w_p \max(0, 1- p_d(p_{1:d-1}^T x_{1:d-1} + x_d ))$ & $\infty$  \\
  \hline
  Neural Net & $\sum_{p\in P} (-p_d) \cdot \log(o_p)$, where $o_p$ is the output for input $p_{1:d-1}$ & NP-hard \cite{virmaux2018lipschitz} \\
  \hline
\end{tabular}
 \vspace{-0.2in}
\end{table*}
\normalsize
\footnotetext{The model $x$ denotes the center of enclosing ball for MEB and the set of centers for $k$-means. For PCA, $x=W W^T$, where $W$ is a $d\times l$ matrix consisting of the first $l$ ($l<d$) principle components as columns. For SVM, 
$x_{1:d-1}\in \mathbb{R}^{d-1}$ is the coefficient vector and $x_d\in \mathbb{R}$ is the offset. }

In the distributed setting, we take the best-performing algorithm in the centralized setting ('RCC-kmeans') and evaluate its distributed extensions -- including CDCC \cite{Balcan13NIPS} and DRCC. 

\emph{\bf Datasets:} We use five real datasets: (1) Fisher's iris data \cite{FisherIris}, which is a $5$-dimensional dataset consisting of measurements of $150$ iris specimens and their species, (2) Facebook metrics \cite{FacebookMetrics}, which is a $19$-dimensional dataset consisting of features of $500$ posts published in a popular Facebook page, (3) Pendigits data \cite{Pendigits}, which is a $17$-dimensional dataset consisting of feature vectors of $7494$ handwritten digits, (4) MNIST data \cite{MNIST}, which consists of $60,000$ images of handwritten digits for training plus $10,000$ images for testing, each trimmed to $20\times 20$ pixels, and 
(5) Human Activity Recognition (HAR) using Smartphones data \cite{Anguita13ESANN}, which contains $10,299$ samples of smartphone sensor readings during $6$ different activities, each sample containing $561$ readings.

\begin{figure}[t]
\begin{minipage}{0.238\textwidth}
\centerline{
\includegraphics[width=\textwidth,height=3.3cm]{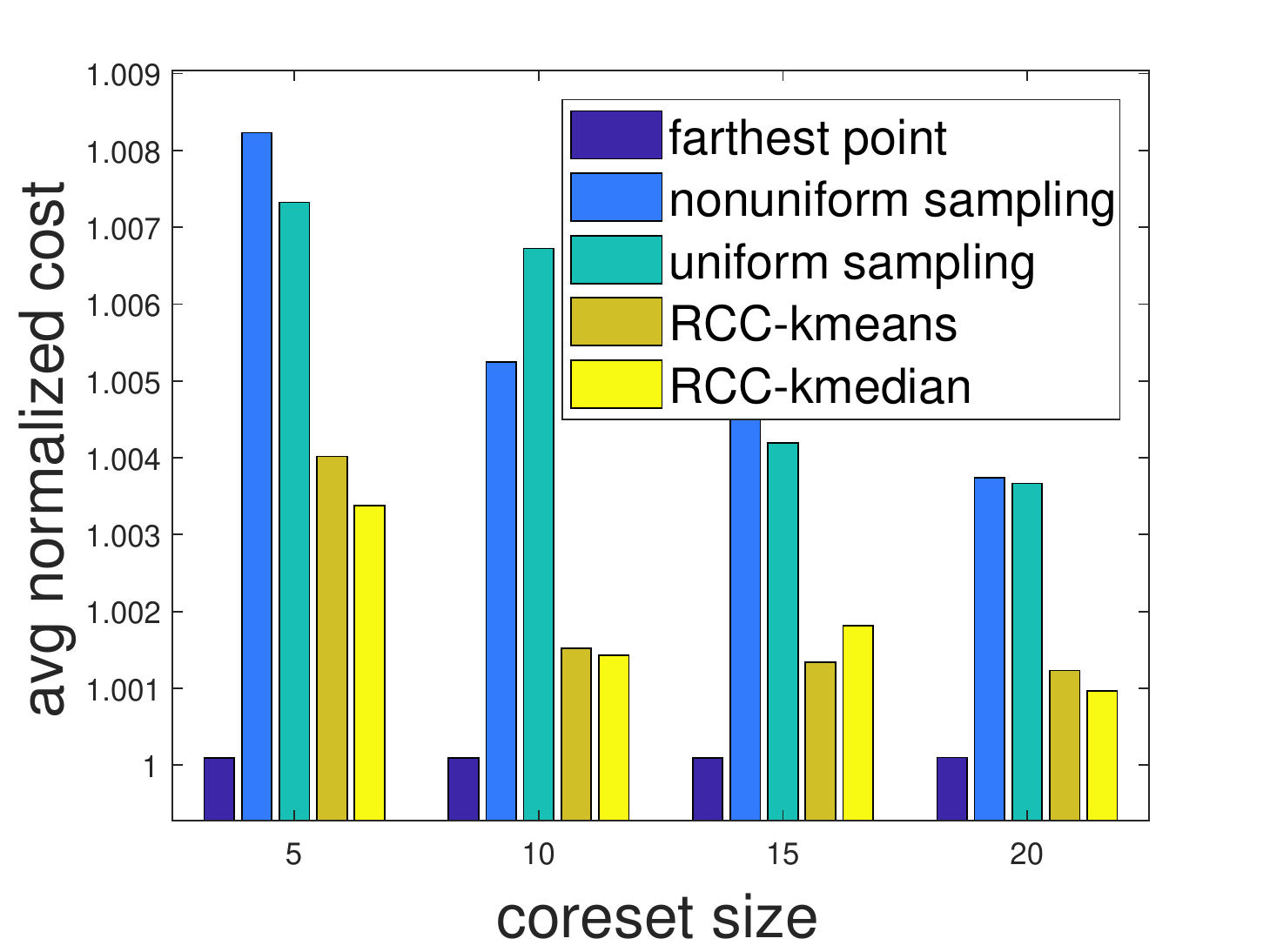}}
\centerline{\scriptsize (a) MEB}
\end{minipage}
\begin{minipage}{0.238\textwidth}
\centerline{
\includegraphics[width=\textwidth,height=3.3cm]{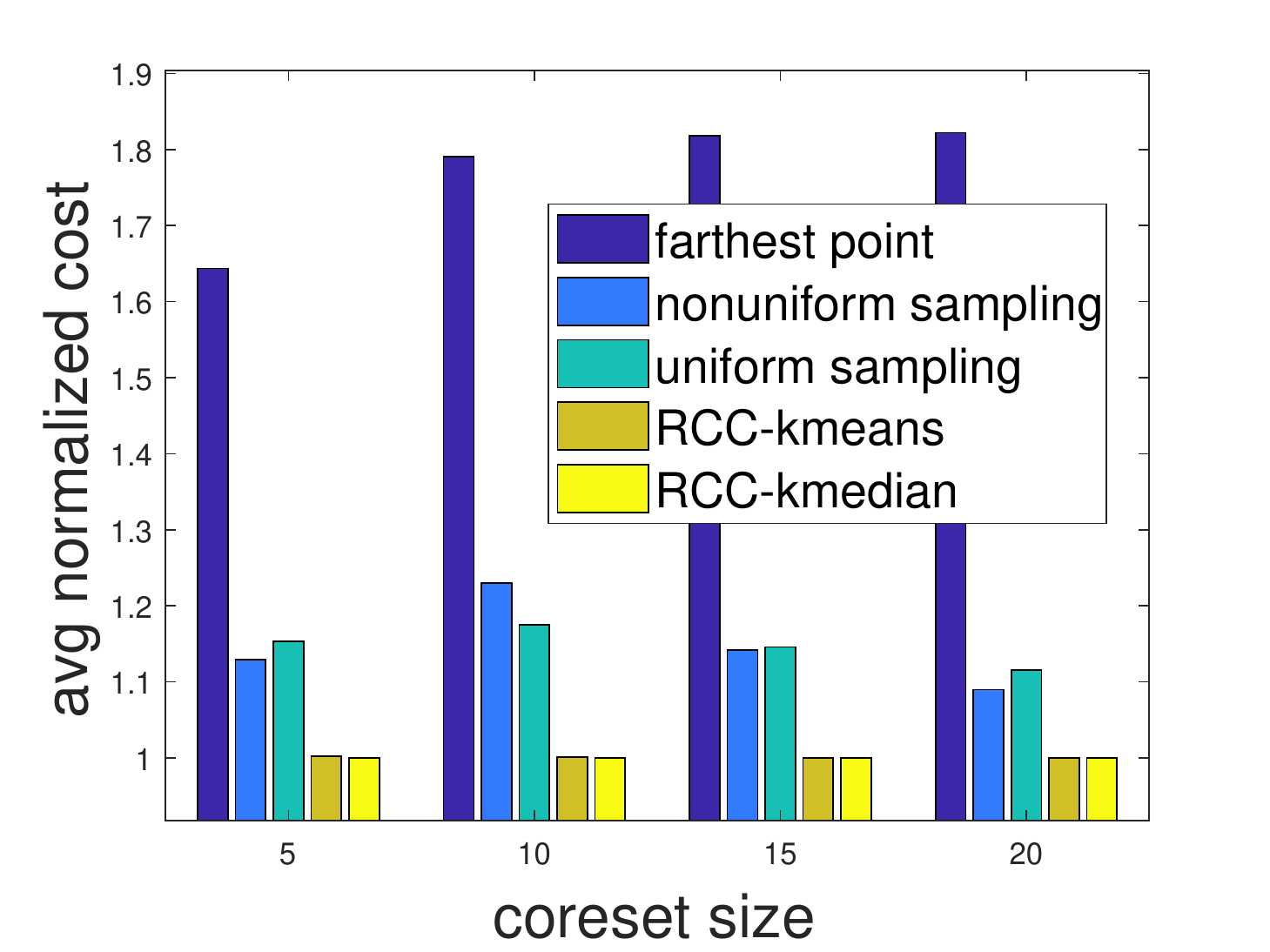}}
\centerline{\scriptsize (b) $k$-means ($k=2$)}
\end{minipage}
  \begin{minipage}{.238\textwidth}
  \centerline{
   \includegraphics[width=\textwidth,,height=3.3cm]{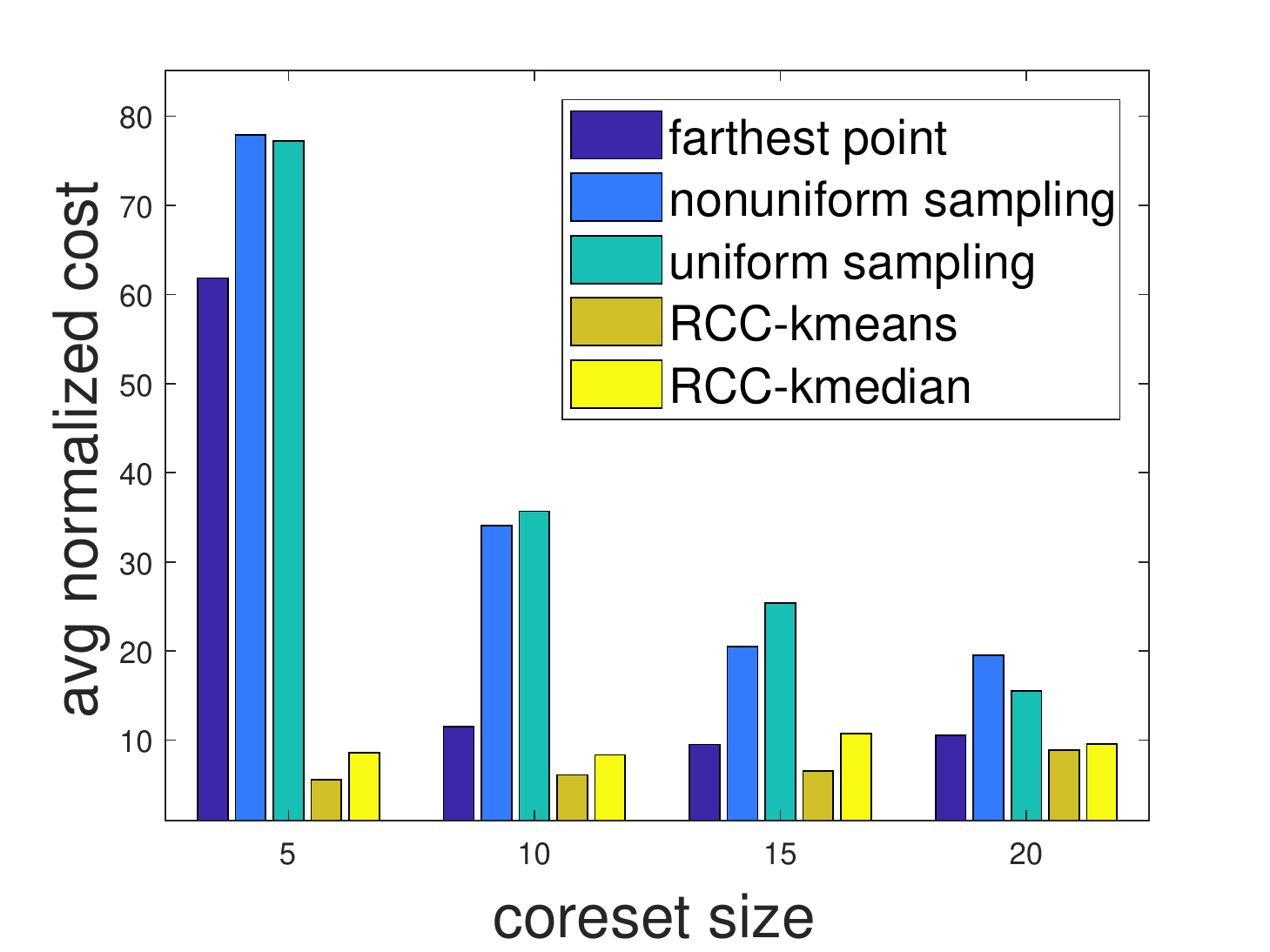}}
    \centerline{\scriptsize (c) PCA (3 components) }
  \end{minipage}
  \begin{minipage}{0.238\textwidth}
    \centerline{
  \includegraphics[width=\textwidth,height=3.3cm]{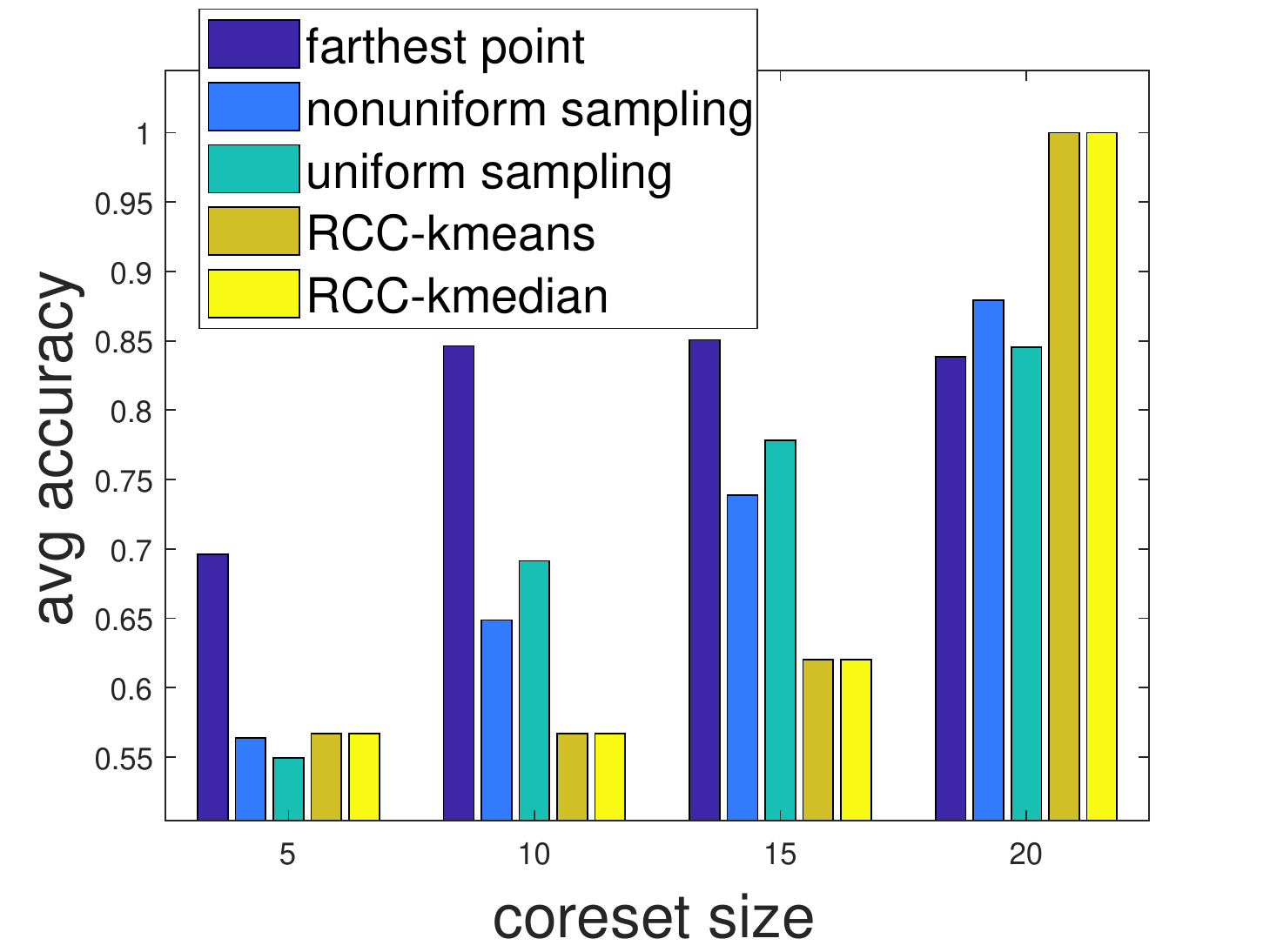}}
\centerline{\scriptsize (d) SVM (`setosa': 1; others: -1)}   
   \end{minipage}    
\caption{Evaluation on Fisher's iris dataset with varying coreset size (label: `species'). }
\label{fig:iris_size}
\end{figure}

\begin{figure}[t]
\begin{minipage}{0.238\textwidth}
\centerline{
\includegraphics[width=\textwidth,height=3.3cm]{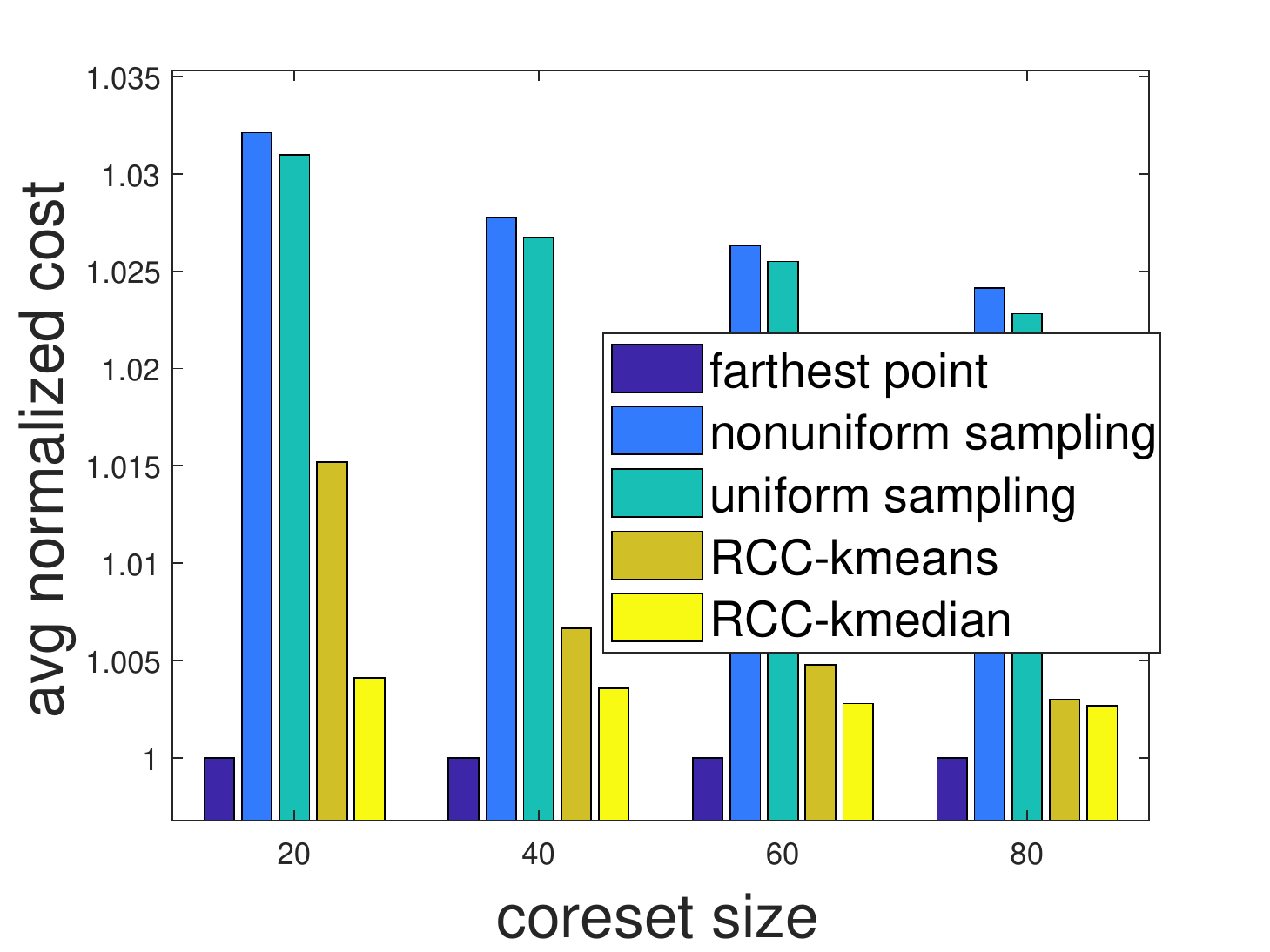}}
\centerline{\scriptsize (a) MEB}
\end{minipage}
\begin{minipage}{0.238\textwidth}
\centerline{
\includegraphics[width=\textwidth,height=3.3cm]{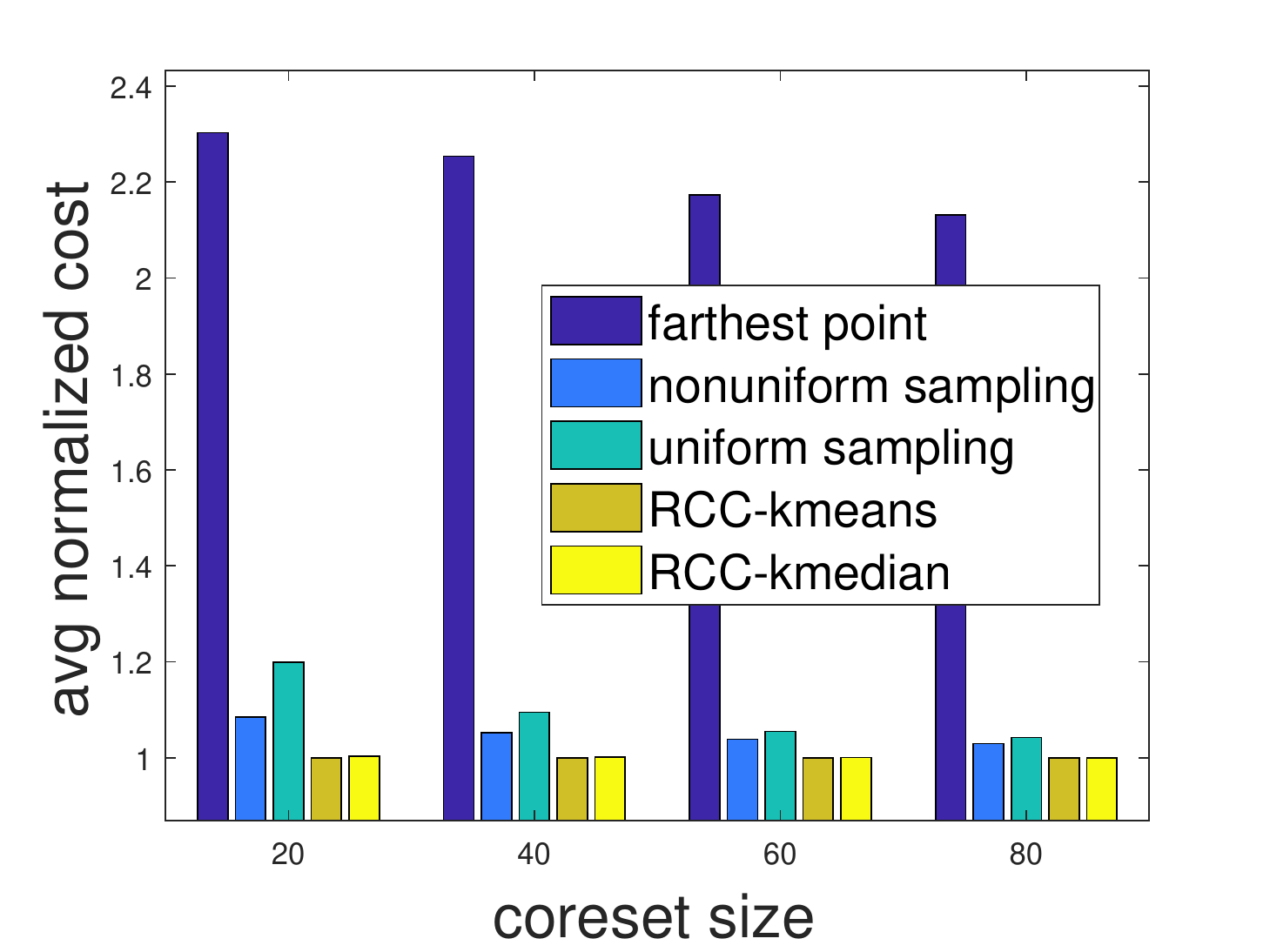}}
\centerline{\scriptsize (b) $k$-means ($k=2$)}
\end{minipage}
  \begin{minipage}{.238\textwidth}
  \centerline{
   \includegraphics[width=\textwidth,,height=3.3cm]{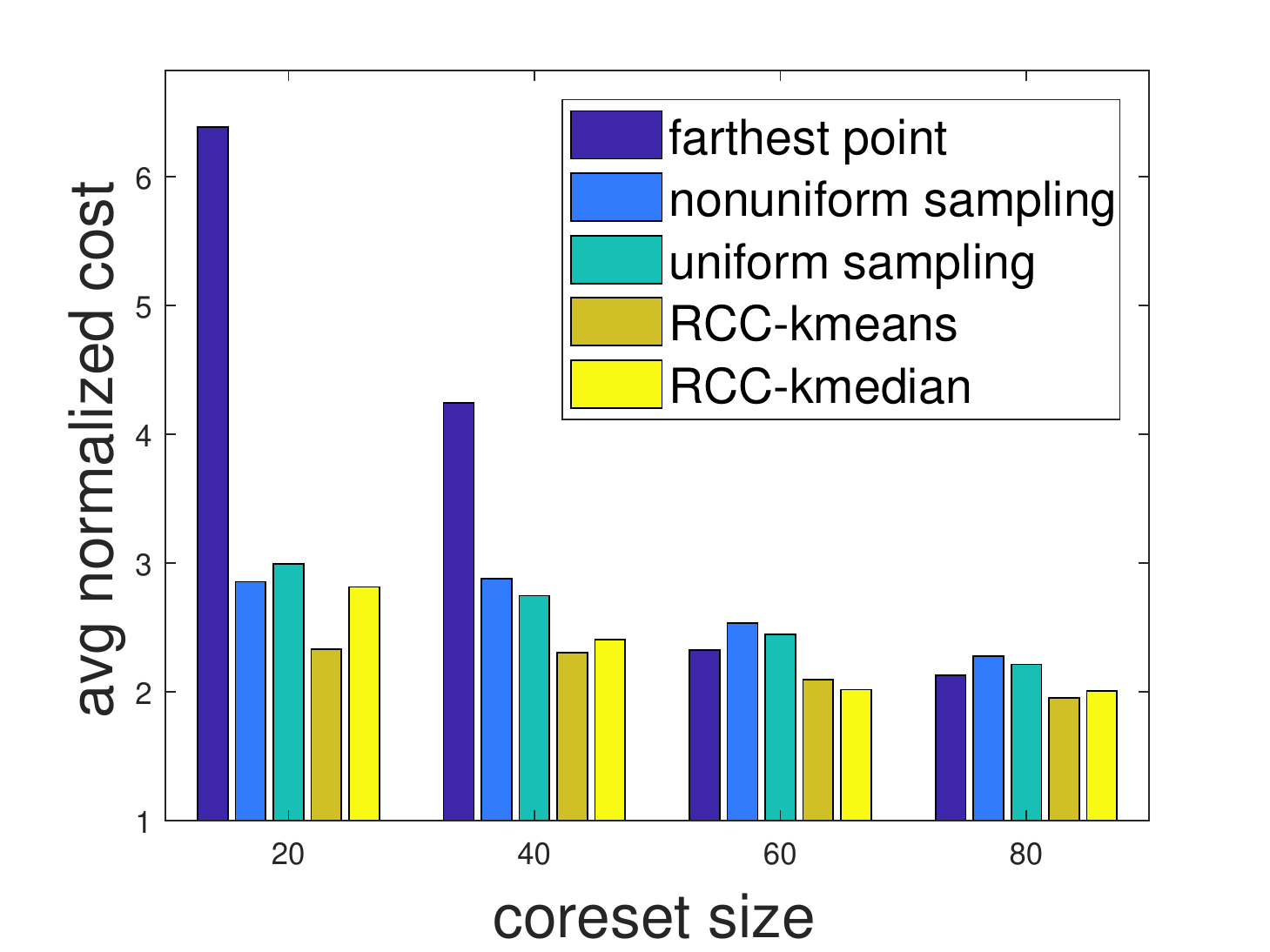}}
    \centerline{\scriptsize (c) PCA (5 components) }
  \end{minipage}
  \begin{minipage}{0.238\textwidth}
    \centerline{
  \includegraphics[width=\textwidth,height=3.3cm]{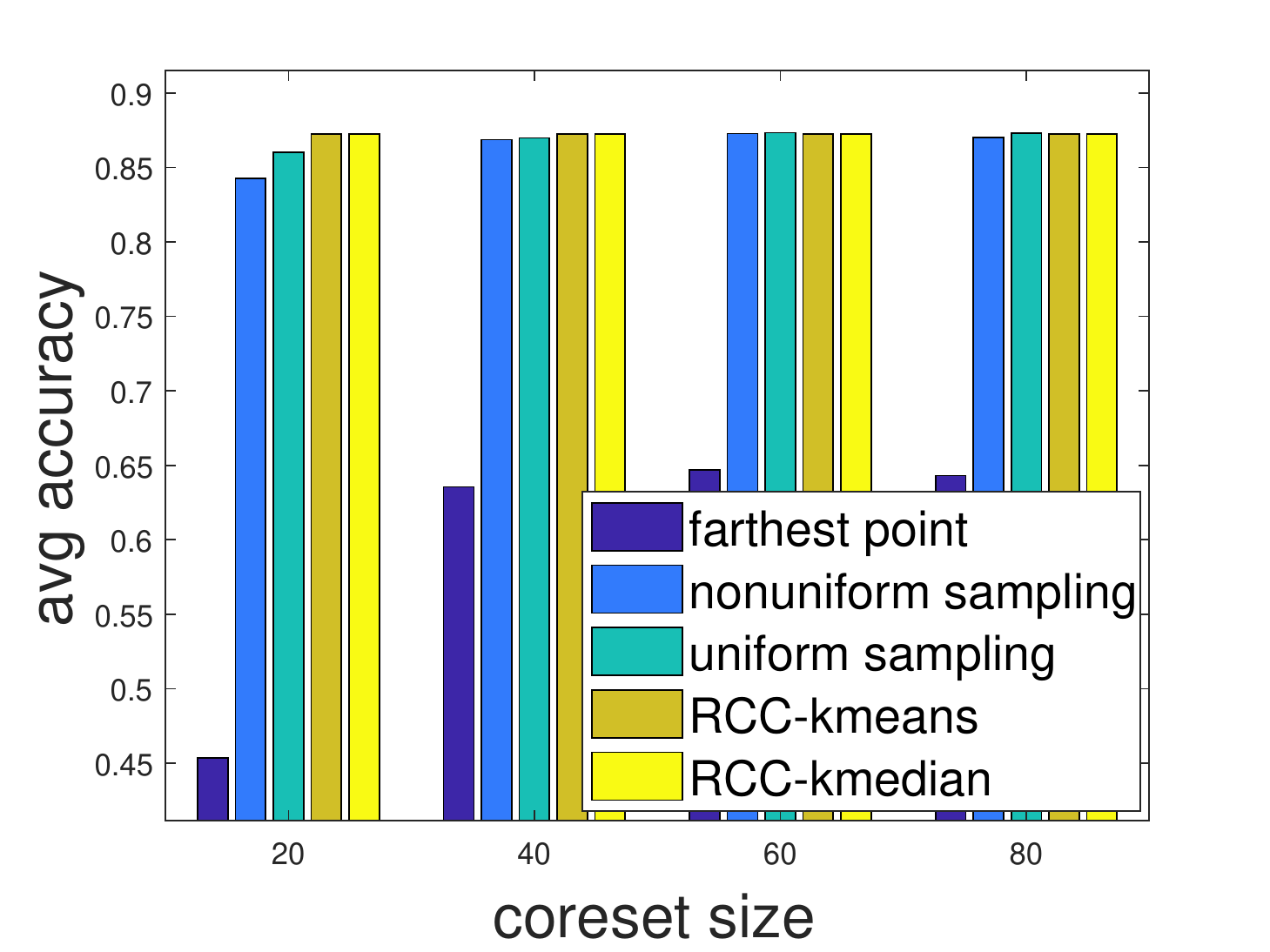}}
\centerline{\scriptsize (d) SVM (`photo': 1; others: -1)}   
   \end{minipage}    
\caption{Evaluation on Facebook metrics dataset with varying coreset size (label: `type'). }
\label{fig:facebook_size}
   \vspace{0em}
\end{figure}

\begin{figure}[tb]
\begin{minipage}{0.238\textwidth}
\centerline{
\includegraphics[width=\textwidth,height=3.3cm]{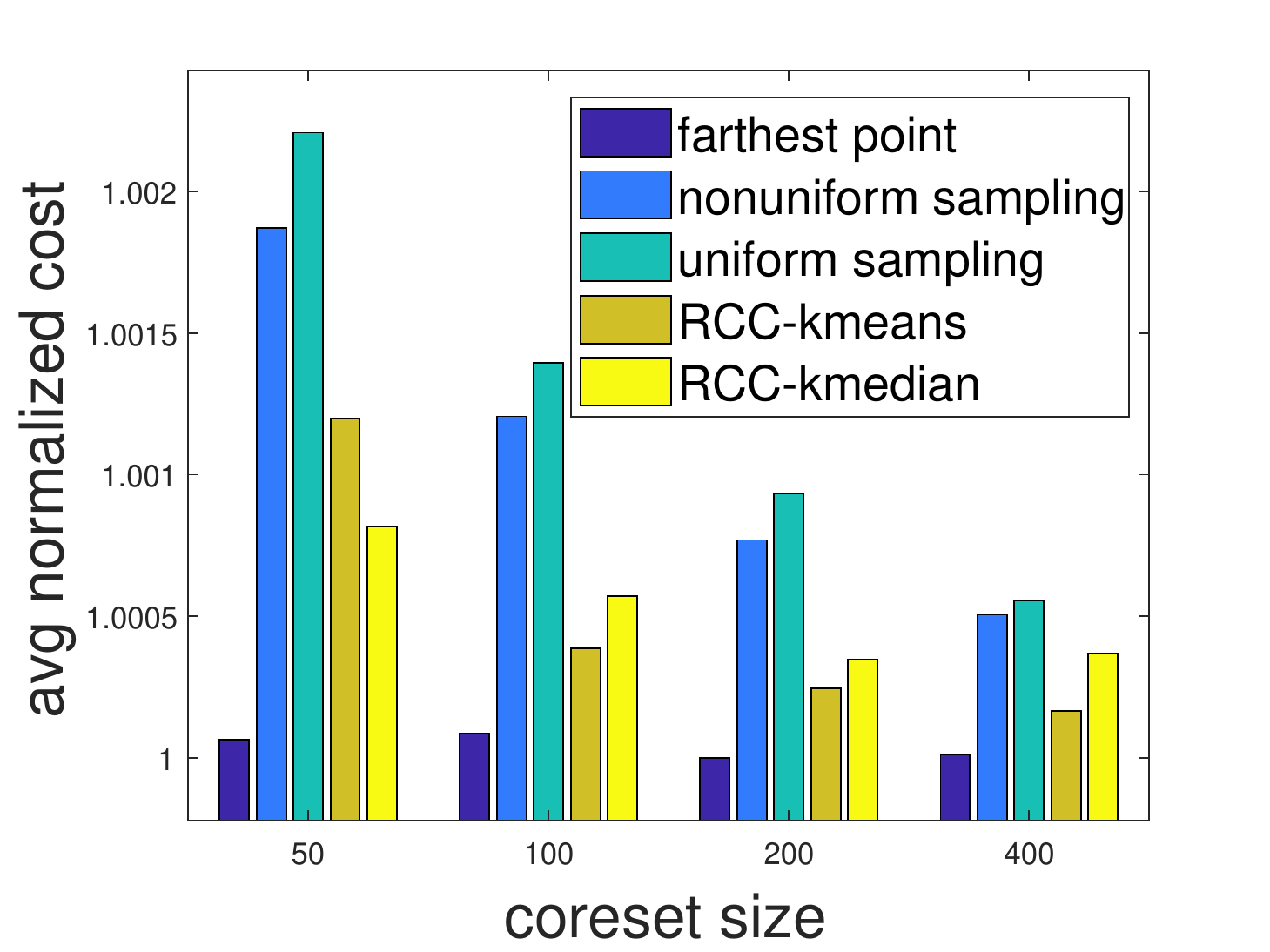}}
\centerline{\scriptsize (a) MEB}
\end{minipage}
\begin{minipage}{0.238\textwidth}
\centerline{
\includegraphics[width=\textwidth,height=3.3cm]{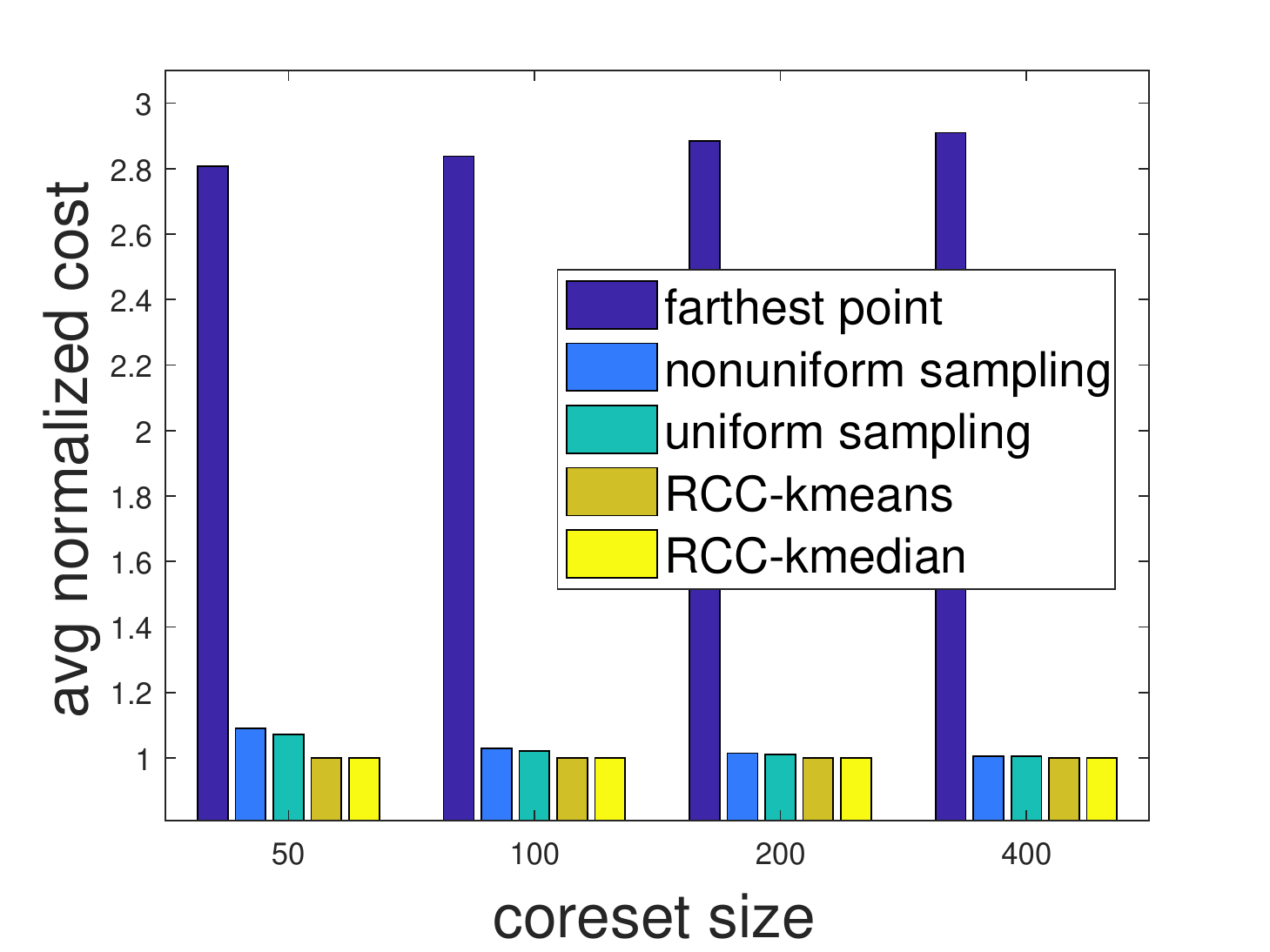}}
\centerline{\scriptsize (b) $k$-means ($k=2$)}
\end{minipage}
  \begin{minipage}{.238\textwidth}
  \centerline{
   \includegraphics[width=\textwidth,,height=3.3cm]{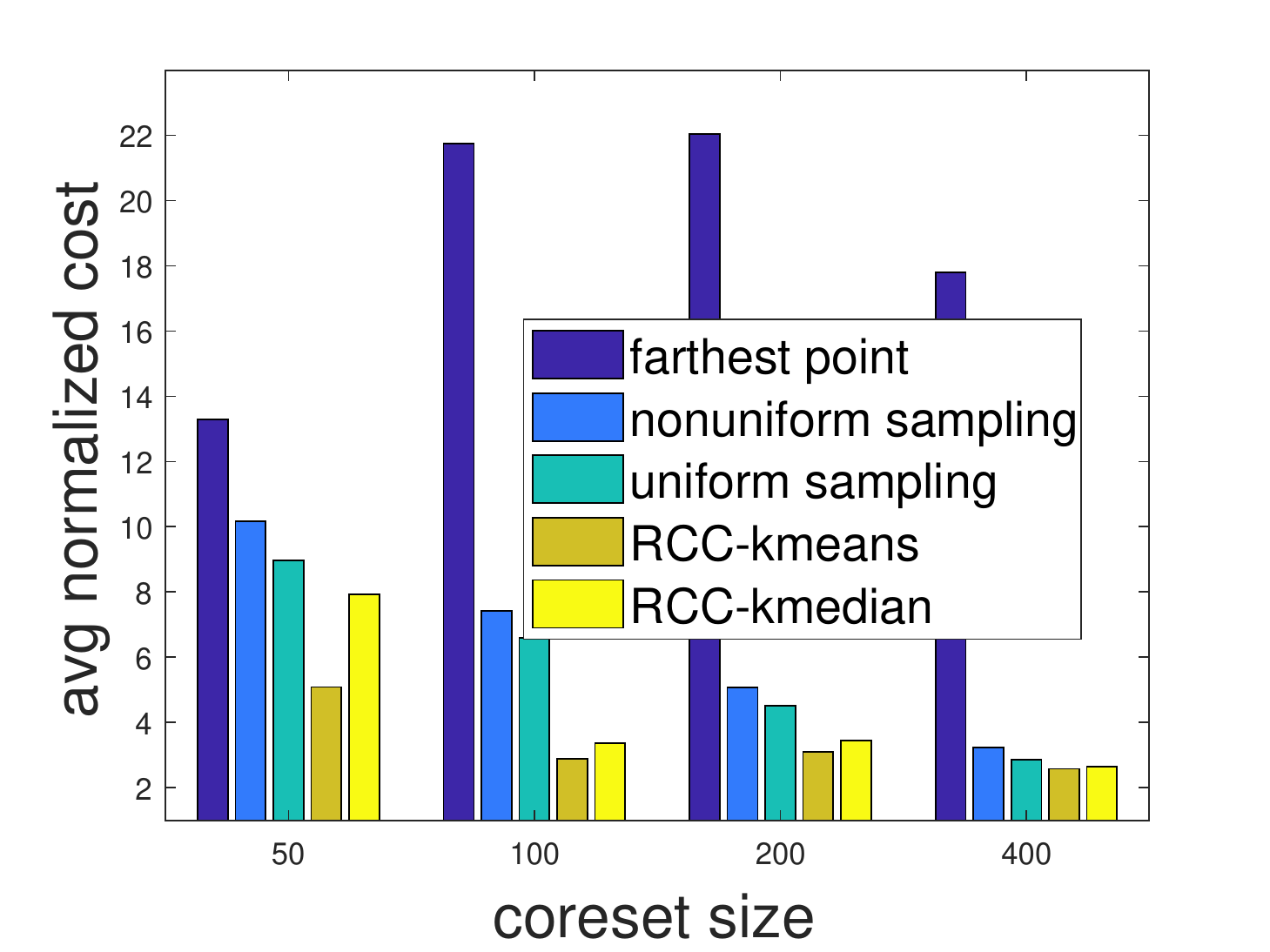}}
    \centerline{\scriptsize (c) PCA (11 components) }
  \end{minipage}
  \begin{minipage}{0.238\textwidth}
    \centerline{
  \includegraphics[width=\textwidth,height=3.3cm]{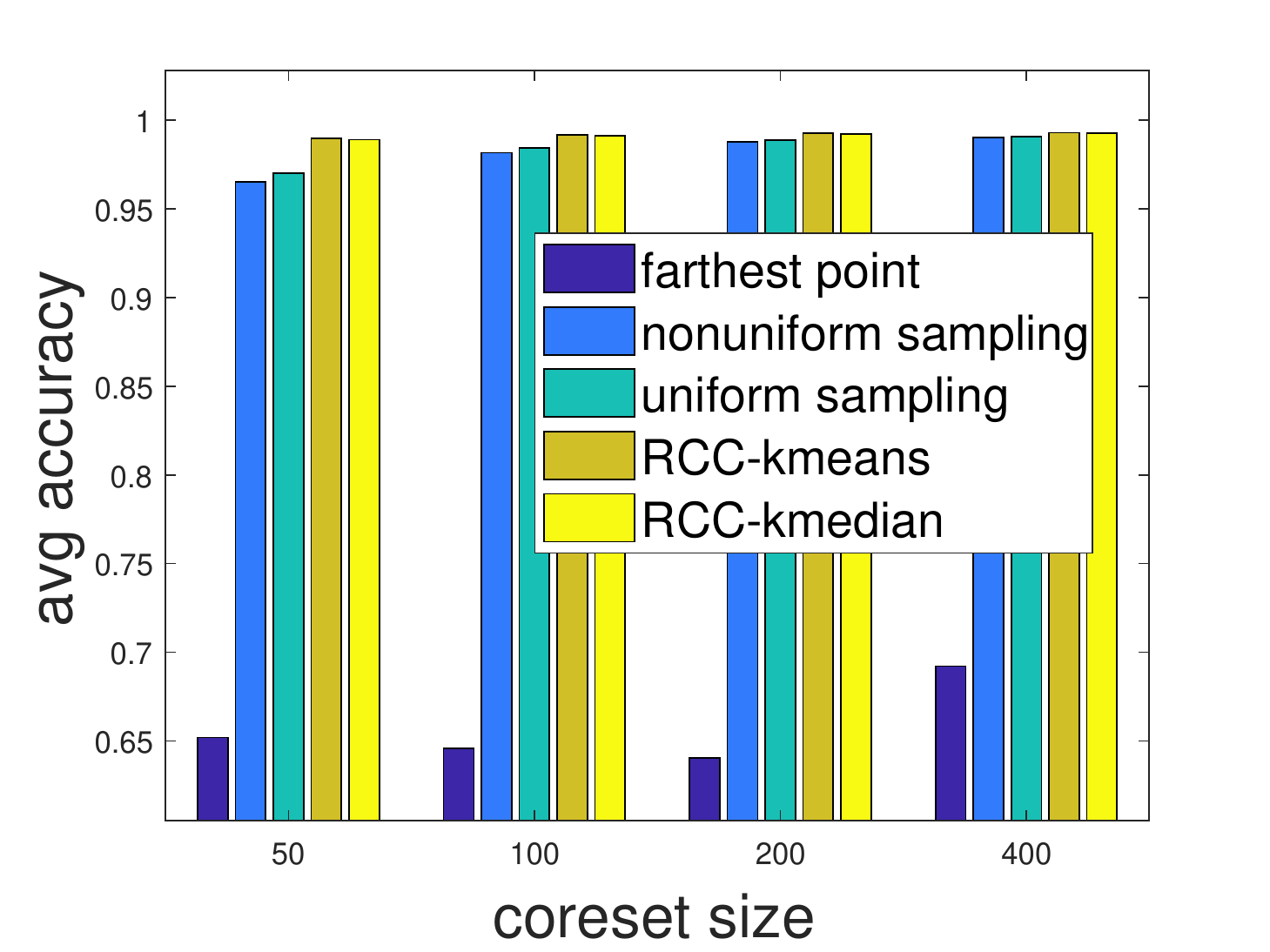}}
\centerline{\scriptsize (d) SVM (`0': 1; others: -1)}   
   \end{minipage}    
\caption{Evaluation on Pendigits with varying coreset size (label: `digit'). }
\label{fig:pendigits_size}
\end{figure}

\begin{figure}[tb]
\begin{minipage}{0.238\textwidth}
\centerline{
\includegraphics[width=\textwidth,height=3.3cm]{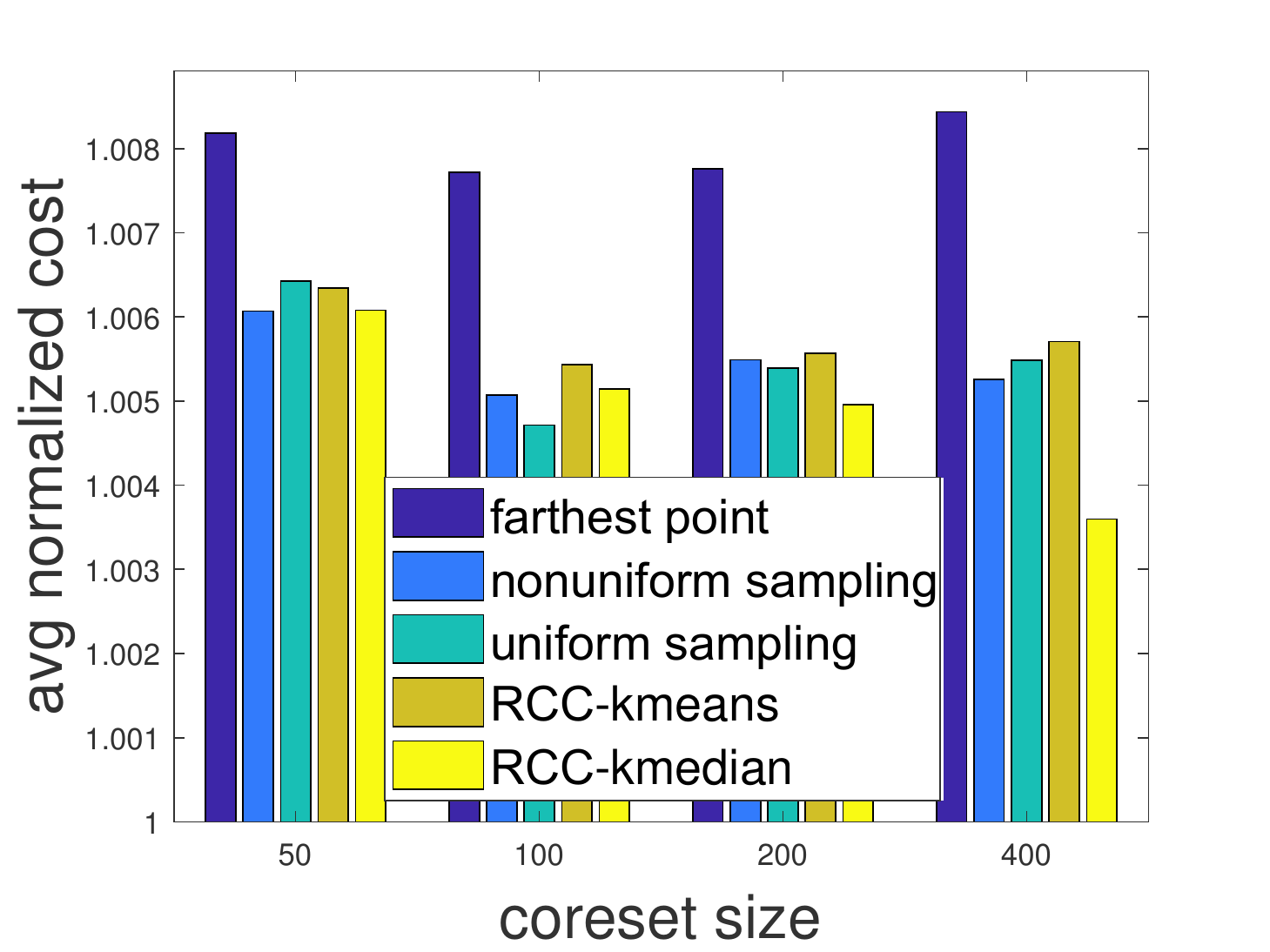}}
\centerline{\scriptsize (a) MEB}
\end{minipage}
\begin{minipage}{0.238\textwidth}
\centerline{
\includegraphics[width=\textwidth,height=3.3cm]{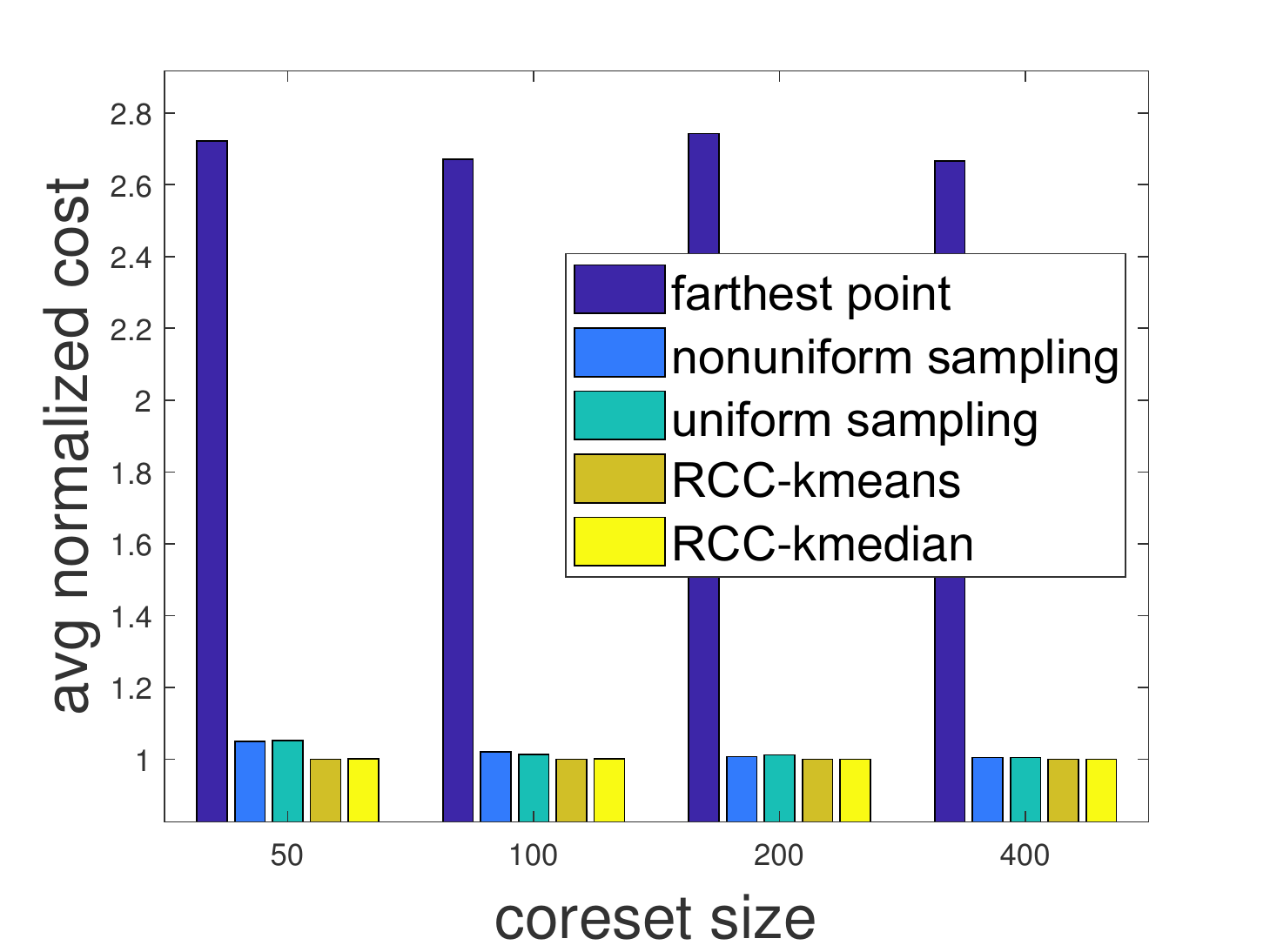}}
\centerline{\scriptsize (b) $k$-means ($k=2$)}
\end{minipage}
  \begin{minipage}{.238\textwidth}
  \centerline{
   \includegraphics[width=\textwidth,,height=3.3cm]{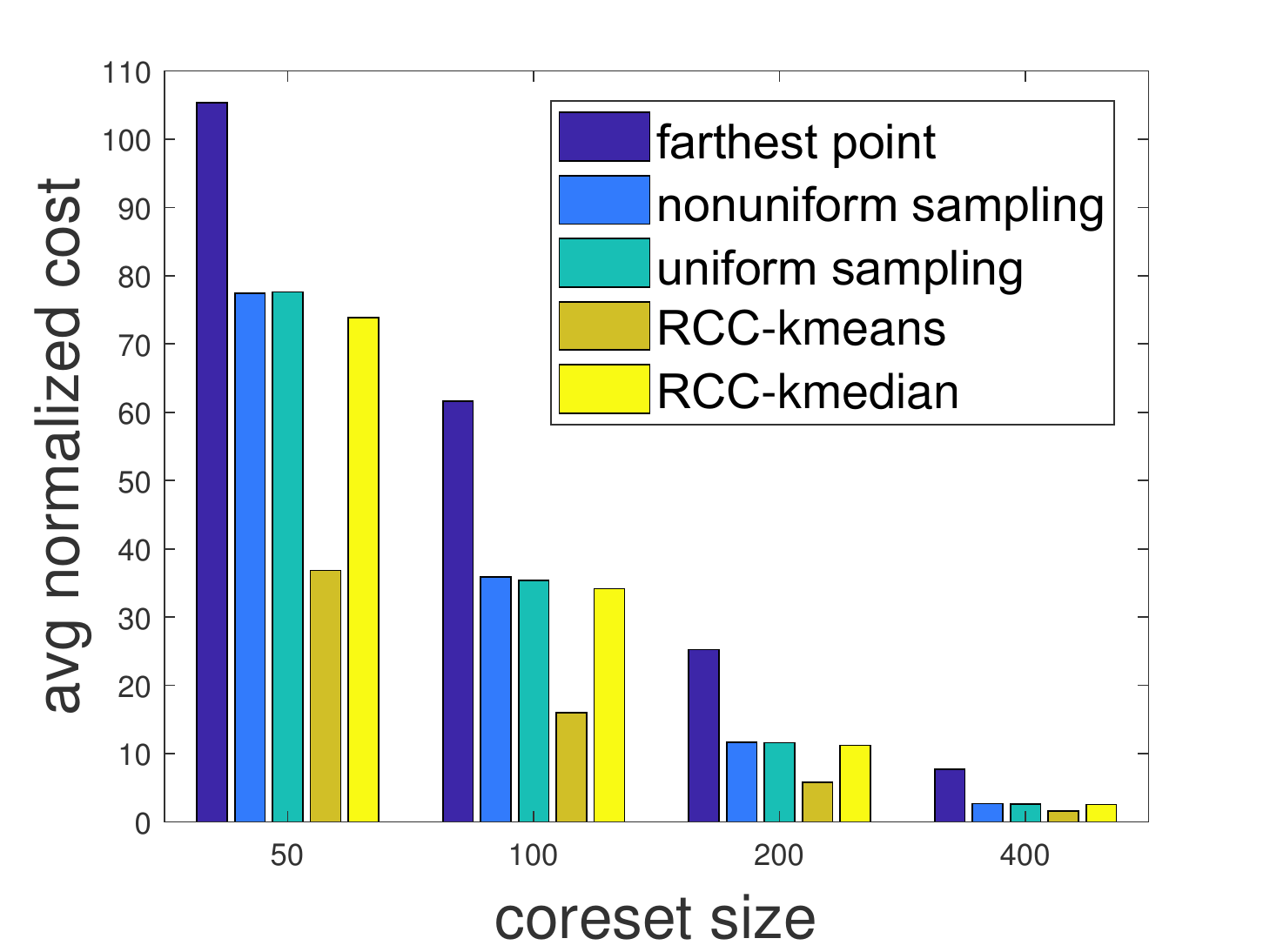}}
    \centerline{\scriptsize (c) PCA (300 components) }
  \end{minipage}
  \begin{minipage}{0.238\textwidth}
    \centerline{
  \includegraphics[width=\textwidth,height=3.3cm]{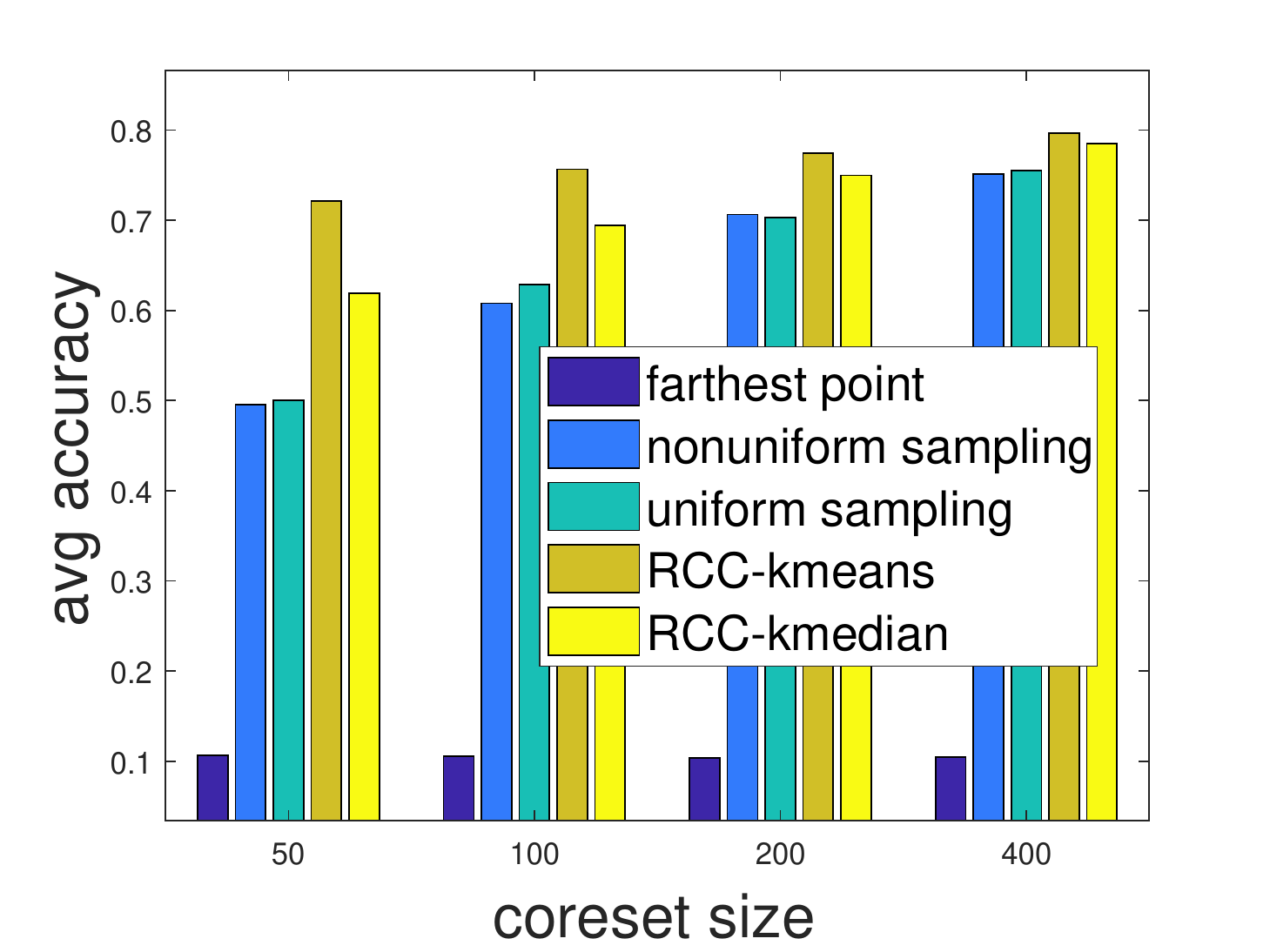}}
\centerline{\scriptsize (d) NN} 
   \end{minipage}    
\caption{Evaluation on MNIST with varying coreset size (label: `labels'). }
\label{fig:mnist_size}
   \vspace{0em}
\end{figure}

\begin{figure}[tb]
\begin{minipage}{0.238\textwidth}
\centerline{
\includegraphics[width=\textwidth,height=3.3cm]{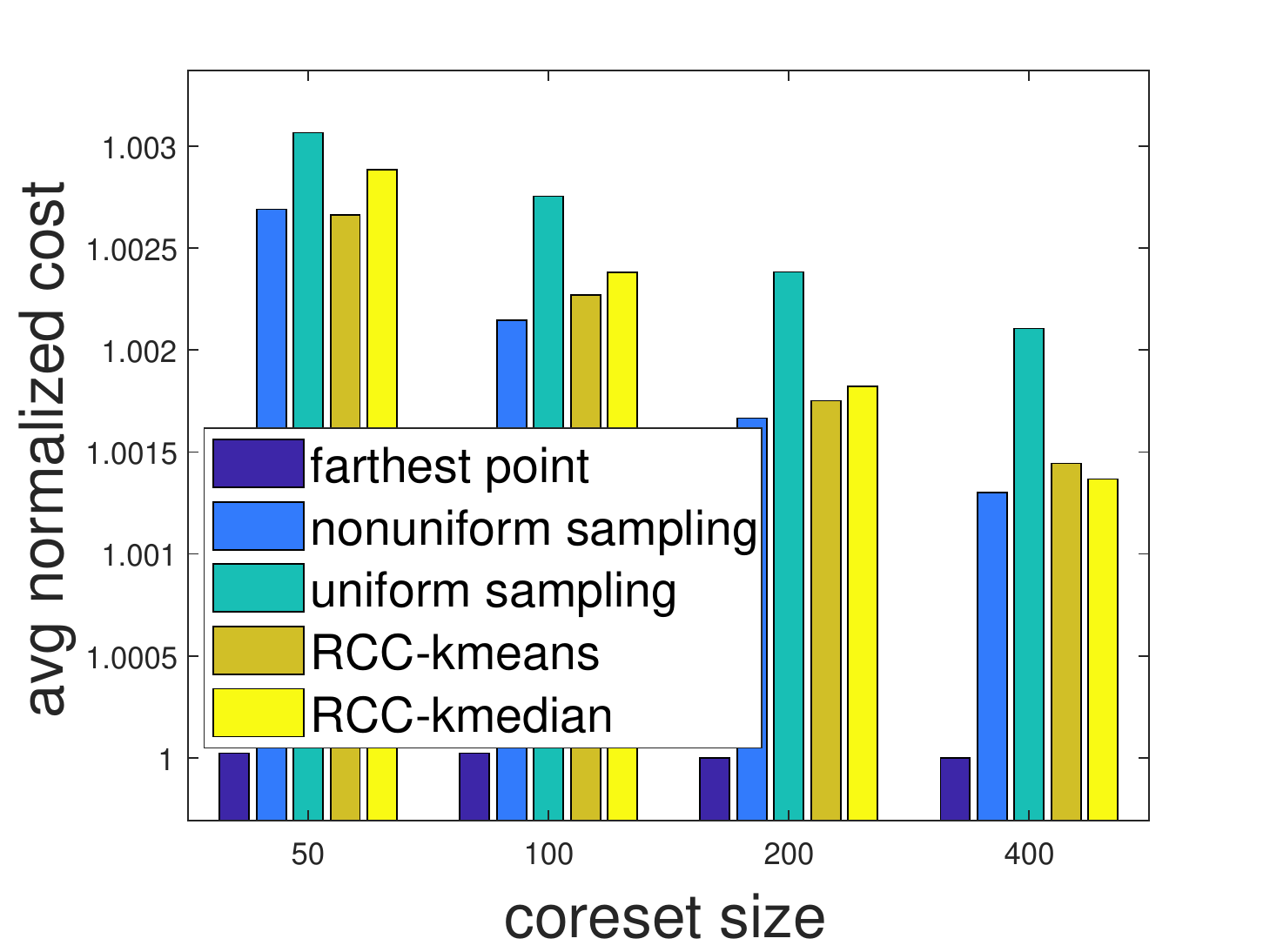}}
\centerline{\scriptsize (a) MEB}
\end{minipage}
\begin{minipage}{0.238\textwidth}
\centerline{
\includegraphics[width=\textwidth,height=3.3cm]{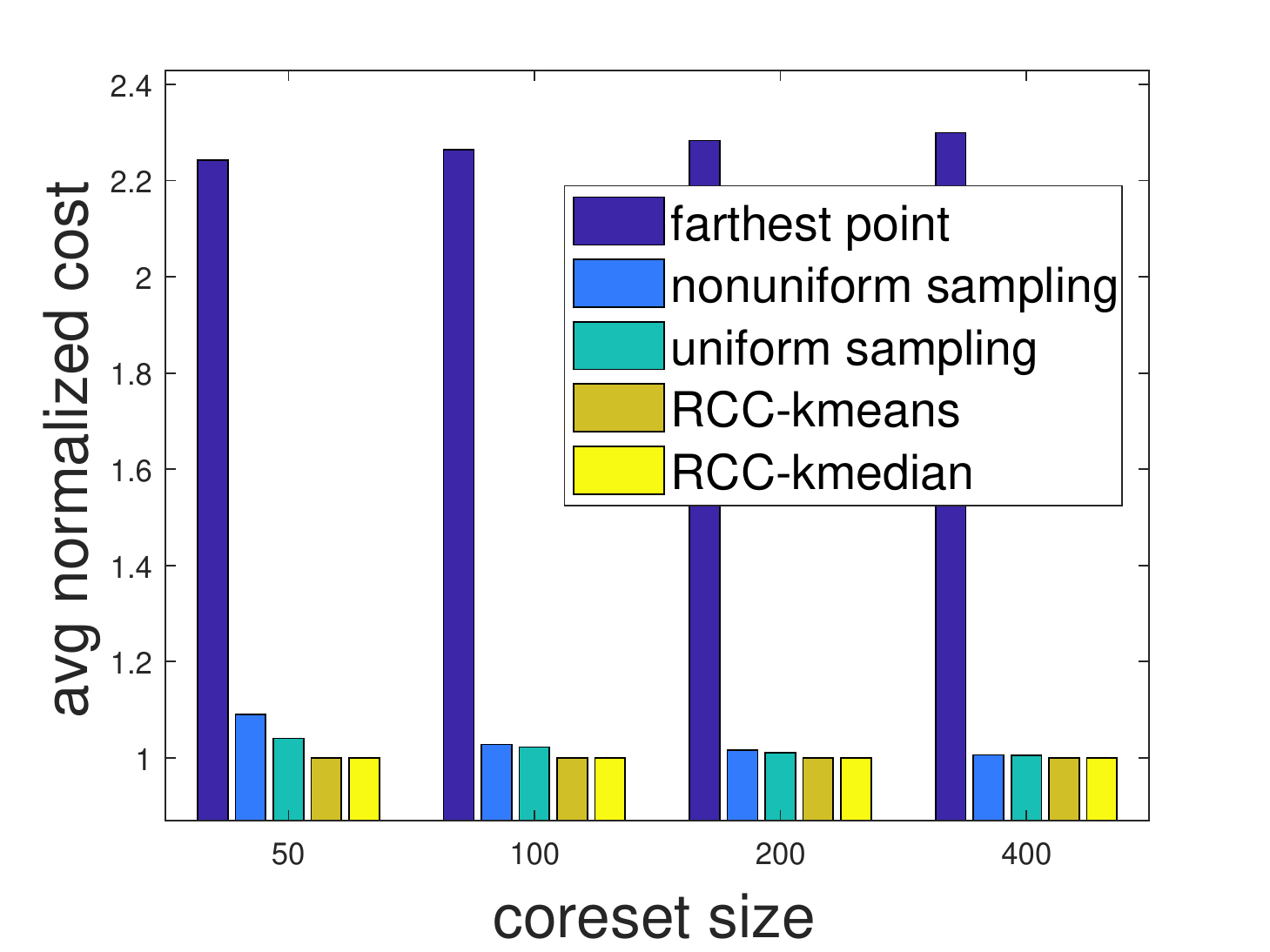}}
\centerline{\scriptsize (b) $k$-means ($k=2$)}
\end{minipage}
  \begin{minipage}{.238\textwidth}
  \centerline{
   \includegraphics[width=\textwidth,,height=3.3cm]{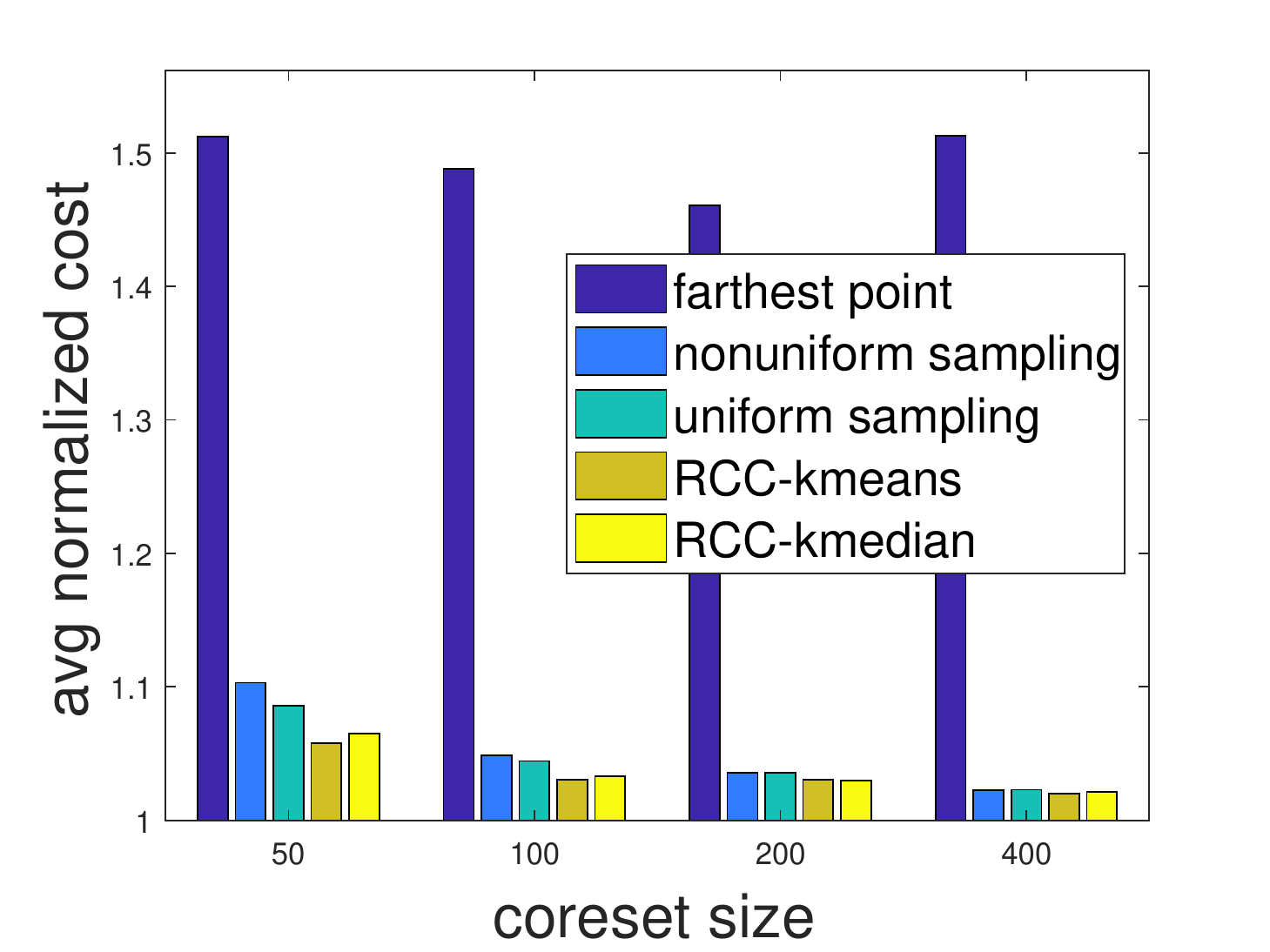}}
    \centerline{\scriptsize (c) PCA (7 components) }
  \end{minipage}
  \begin{minipage}{0.238\textwidth}
    \centerline{
  \includegraphics[width=\textwidth,height=3.3cm]{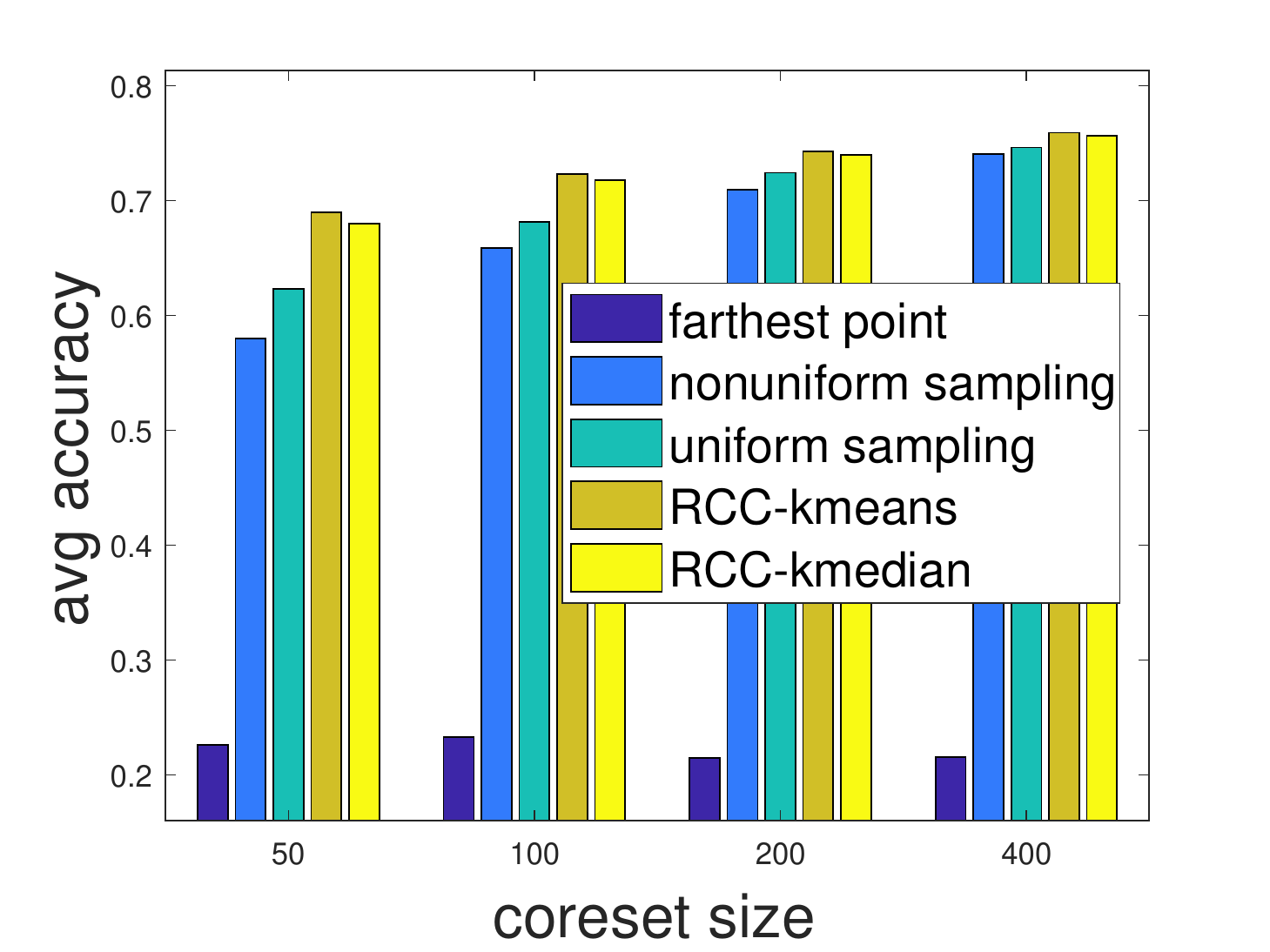}}
\centerline{\scriptsize (d) NN} 
   \end{minipage}    
\caption{Evaluation on HAR with varying coreset size (label: `labels'). }
\label{fig:har_size}
   \vspace{0em}
\end{figure}


\begin{figure}[tb]
\begin{minipage}{0.238\textwidth}
\centerline{
\includegraphics[width=\textwidth,height=3.3cm]{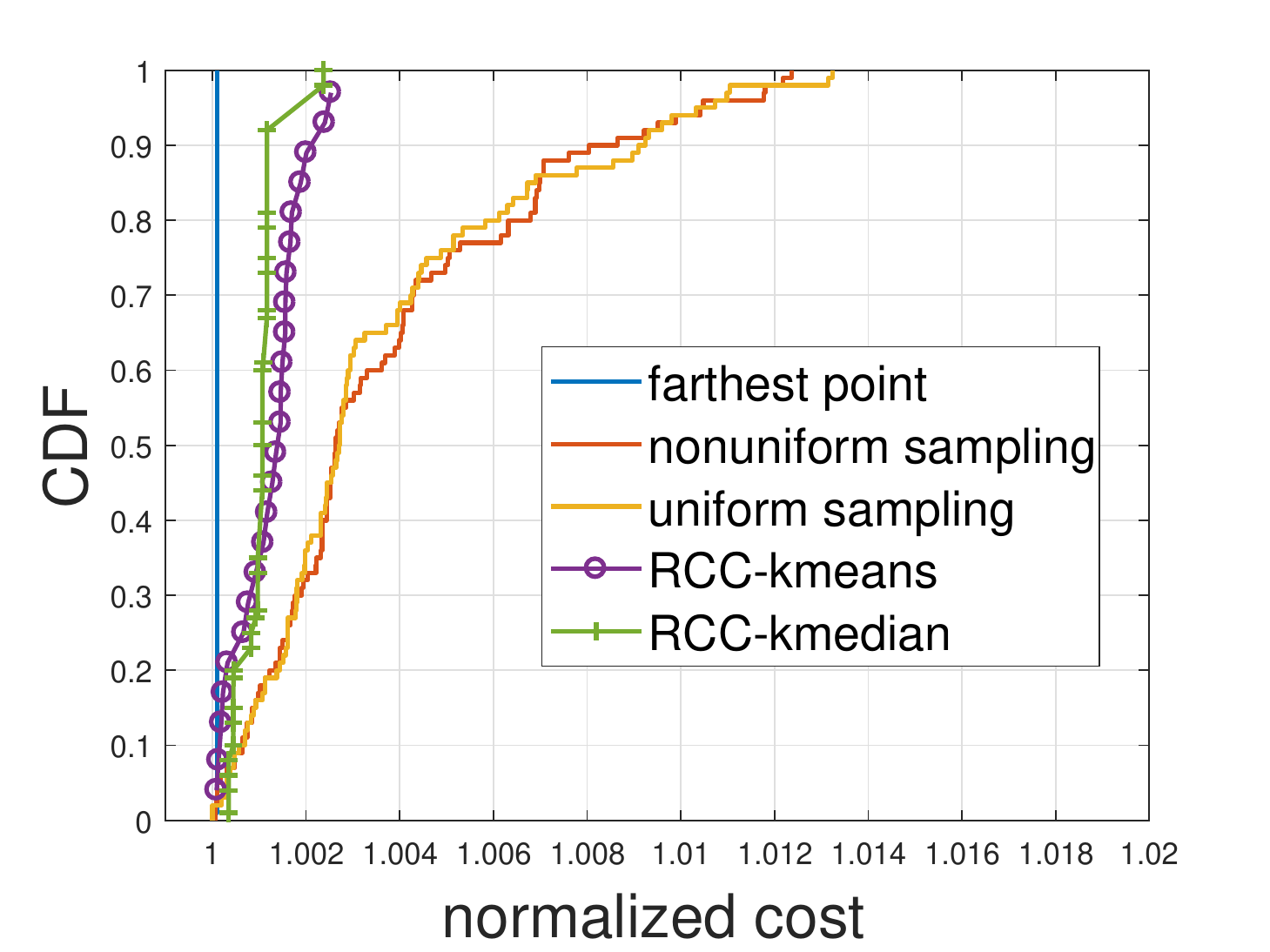}}
\centerline{\scriptsize (a) MEB}
\end{minipage}
\begin{minipage}{0.238\textwidth}
\centerline{
\includegraphics[width=\textwidth,height=3.3cm]{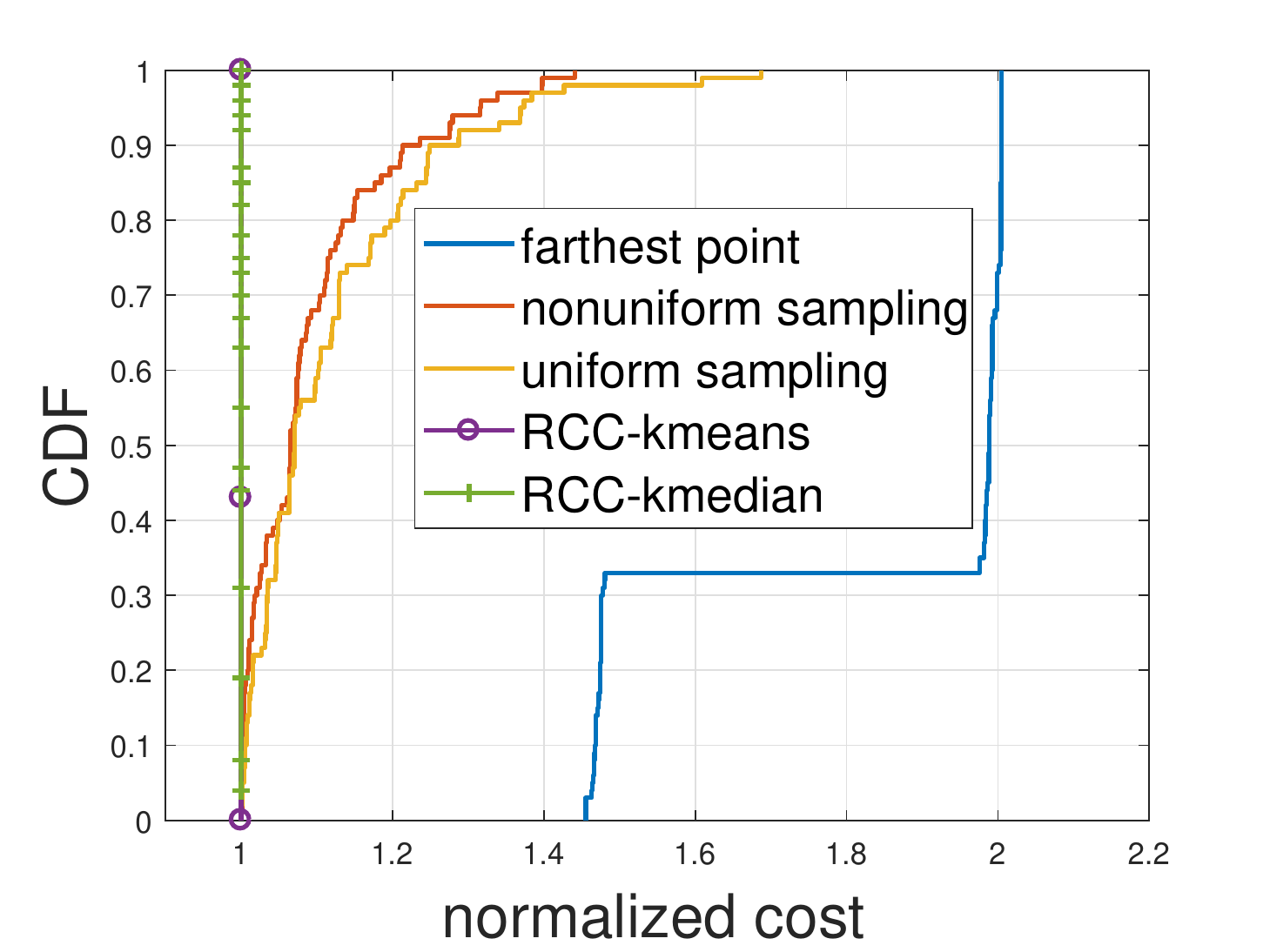}}
\centerline{\scriptsize (b) $k$-means ($k=2$)}
\end{minipage}
  \begin{minipage}{.238\textwidth}
  \centerline{
  \includegraphics[width=\textwidth,,height=3.3cm]{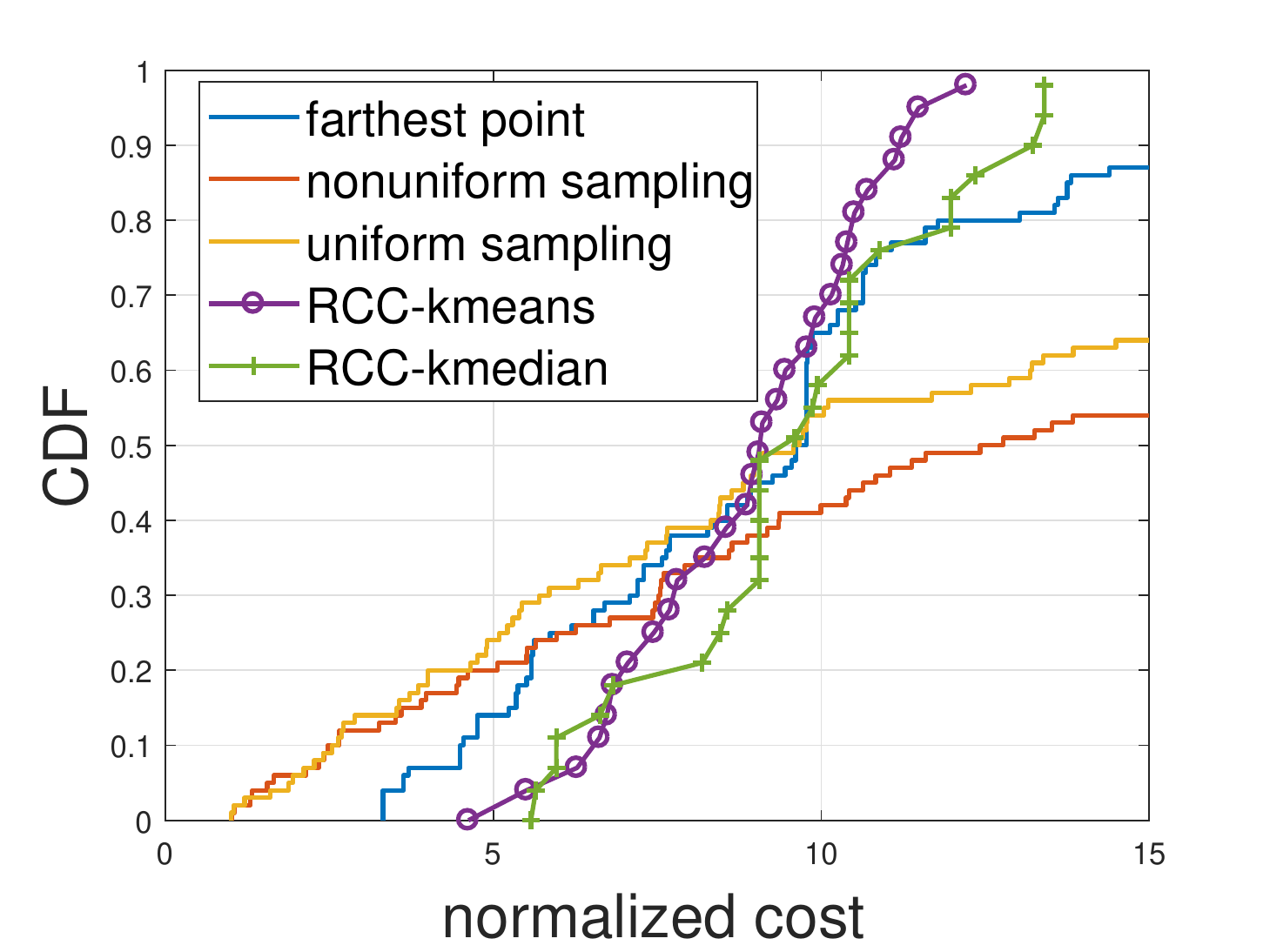}}
    \centerline{\scriptsize (c) PCA (3 components)}
  \end{minipage}
  \begin{minipage}{0.238\textwidth}
    \centerline{
  \includegraphics[width=\textwidth,height=3.3cm]{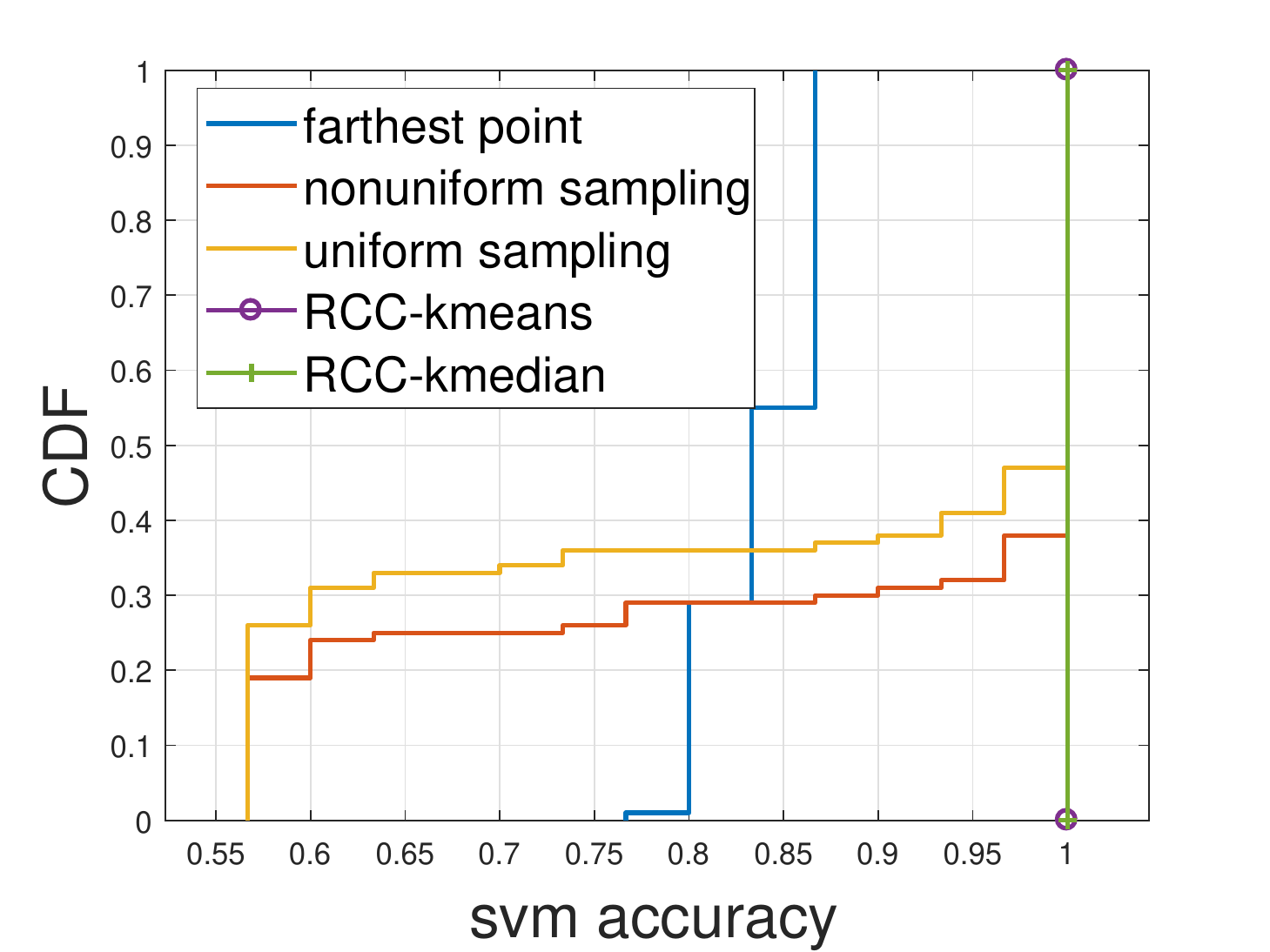}}
\centerline{\scriptsize (d) SVM (`setosa': 1; others: -1)}   
  \end{minipage}    
\caption{Detailed evaluation on Fisher's iris dataset (label: `species', coreset size: $20$).}
\label{fig:fisher}
\end{figure}

\begin{figure}[tb]
\begin{minipage}{0.238\textwidth}
\centerline{
\includegraphics[width=\textwidth,height=3.6cm]{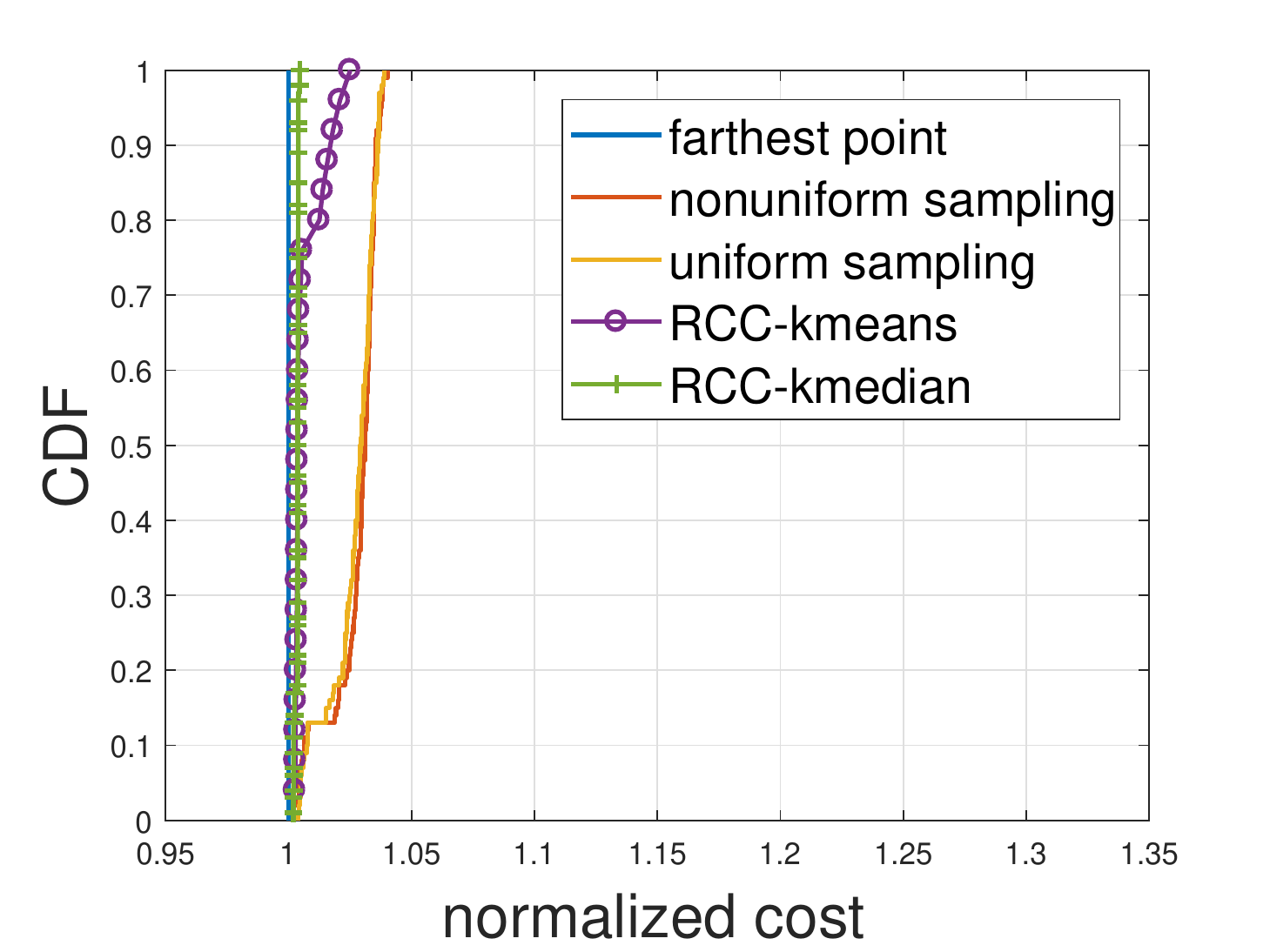}}
\centerline{\scriptsize (a) MEB}
\end{minipage}
\begin{minipage}{0.238\textwidth}
\centerline{
\includegraphics[width=\textwidth,height=3.6cm]{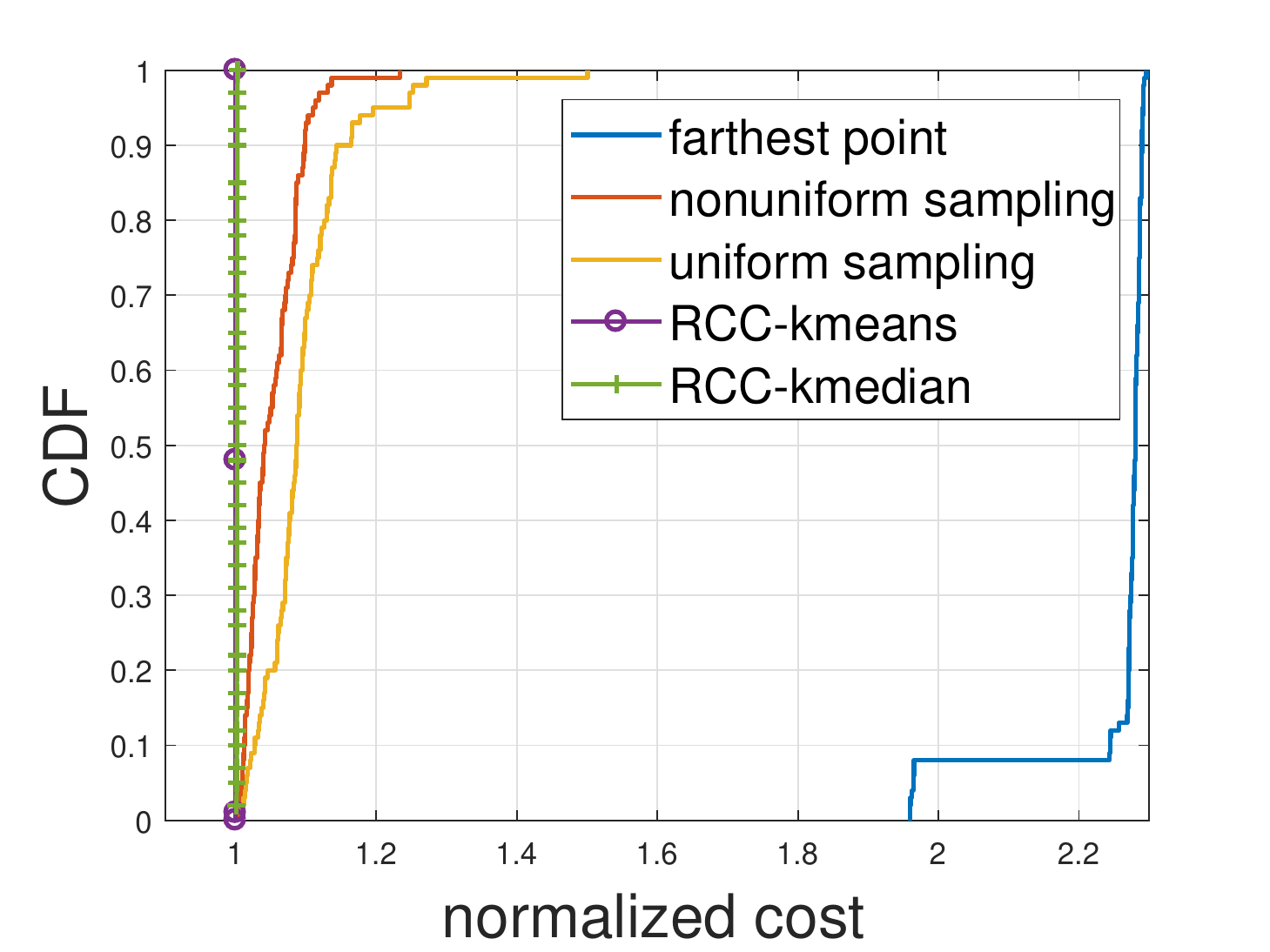}}
\centerline{\scriptsize (b) $k$-means ($k=2$)}
\end{minipage}
  \begin{minipage}{.238\textwidth}
  \centerline{
  \includegraphics[width=\textwidth,,height=3.6cm]{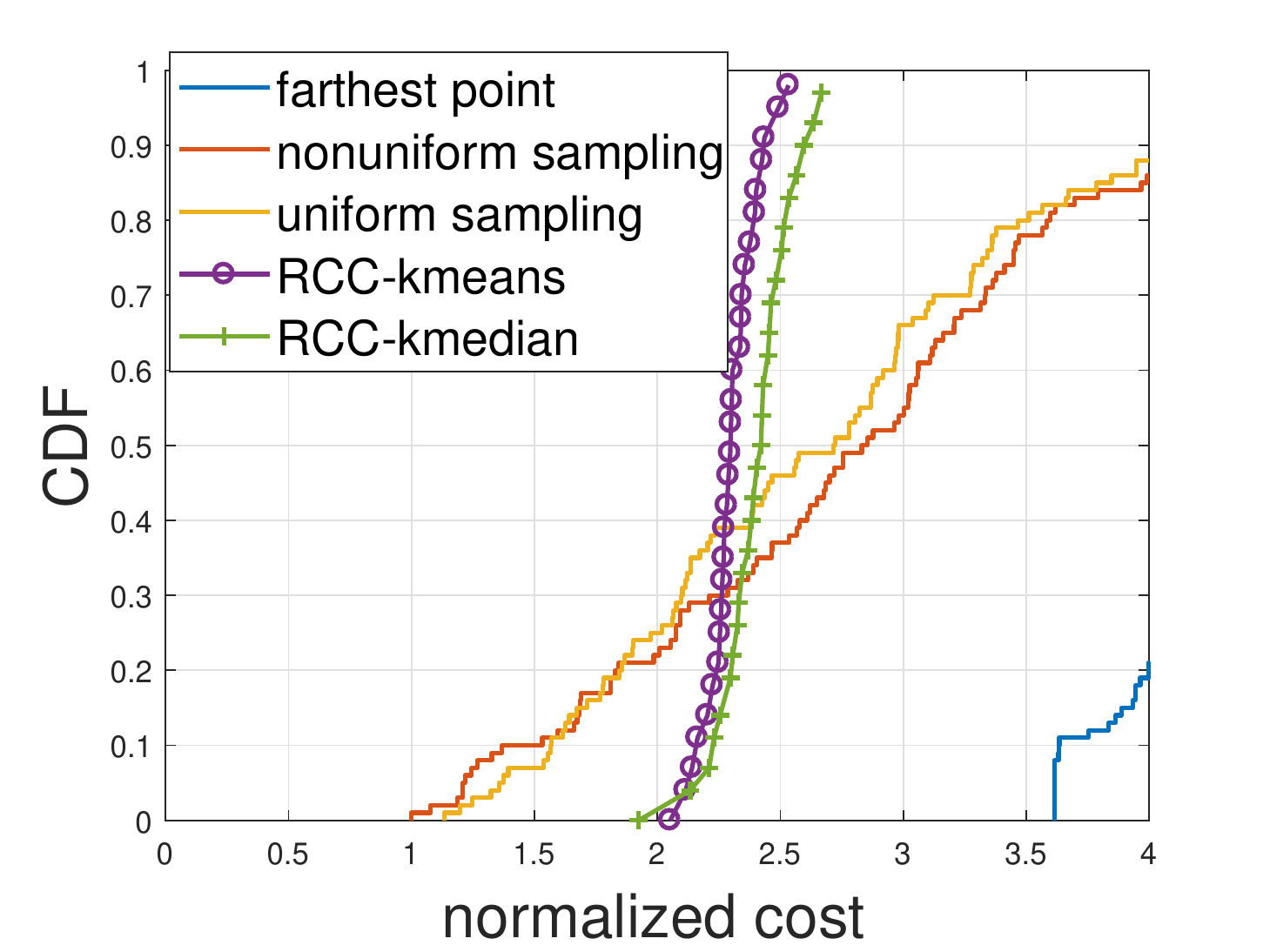}}
    \centerline{\scriptsize (c) PCA (5 components) }
  \end{minipage}
  \begin{minipage}{0.238\textwidth}
    \centerline{
  \includegraphics[width=\textwidth,height=3.6cm]{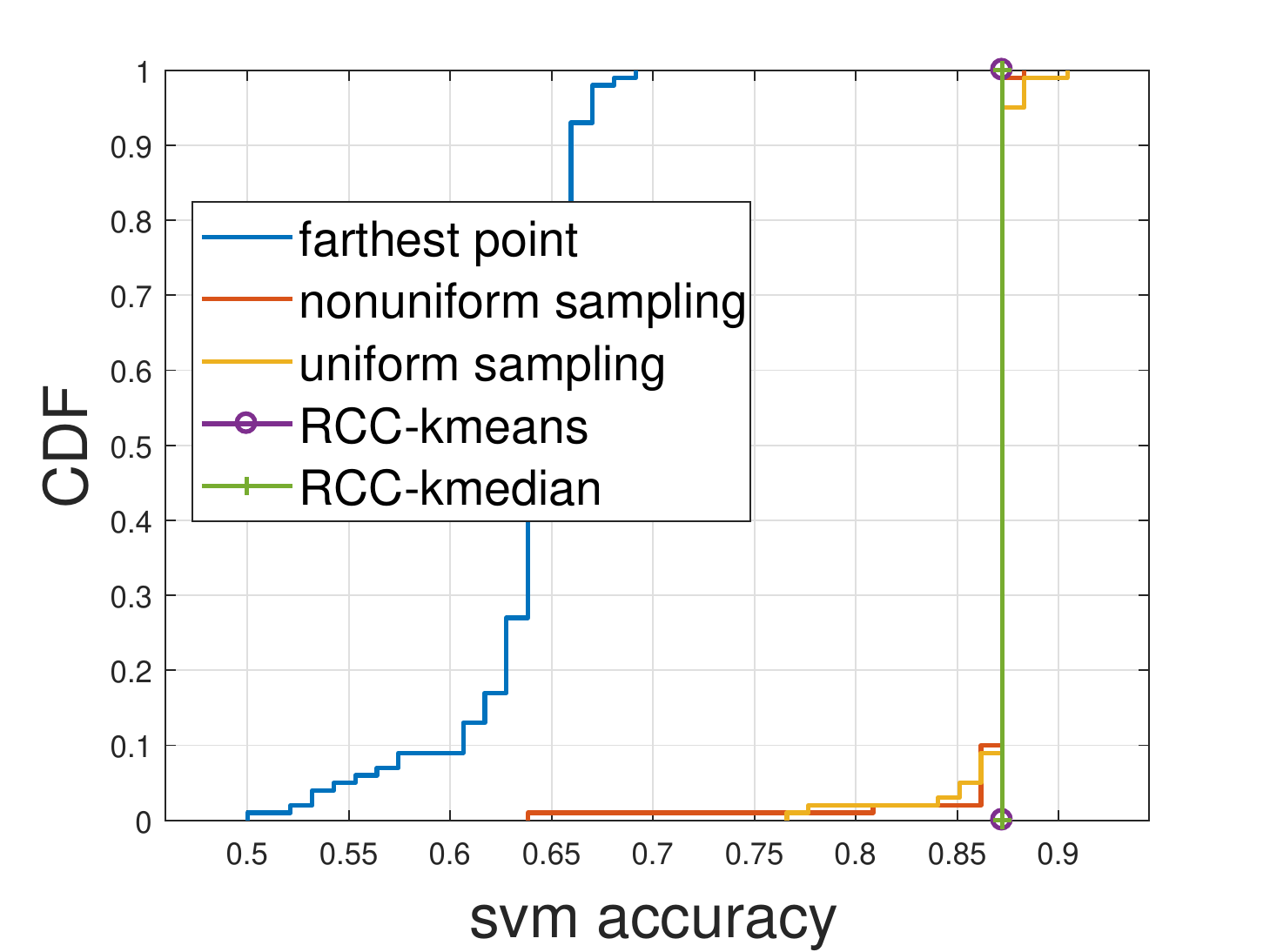}}
\centerline{\scriptsize (d) SVM (`photo': 1; others: -1) }   
  \end{minipage}    
\caption{Detailed evaluation on Facebook metrics dataset (label: `type', coreset size: $40$).}
\label{fig:facebook}
\end{figure}

We normalize each numerical dimension to $[0,1]$. We map labels to numbers such that the distance between two points with different labels is no smaller than the distance between points with the same label.
Given a $d$-dimensional dataset (including labels) with $L$ types of labels, we map type-$l$ label to $(l-1)\tau$ ($l\in [L]$) for $\tau = \lceil \sqrt{d-1} \rceil$. See Table~\ref{tab:dataset} for a summary. 
In testing SVM, we map one label to `1' and the rest to `-1'. 
Each data point has a unit weight. 

To generate distributed datasets, we use three schemes: (i) \emph{uniform}, where the points are uniformly distributed across $n$ nodes, (ii) \emph{specialized}, where each node is associated with one label and contains all the data points with this label, and (iii) \emph{hybrid}, where the first $n_0$ nodes are ``specialized'' as in (ii), and the remaining data are randomly partitioned among the remaining nodes. \looseness=-1

\emph{\bf Machine learning problems:} We evaluate three unsupervised learning problems---MEB, $k$-means, and PCA, and two supervised learning problem---SVM and Neural Network (NN). In our NN experiment, we define a three-layer network, whose hidden layer has 100 neurons. As an NN of this size is typically used on datasets with at least $100$ features, we only evaluate NN on MNIST and HAR, and replace it by SVM for the other datasets.
Table~\ref{tab:cost} gives their cost functions, where for a data point $p\in \mathbb{R}^d$, $p_{1:d-1}\in \mathbb{R}^{d-1}$ denotes the numerical portion and $p_d\in \mathbb{R}$ denotes the label. The meaning of the model parameter $x$ is problem-specific, as explained in the footnote. 
We also provide (upper bounds of) the Lipschitz constant $\rho$ except for NN, since it is NP-hard to evaluate $\rho$ for even a two-layer network \cite{virmaux2018lipschitz}; see 
\if\thisismainpaper1
Appendix~B in \cite{coreset19:report} 
\else
Appendix~B
\fi
for analysis. Here $l$ is the number of principle components computed by PCA, and $\Delta$ is the diameter of the sample space. In our experiments, $\Delta = \sqrt{(d-1)(L^2-2L+2)}$, which is $4.5$ for Fisher's iris, $13.4$ for Facebook, $36.2$ for Pendigits, $181.1$ for MNIST, 
and $120.8$ for HAR. While SVM and NN do not have a meaningful $\rho$, we still include them to stress-test our algorithm.  

\emph{\bf Performance metrics:} For the unsupervised learning problems (MEB, $k$-means, and PCA), we evaluate the performance by the normalized cost as explained in Section~\ref{subsec:Motivating Experiment}. For the supervised learning problems (SVM and NN), we evaluate the performance by the accuracy in predicting the labels of testing data. MNIST and HAR datasets are already divided into training set and testing set. For other datasets, we use the first $80\%$ of data for training, and the rest for testing. 

\emph{\bf Results in centralized setting:} 
Figures~\ref{fig:iris_size}--\ref{fig:har_size} show the performances achieved at a variety of coreset sizes, averaged over $100$ Monte Carlo runs. Better performance is indicated by lower cost for an unsupervised learning problem or higher accuracy for a supervised learning problem.  
Note that even the largest coresets generated in these experiments are much smaller (by $84$--$99.3\%$) than the original dataset, implying significant reduction in the communication cost by reporting a coreset instead of the raw data. 

We see that the proposed algorithms (`RCC-kmeans' and `RCC-kmedian') perform either the best or comparably to the best across all the datasets and all the machine learning problems. The farthest point algorithm in \cite{Badoiu03SODA}, designed for MEB, can perform very poorly for other machine learning problems. The sampling-based algorithms ('nonuniform sampling' \cite{Feldman11STOC} and `uniform sampling') perform relatively poorly for MEB and PCA. 
Generally, we see that the advantages of RCC algorithms are more significant at small coreset sizes. One exception is the SVM accuracy for Fisher's iris (Figure~\ref{fig:iris_size}~(d)), where points on the peripheral of the dataset (which are likely to be chosen by the farthest point algorithm) happen to have different labels and induce a rough partition between the points labeled `$1$' and those labeled `$-1$', causing better performance for `farthest point' at very small coreset sizes.

\begin{figure}[t]
\begin{minipage}{0.238\textwidth}
\centerline{
\includegraphics[width=\textwidth,height=3.6cm]{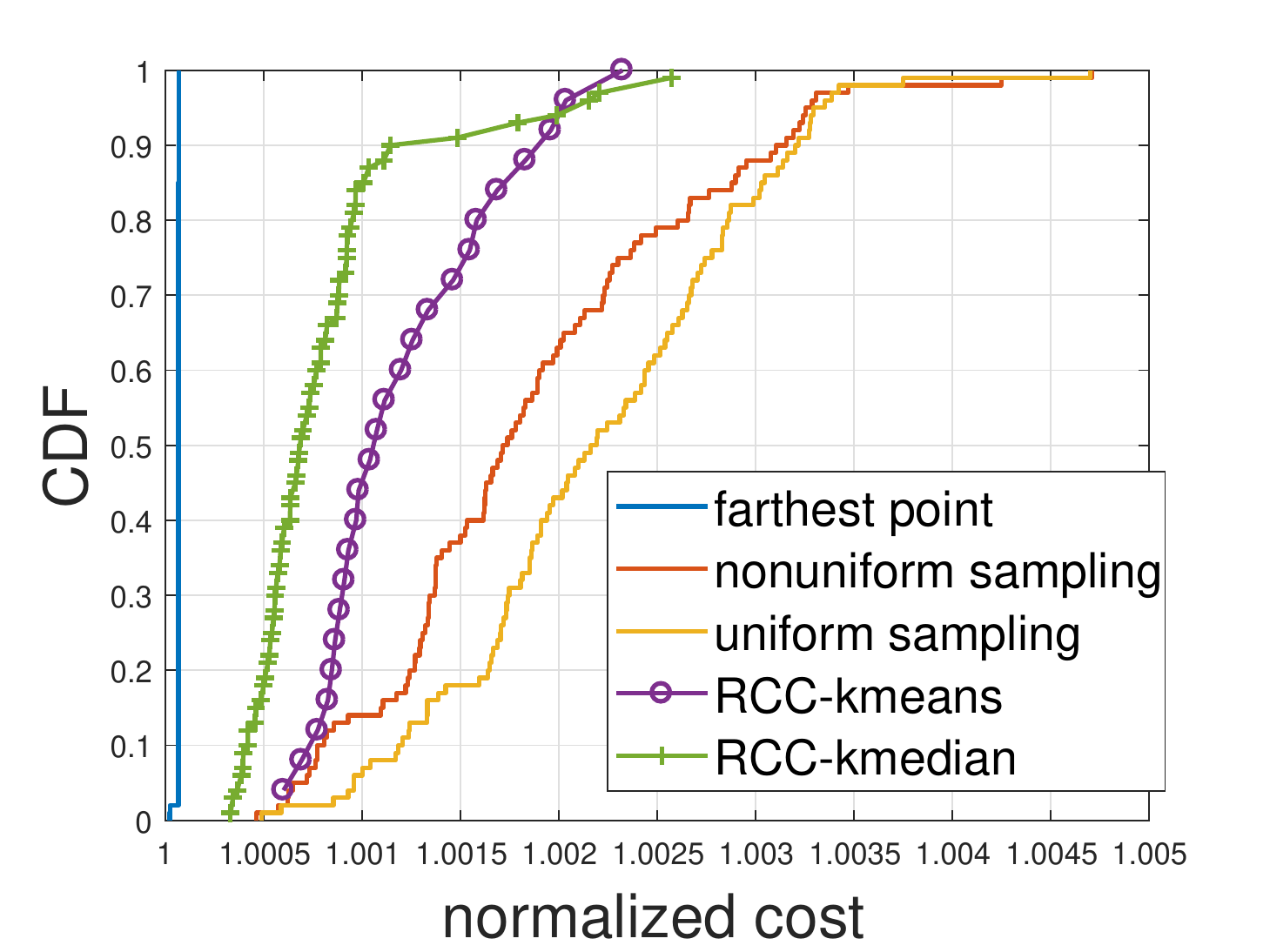}}
\centerline{\scriptsize (a) MEB}
\end{minipage}
\begin{minipage}{0.238\textwidth}
\centerline{
\includegraphics[width=\textwidth,height=3.6cm]{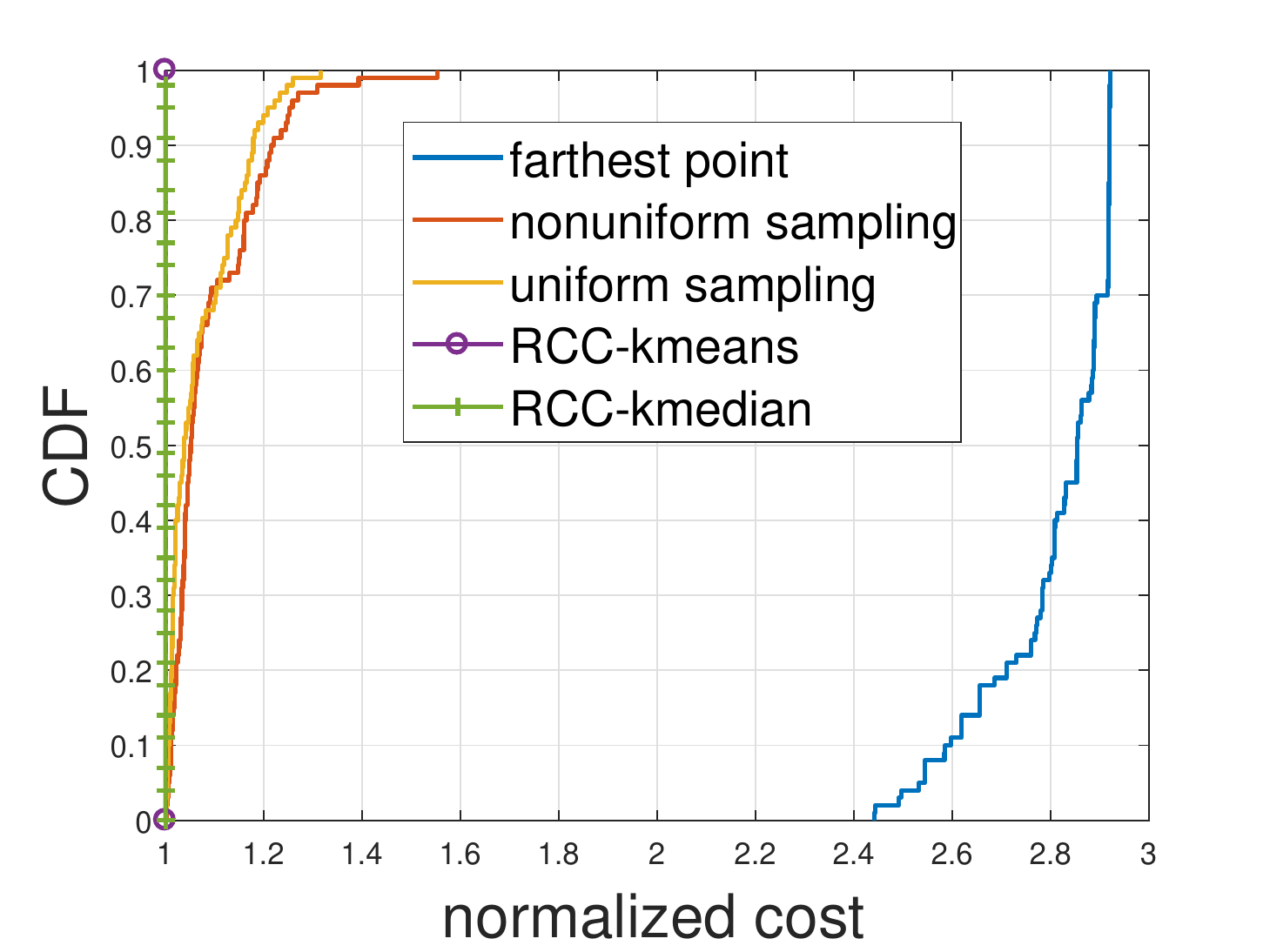}}
\centerline{\scriptsize (b) $k$-means ($k=2$)}
\end{minipage}
  \begin{minipage}{.238\textwidth}
  \centerline{
   \includegraphics[width=\textwidth,,height=3.6cm]{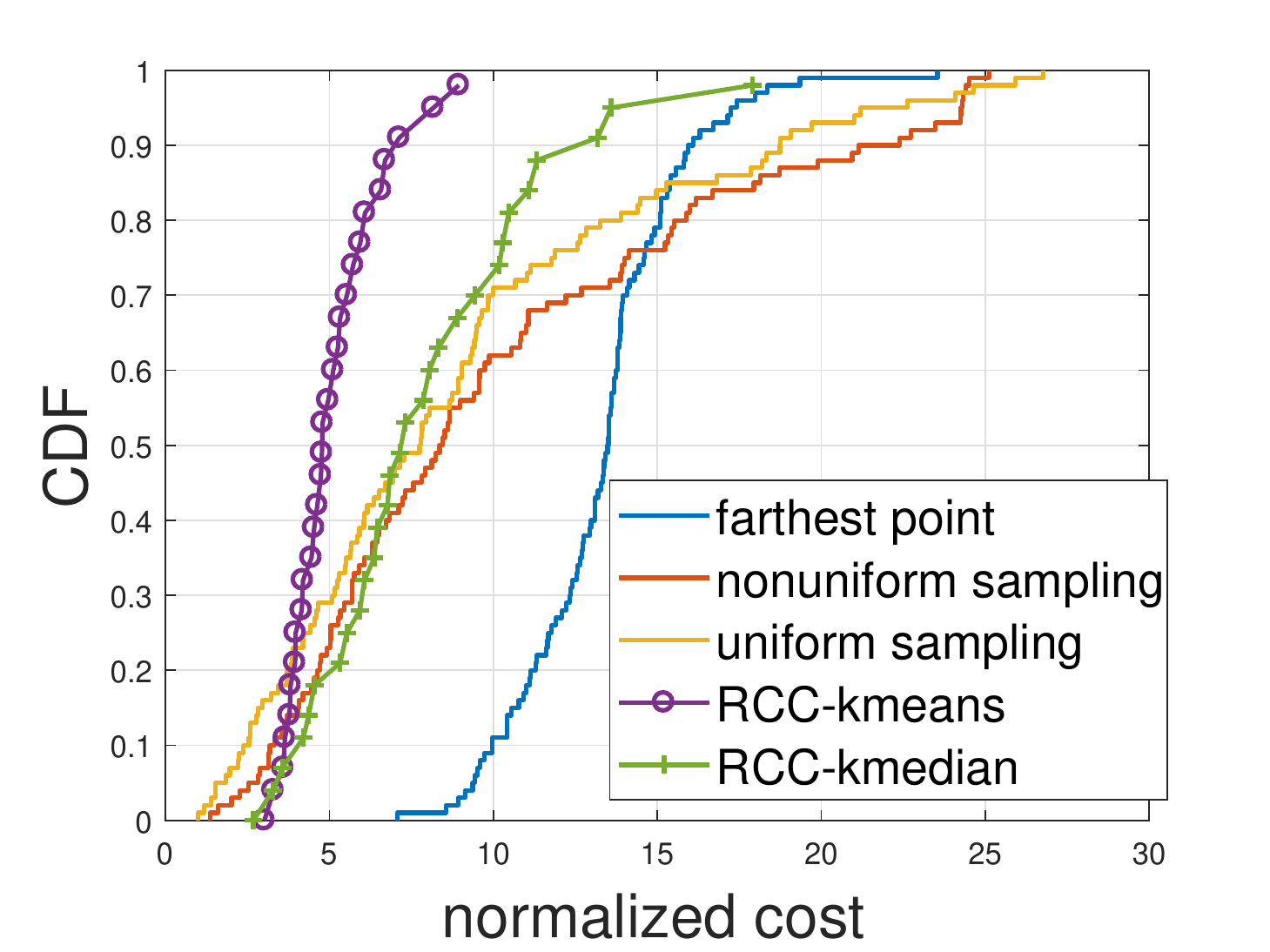}}
    \centerline{\scriptsize (c) PCA (11 components) }
  \end{minipage}
  \begin{minipage}{0.238\textwidth}
    \centerline{
  \includegraphics[width=\textwidth,height=3.6cm]{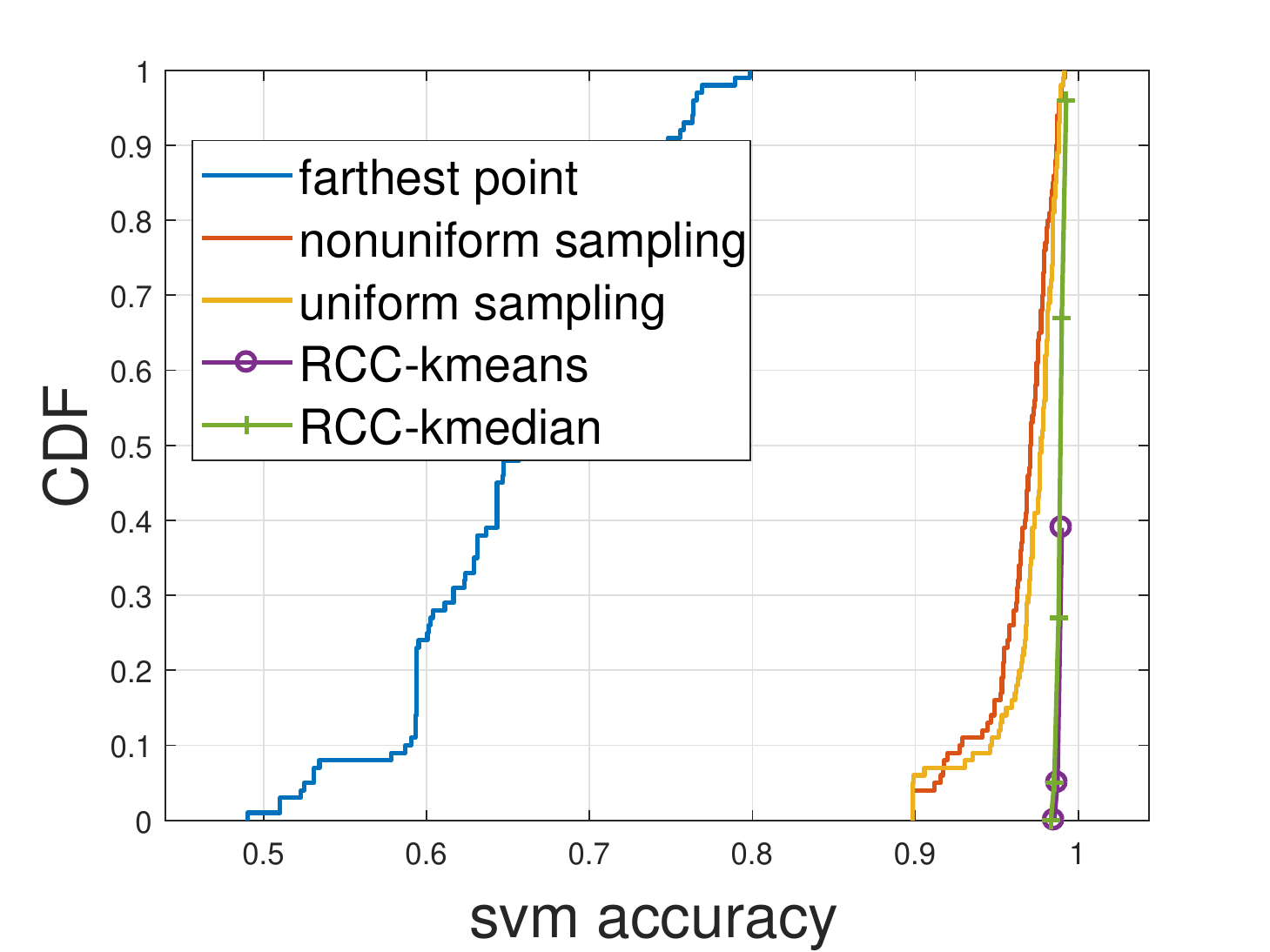}}
\centerline{\scriptsize (d) SVM (`0': 1; others: -1) }   
   \end{minipage}    
\caption{Detailed evaluation on Pendigits dataset (label: `digit', coreset size: $40$).  }
\label{fig:pendigits}
\end{figure}

\begin{figure}[t]
\begin{minipage}{0.238\textwidth}
\centerline{
\includegraphics[width=1.05\textwidth,height=3.6cm]{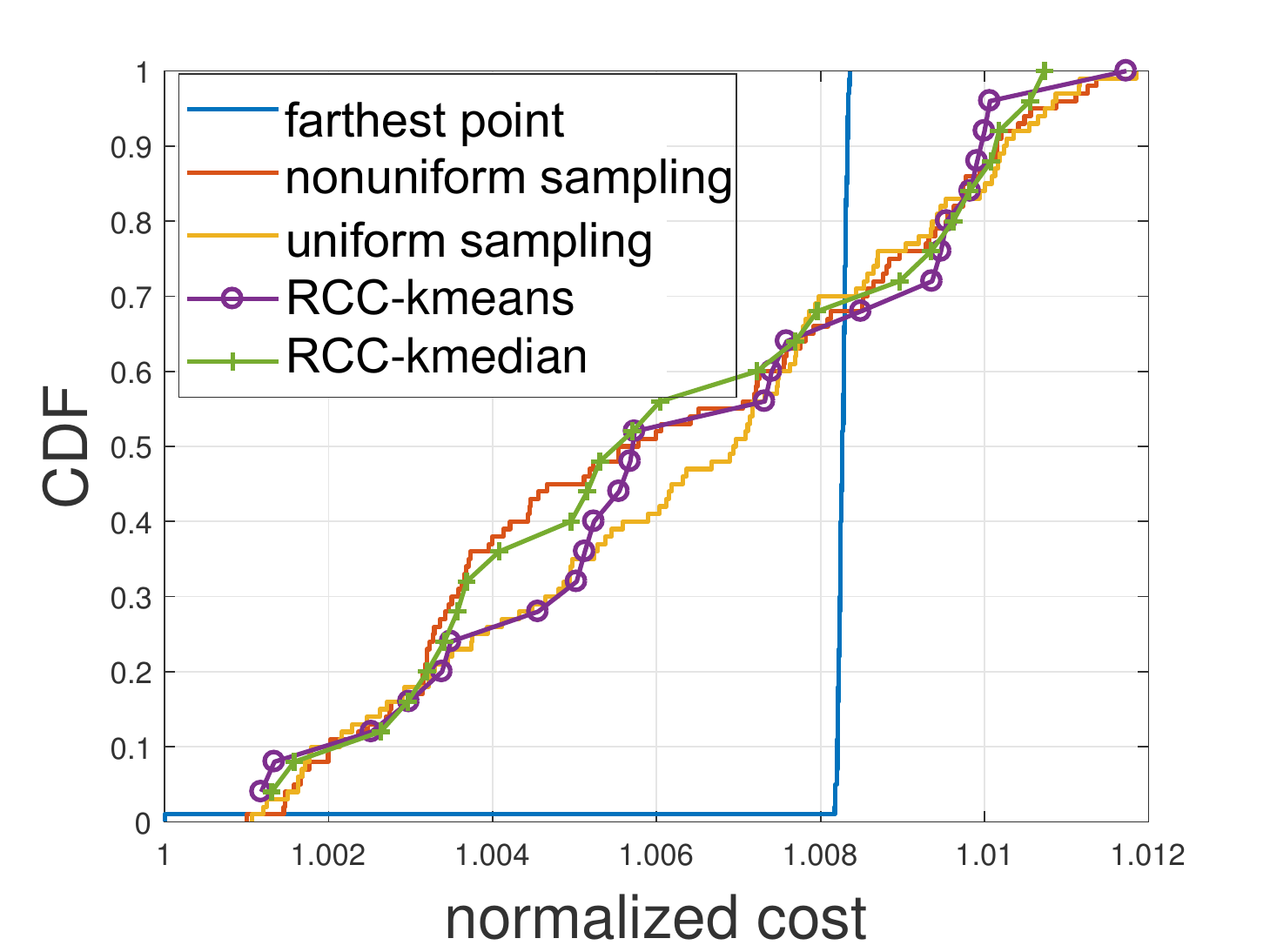}}
\centerline{\scriptsize (a) MEB}
\end{minipage}
\begin{minipage}{0.238\textwidth}
\centerline{
\includegraphics[width=1.05\textwidth,height=3.6cm]{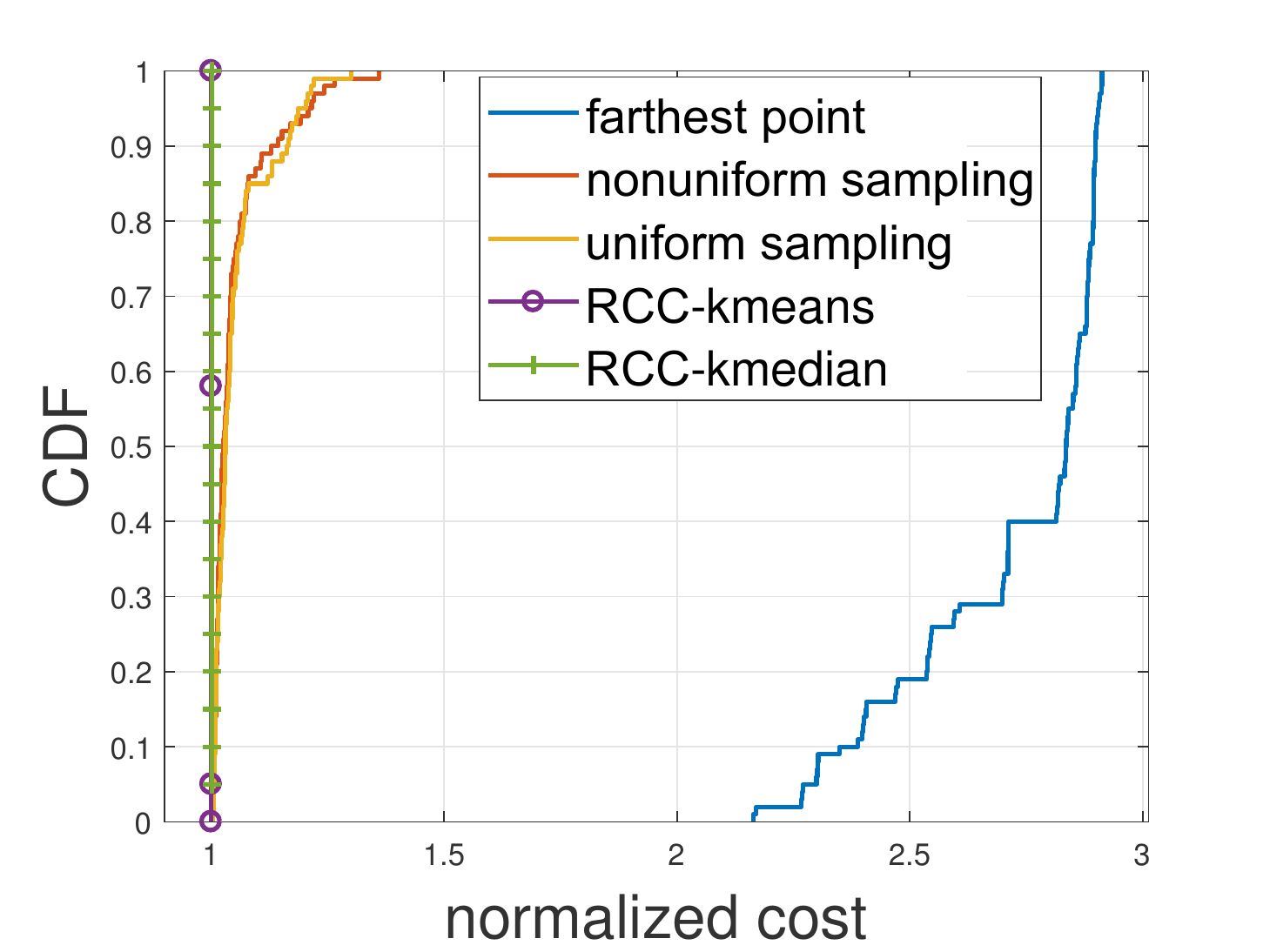}}
\centerline{\scriptsize (b) $k$-means ($k=2$)}
\end{minipage}
  \begin{minipage}{.238\textwidth}
  \centerline{
   \includegraphics[width=1.05\textwidth,,height=3.6cm]{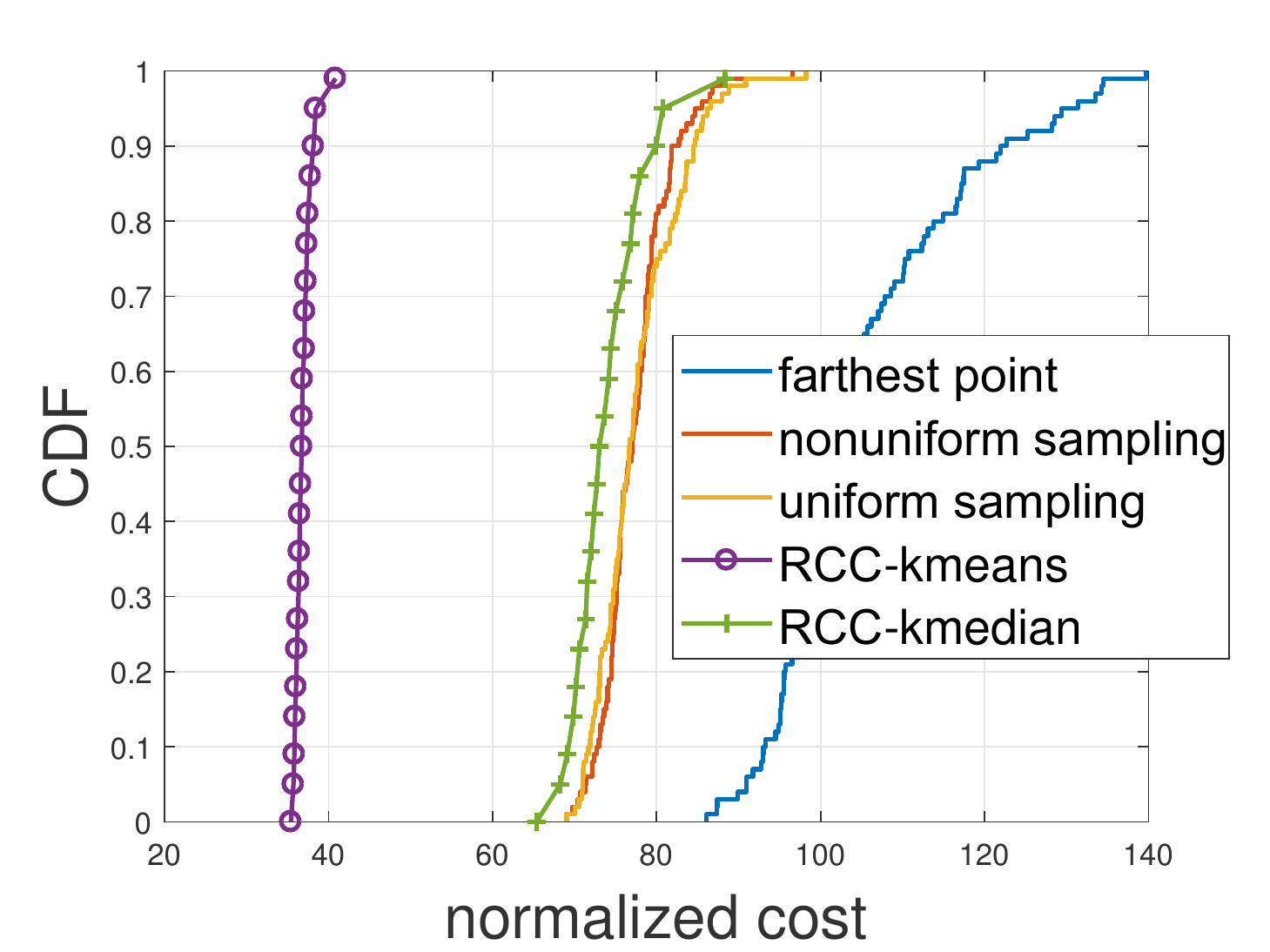}}
    \centerline{\scriptsize (c) PCA (300 components) }
  \end{minipage}
  \begin{minipage}{0.238\textwidth}
    \centerline{
  \includegraphics[width=1.05\textwidth,height=3.6cm]{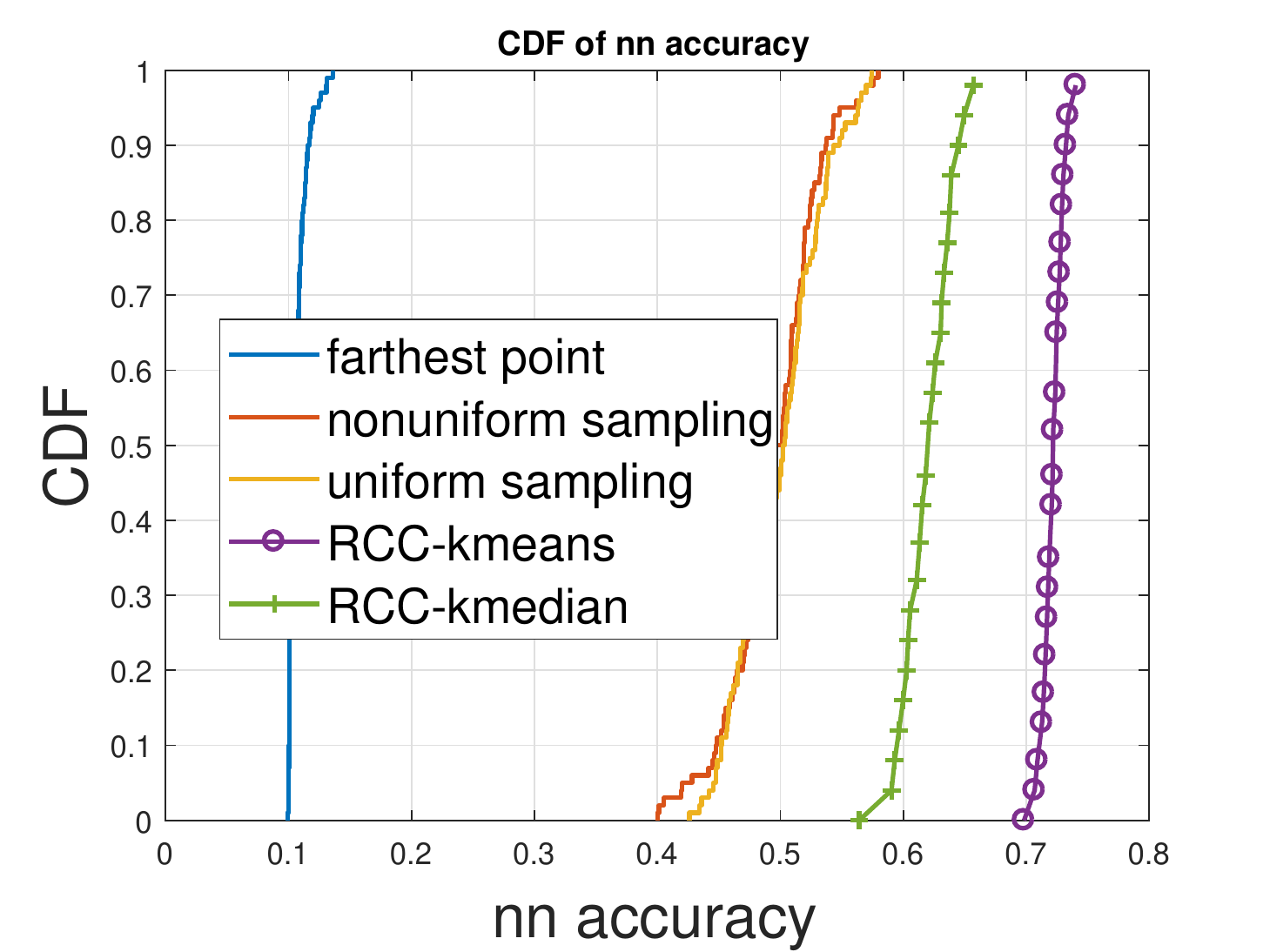}}
\centerline{\scriptsize (d) NN} 
   \end{minipage}    
\caption{Detailed evaluation on MNIST dataset (label: `labels', coreset size: $50$). }
\label{fig:mnist50}
  \vspace{0em}
\end{figure}

\begin{figure}[t]
\begin{minipage}{0.238\textwidth}
\centerline{
\includegraphics[width=1.05\textwidth,height=3.6cm]{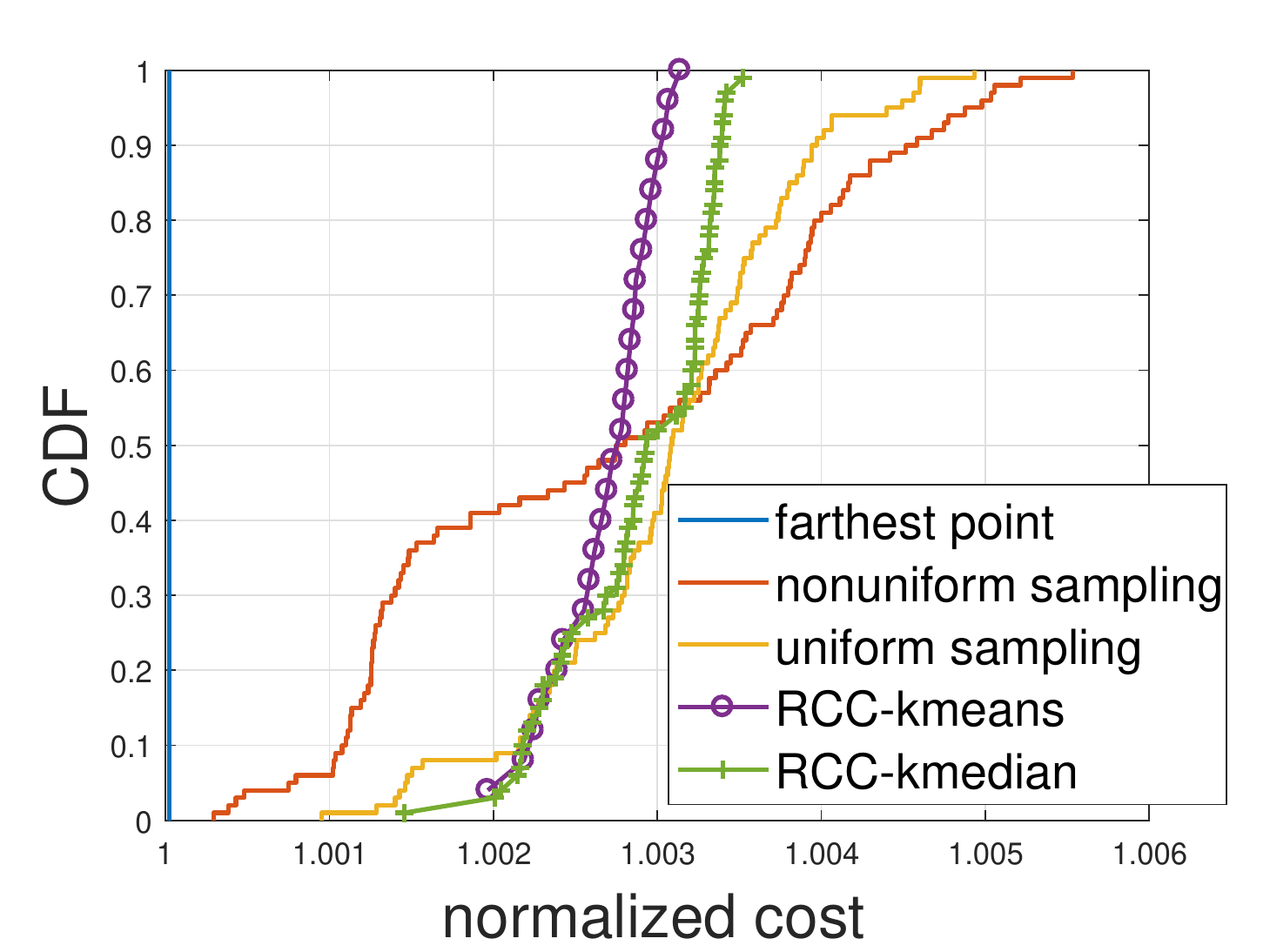}}
\centerline{\scriptsize (a) MEB}
\end{minipage}
\begin{minipage}{0.238\textwidth}
\centerline{
\includegraphics[width=1.05\textwidth,height=3.6cm]{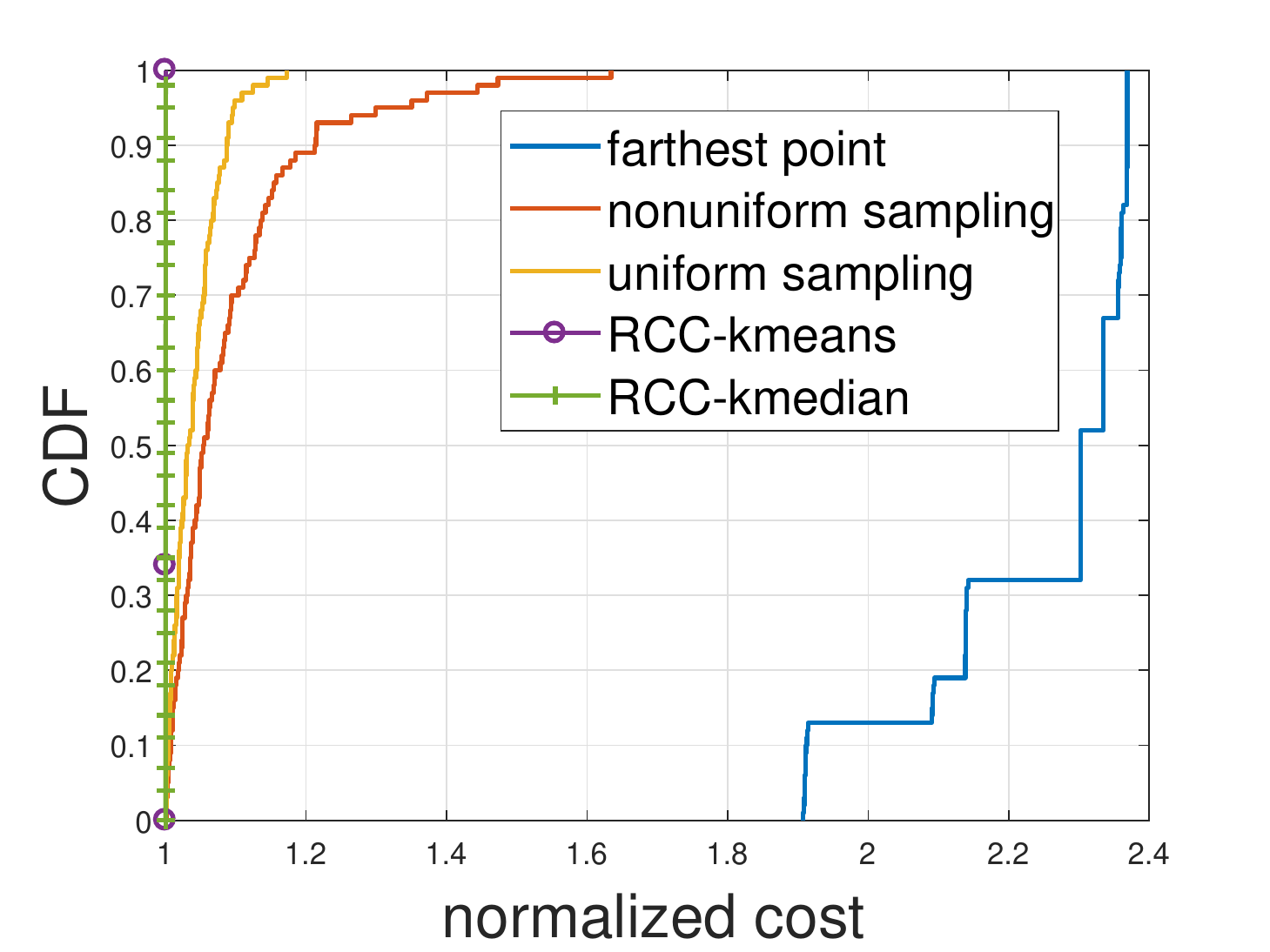}}
\centerline{\scriptsize (b) $k$-means ($k=2$)}
\end{minipage}
  \begin{minipage}{.238\textwidth}
  \centerline{
   \includegraphics[width=1.05\textwidth,,height=3.6cm]{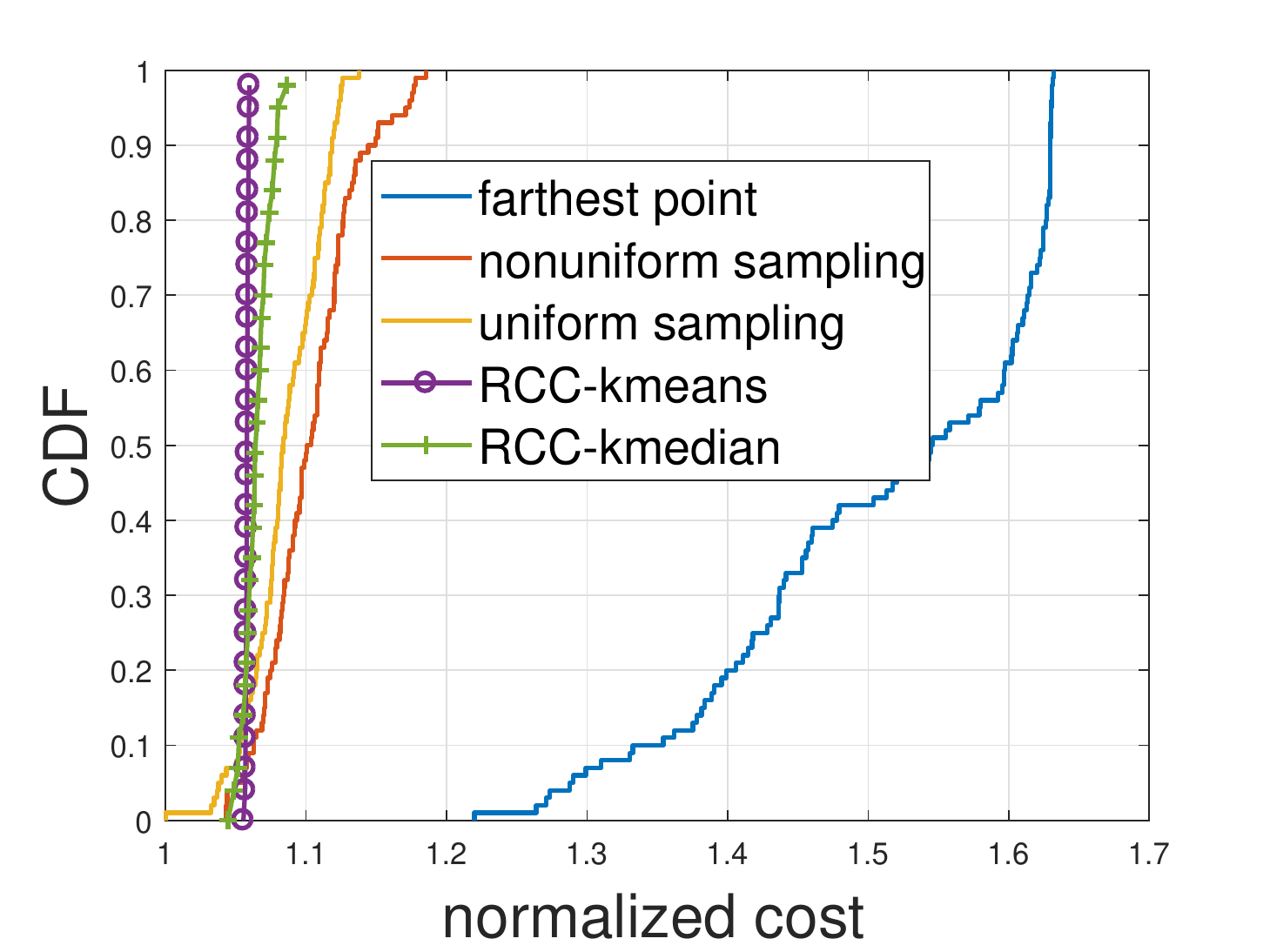}}
    \centerline{\scriptsize (c) PCA (7 components) }
  \end{minipage}
  \begin{minipage}{0.238\textwidth}
    \centerline{
  \includegraphics[width=1.05\textwidth,height=3.6cm]{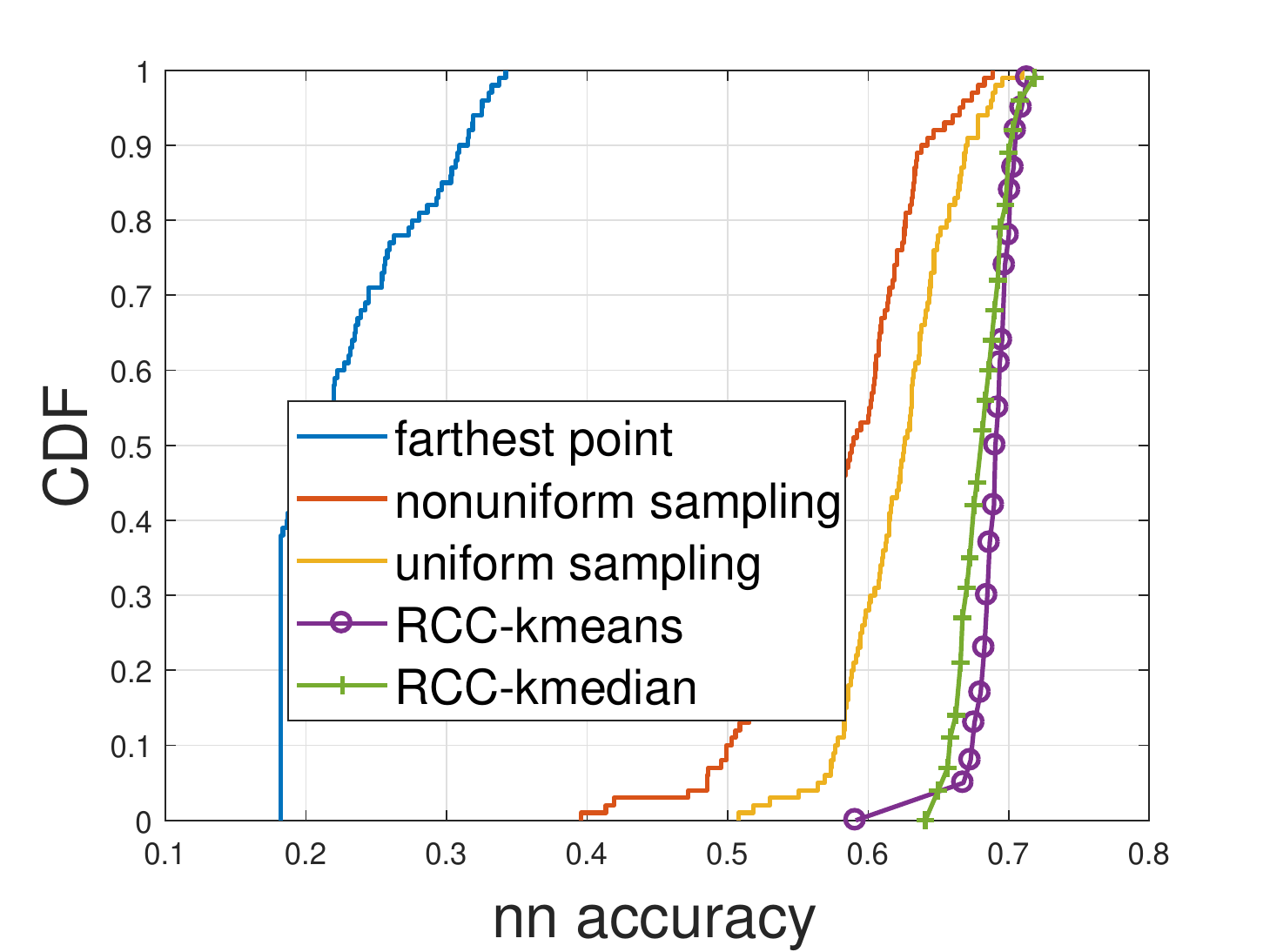}}
\centerline{\scriptsize (d) NN} 
   \end{minipage}    
\caption{Detailed evaluation on HAR dataset (label: `labels', coreset size: $50$). }
\label{fig:har50}
  \vspace{0em}
\end{figure}

\begin{table}[t]
{
\footnotesize
\renewcommand{\arraystretch}{1.3}
\caption{Average Running Time (sec) (`FP': farthest point, `NS': nonuniform sampling, `US': uniform sampling, `RS': RCC-kmeans, `RN': RCC-kmedian) } \label{tab:time}
\centerline{
\begin{tabular}{r|c|c|c|c|c}
  \hline
  algorithm & Fisher & Facebook & Pendigits & MNIST & HAR \\
  \hline
FP 
& 1.62 & 3.00 & 2.53 & 21.69 & 25.92 \\
\hline
NS 
& 0.019 & 0.027 & 0.095 & 7.42 & 0.69 \\
\hline 
US 
& 2.10e-04 & 4.60e-04 & 3.80e-04 & 0.01 & 0.0013 \\
\hline
RS 
& 0.0083 & 0.011 & 0.042 & 18.76 & 1.46 \\
\hline 
RN 
& 0.028 & 0.30 & 0.40 & 100.64 & 12.39 \\
  \hline
\end{tabular}
}
}
 \vspace{.5em}
\end{table}

Besides the average normalized costs, we also evaluated the CDFs of the results, shown in Figures~\ref{fig:fisher}--\ref{fig:har50}. 
The results show similar comparisons as observed before. 
Moreover, we see that the proposed algorithms (`RCC-kmeans' and `RCC-kmedian') also have significantly less performance variation than the benchmarks, especially the sampling-based algorithms (`nonuniform sampling' and `uniform sampling'). This means that the quality of the coresets constructed by the proposed algorithms is more reliable, which is a desirable property. 

Between the proposed algorithms, `RCC-kmeans' sometimes outperforms `RCC-kmedian', e.g., Figure~\ref{fig:mnist_size}~(c--d).  
Moreover, we note that `RCC-kmeans' can be an order of magnitude faster than `RCC-kmedian', as shown in Table~\ref{tab:time}. Note that our primary goal in constructing a robust coreset is to reduce the communication cost in scenarios like Figure~\ref{fig:systemmodel} while supporting diverse machine learning problems, instead of speeding up the coreset construction at the data source. This result shows that such robustness may come with certain penalty in running time. Nevertheless, the running time of `RCC-kmeans' is comparable to the benchmarks.  



\begin{figure}[tb]
\begin{minipage}{0.24\textwidth}
\centerline{
\includegraphics[width=1.05\textwidth,height=3.7cm]{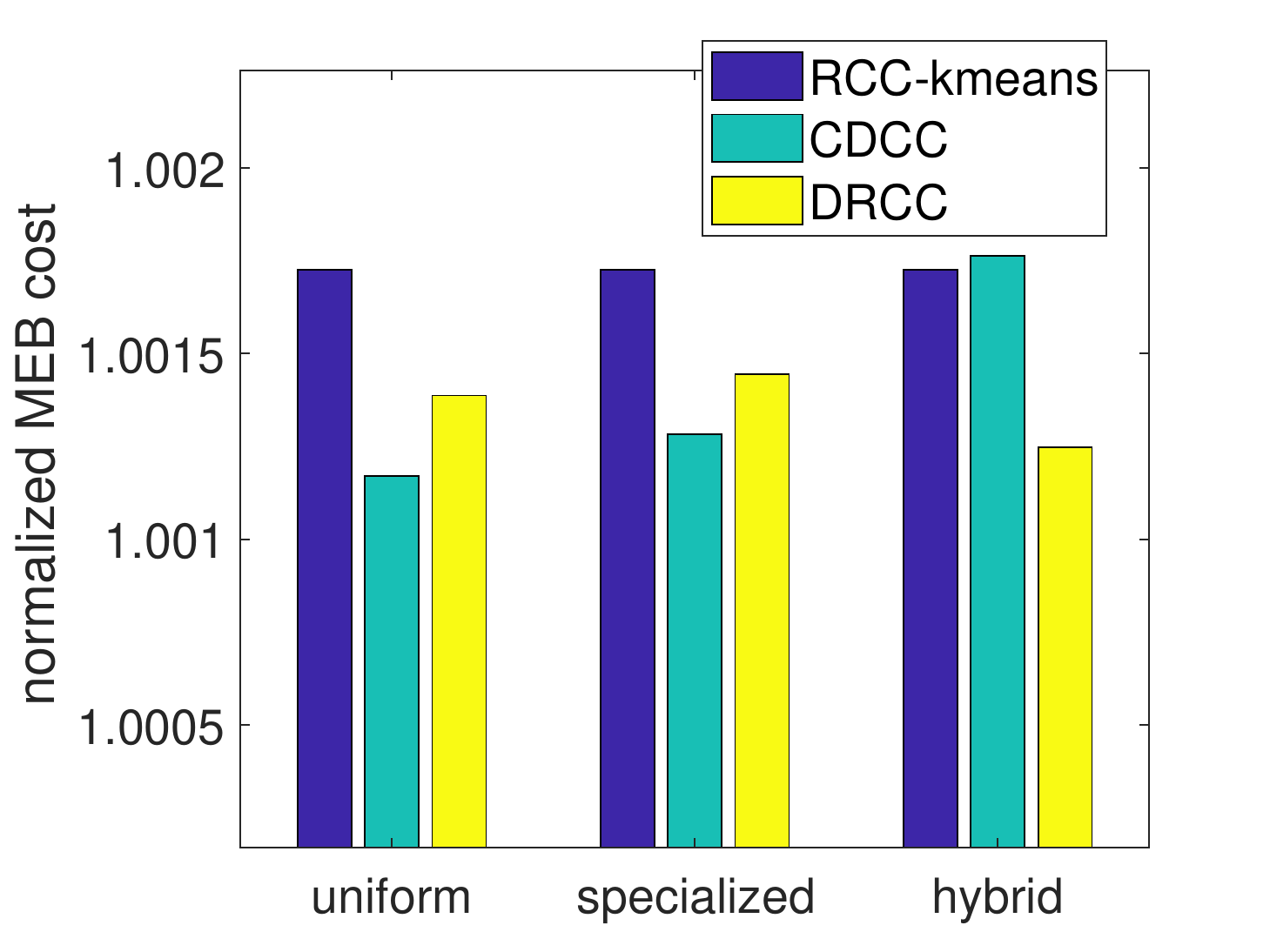}}
\centerline{\scriptsize (a) MEB}
\end{minipage}
\begin{minipage}{0.24\textwidth}
\centerline{
\includegraphics[width=1.05\textwidth,height=3.7cm]{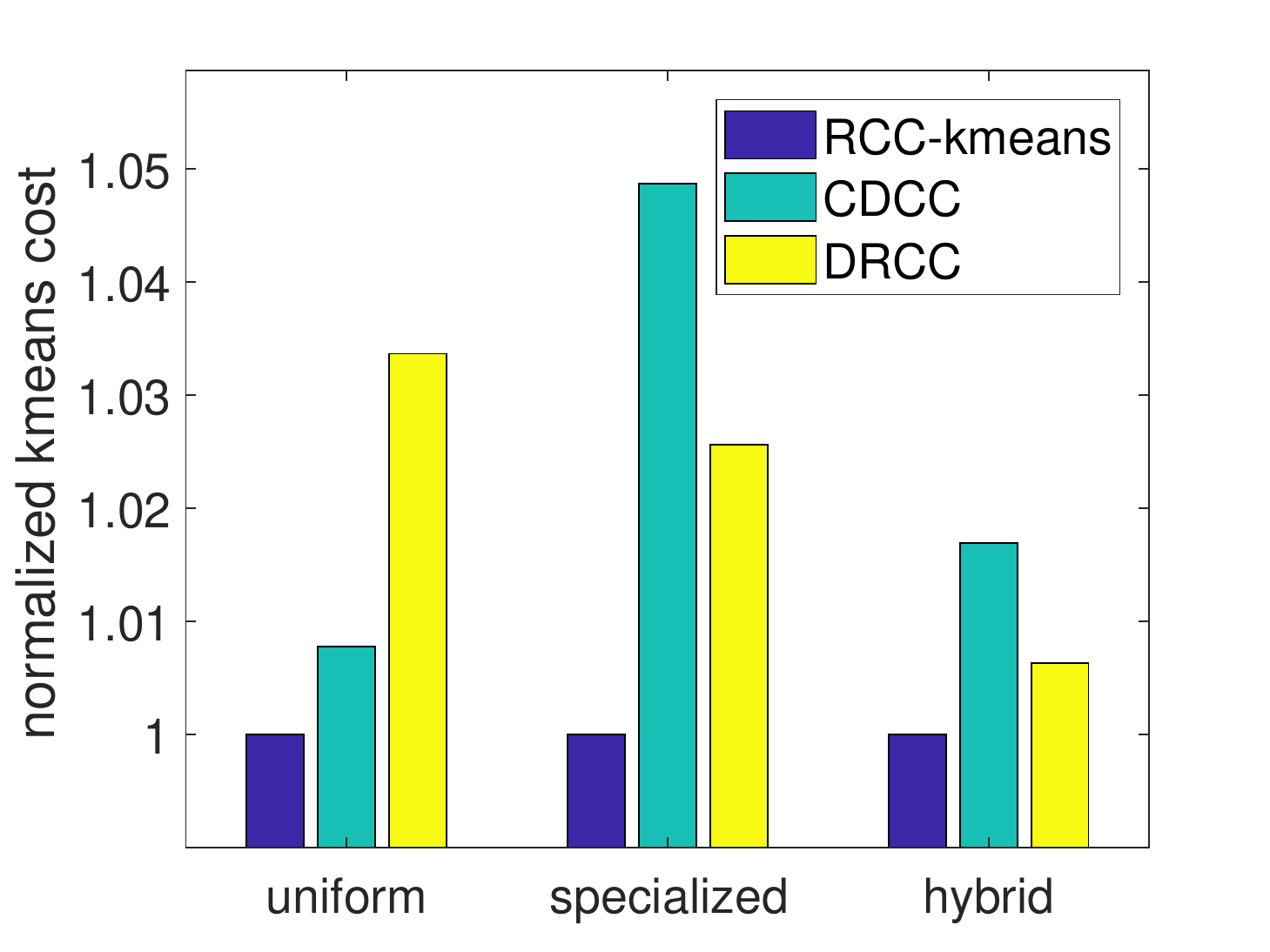}}
\centerline{\scriptsize (b) $k$-means ($k=2$)}
\end{minipage}
  \begin{minipage}{.24\textwidth}
  \centerline{
   \includegraphics[width=1.05\textwidth,height=3.7cm]{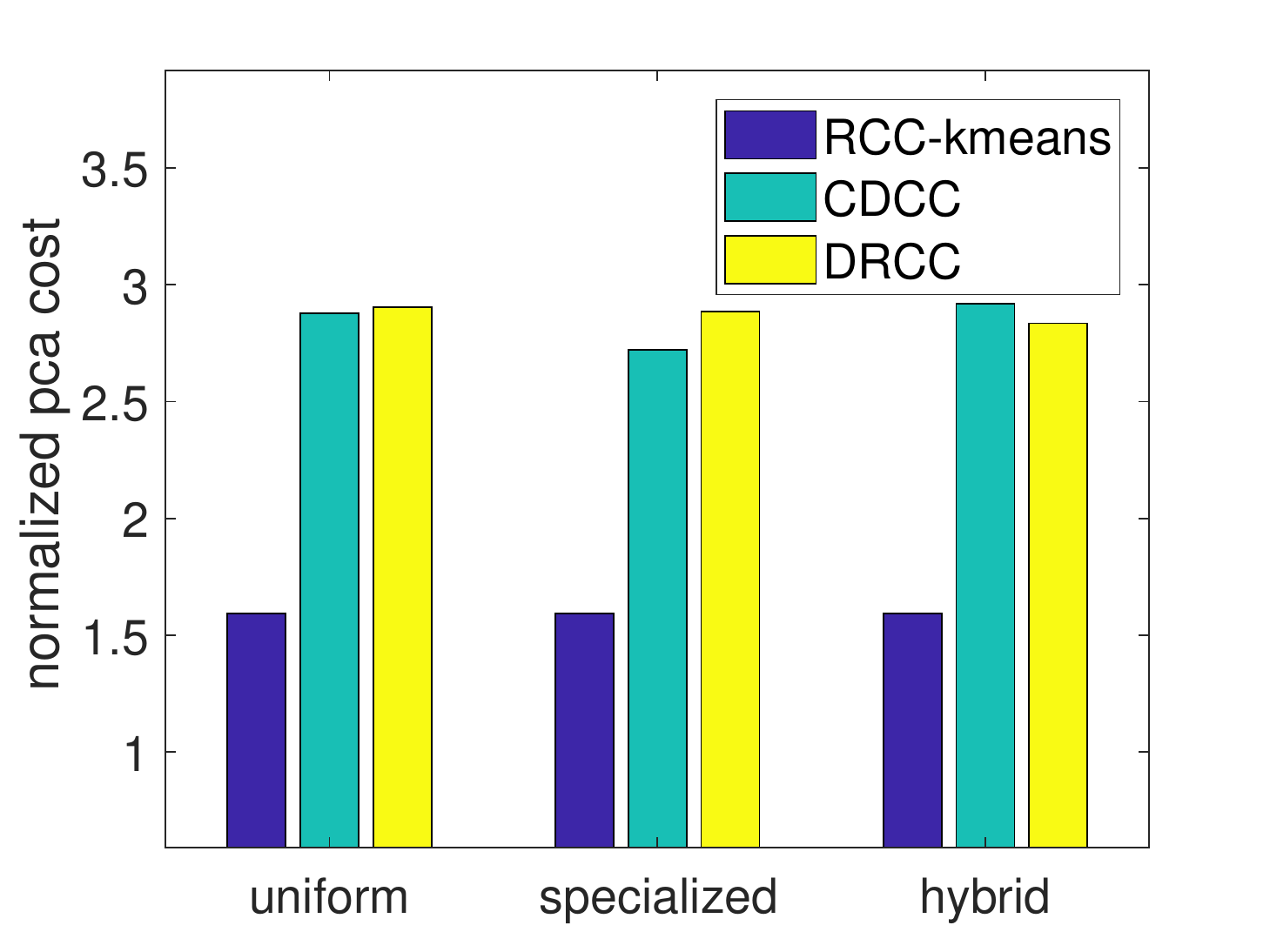}}
    \centerline{\scriptsize (c) PCA (300 components) }
  \end{minipage}
  \begin{minipage}{0.24\textwidth}
    \centerline{
  \includegraphics[width=1.05\textwidth,height=3.7cm]{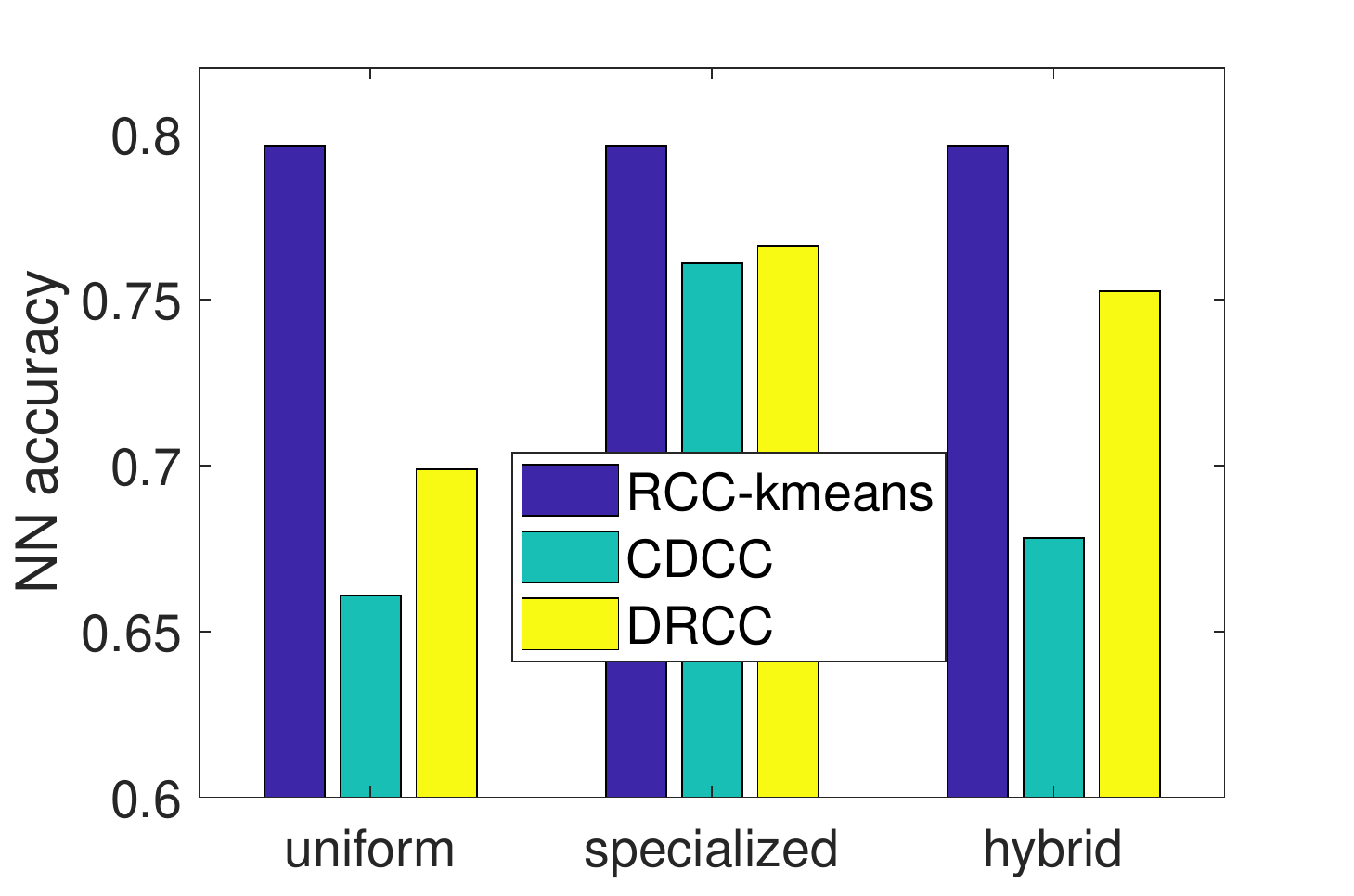}} 
\centerline{\scriptsize (d) NN} 
   \end{minipage}    
\caption{Evaluation on MNIST in distributed setting (label: `labels', coreset size: 400, $K=10$). }
\label{fig:MNIST_distributed}
\end{figure}

\begin{figure}[tb]
\begin{minipage}{0.24\textwidth}
\centerline{
\includegraphics[width=1.05\textwidth,height=3.7cm]{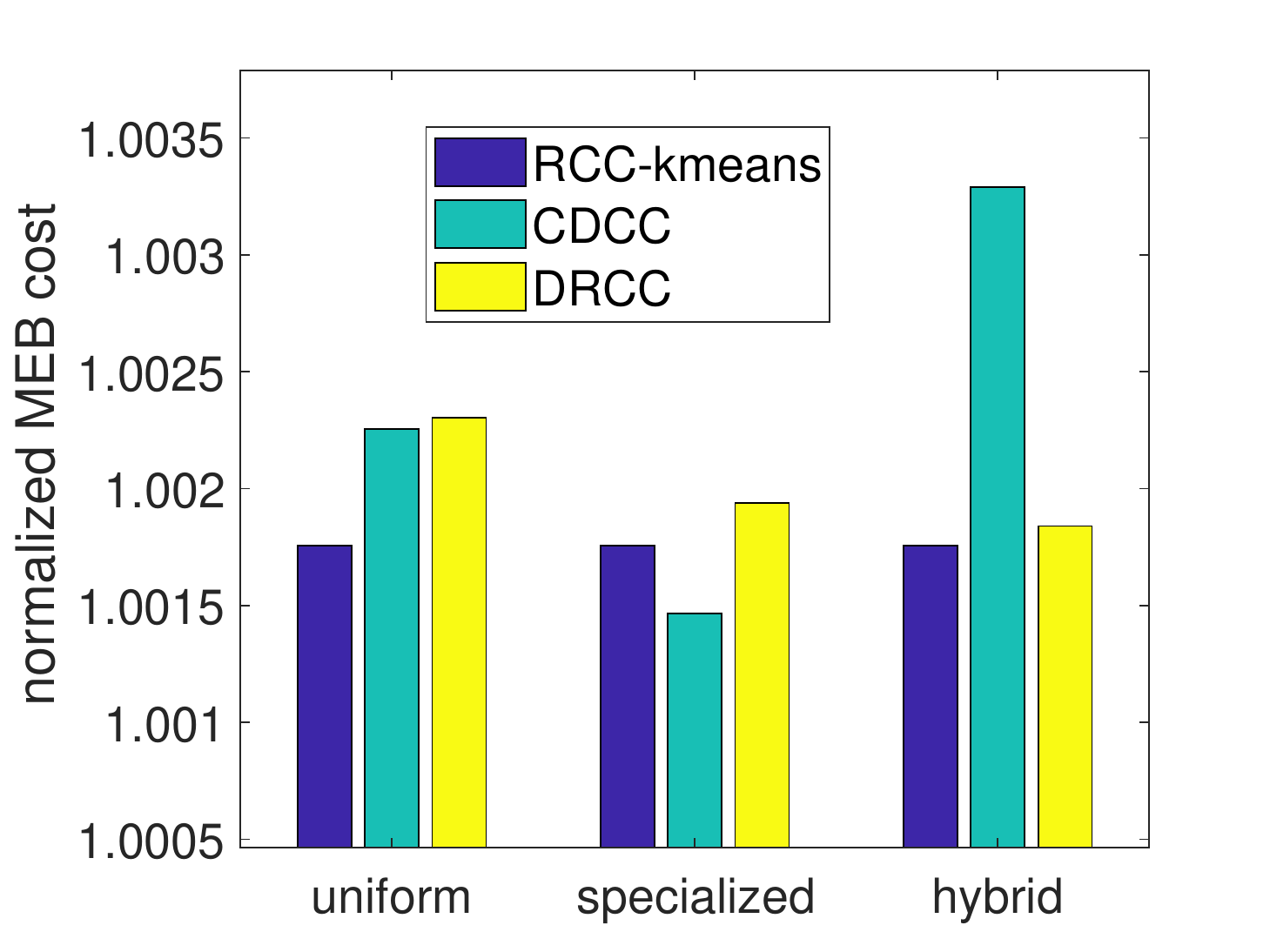}}
\centerline{\scriptsize (a) MEB}
\end{minipage}
\begin{minipage}{0.24\textwidth}
\centerline{
\includegraphics[width=1.05\textwidth,height=3.7cm]{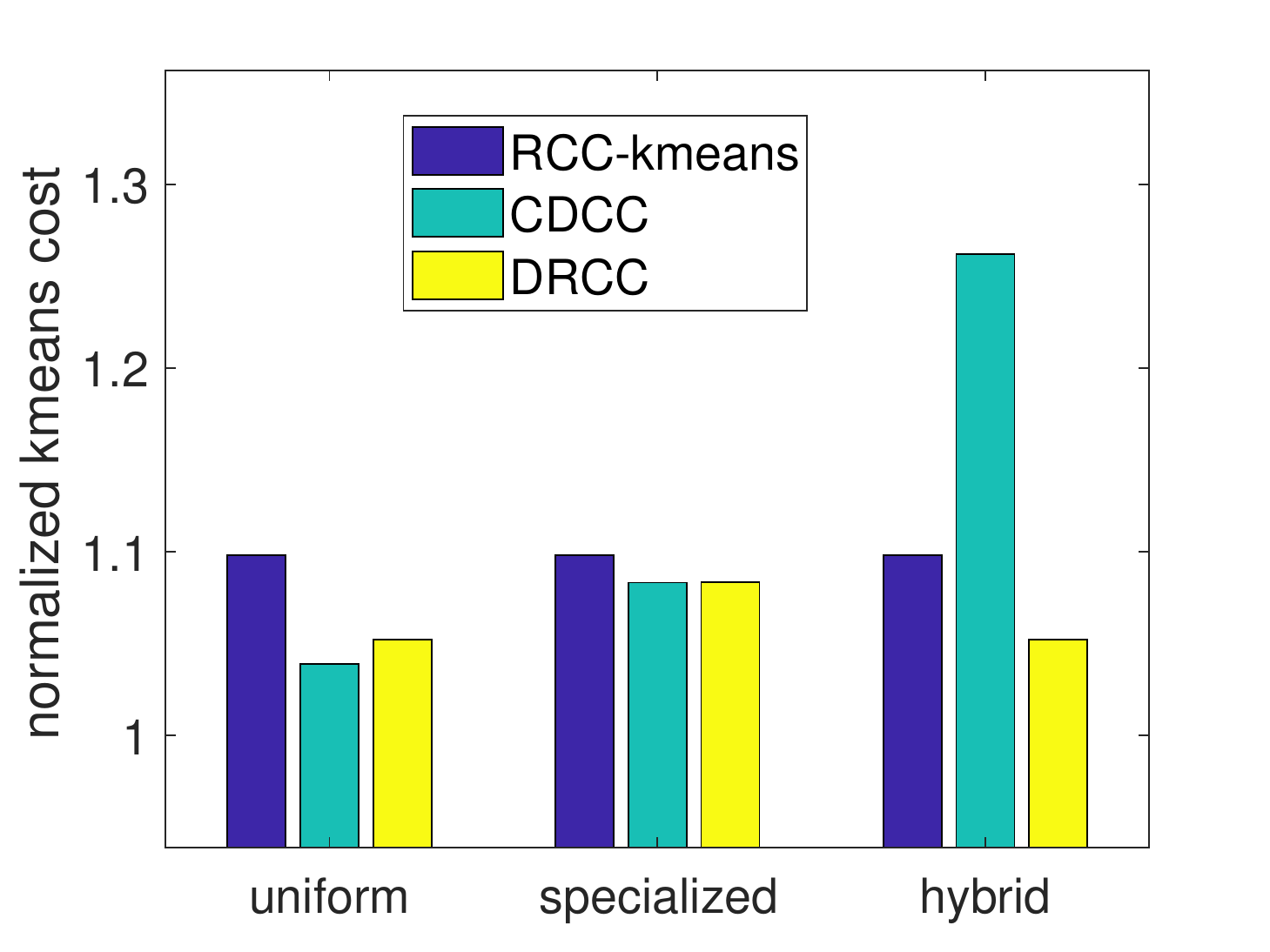}}
\centerline{\scriptsize (b) $k$-means ($k=2$)}
\end{minipage}
  \begin{minipage}{.24\textwidth}
  \centerline{
   \includegraphics[width=1.05\textwidth,height=3.7cm]{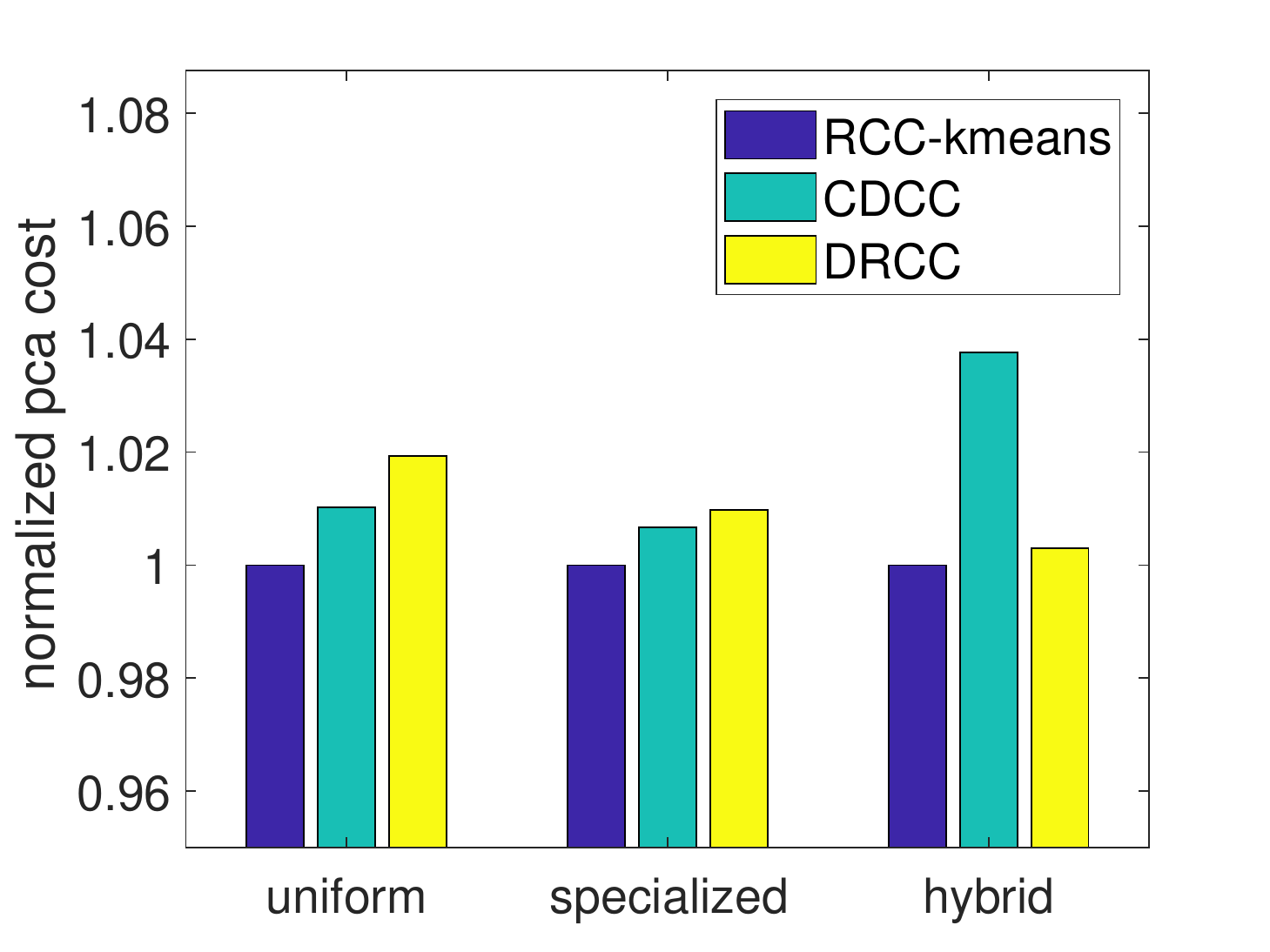}}
    \centerline{\scriptsize (c) PCA (7 components) }
  \end{minipage}
  \begin{minipage}{0.24\textwidth}
    \centerline{
  \includegraphics[width=1.05\textwidth,height=3.7cm]{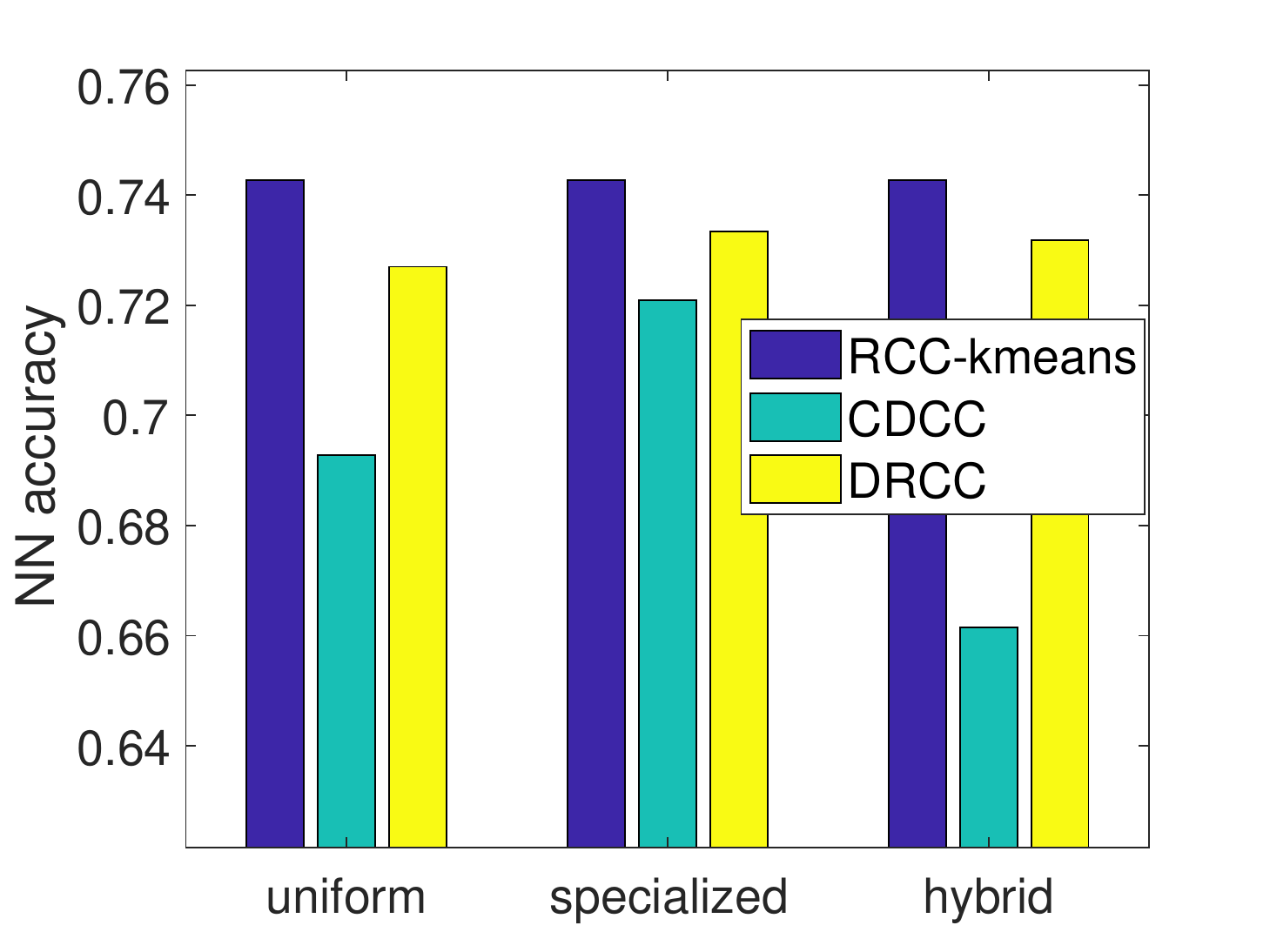}} 
\centerline{\scriptsize (d) NN} 
   \end{minipage}    
\caption{Evaluation on HAR in distributed setting (label: `activity', coreset size: 200, $K=10$). }
\label{fig:HAR_distributed}
\end{figure}


\emph{\bf Results in distributed setting:}
We use MNIST and HAR to generate distributed datasets at $n=L$ nodes according to the three aforementioned schemes (`uniform', `specialized', `hybrid'), where node $j$ ($j\in [n]$) is associated with label `$j-1$', and $n_0=\lfloor n/2 \rfloor$. We parameterize CDCC with $k=2$ according to the evaluated $k$-means problem. For DRCC (Algorithm~\ref{Alg:distributed RCC}), we solve line~\ref{DRCC: find k_j} by a greedy heuristic. As a bound, we also show the performance of RCC based on $k$-means clustering (`RCC-kmeans'), which is the best-performing algorithm {in the centralized setting}. We generate $5$ distributed datasets using each scheme and repeat coreset construction for $5$ times on each dataset. 
Figures~\ref{fig:MNIST_distributed}--\ref{fig:HAR_distributed} show the average results under fixed coreset size $N$ and parameter $K$ (defined in Algorithm~\ref{Alg:distributed RCC}); see Appendix~D for additional evaluation results when we vary these parameters. 

As CDCC blindly configures $k$ coreset points to be local centers at each node, it suffers when the local datasets are highly heterogeneous (under the `hybrid' scheme). 
By automatically tuning the numbers of local centers according to the local clustering costs, DRCC is able to customize its configuration to the dataset at hand and thus achieves a more robust performance, especially under the `hybrid' scheme.  
We note that as CDCC was designed to support $k$-means, its performance for $k$-means can be slightly better than DRCC (e.g., Figure~\ref{fig:MNIST_distributed}--\ref{fig:HAR_distributed}~(b) under the `uniform' scheme), but the difference is very small ($<3\%$). 
Meanwhile, the robustness of DRCC comes with a small cost in communication overhead, where each node needs to report $K$ scalars (here $K=10$) instead of one scalar as in CDCC. However, this difference is negligible compared to the cost of reporting the coreset itself (e.g., $400\times 401$ scalars for MNIST and $200\times 562$ scalars for HAR). Both algorithms generate coresets that are much smaller (by $97.2$--$99.3\%$) than the original dataset, significantly reducing the communication cost. 
Note that `RCC-kmeans' is a centralized algorithm that is only used as a benchmark here.  \looseness=-1

Additionally, Table~\ref{tab:time, distributed} shows the average running times of CDCC and DRCC on each dataset. Similar to the centralized setting, the better robustness of DRCC comes with certain penalty in running time due to the need to compute multiple instances of  $k$-clustering ($k=1,\ldots,K$) on local datasets. 

\begin{table}[t]
{
\footnotesize
\renewcommand{\arraystretch}{1.3}
\caption{Average Running Time (sec) } \label{tab:time, distributed}
\centerline{
\begin{tabular}{r|c|c}
  \hline
  algorithm &  MNIST & HAR \\
  \hline
CDCC  & 13.84 & 1.55 \\
\hline
DRCC & 31.66 & 2.42 \\
  \hline
\end{tabular}
}
}
 \vspace{.5em}
\end{table}

\section{Conclusion}\label{sec:Conclusion}

We show, both theoretically and empirically, that the $k$-clustering centers form a coreset that provides a guaranteed approximation for a broad set of machine learning problems with sufficiently continuous cost functions. As $k$-clustering (including $k$-means/median) is one of the most well-studied machine learning problems, this result allows us to leverage existing $k$-clustering algorithms for robust coreset construction. In particular, we have adapted an existing distributed $k$-clustering algorithm to construct a robust coreset over distributed data with a very low communication overhead. Our extensive evaluations on real datasets verify the superior performance of the proposed algorithms in supporting diverse machine learning problems.




\appendixnumbering{A}
\section*{Appendix A: Proofs}

\subsection{Proof of \Cref{lem:1}:}
\begin{proof}
By definition, $\opt(P, k) = \sum^k_{i = 1} \opt(P_i, 1)$. By letting $X_i \coloneqq \{X_{i ,1}, X_{i, 2}\} $ be the centers of the optimal 2-clustering of $P_i$, we have
\begin{align}
&\sum^k_{i = 1} \opt(P_i, 2) = \sum^k_{i = 1} \sum_{p \in P_i}w_p (\min_{x \in X_i} \dist(p, x))^z \nonumber \\
&\geq \sum^k_{i = 1} \sum_{p \in P_i}w_p (\min_{x \in \bigcup^k_{i = 1} X_i} \dist(p, x))^z \geq \opt(P, 2k).
\end{align}
Thus we have
\begin{align}
\sum^k_{i = 1} (\opt(P_i, 1) - \opt(P_i, 2))  \leq \opt(P, k) - \opt(P, 2k) \leq \epsilon'.
\end{align}
Since $\opt(P_i, 1) - \opt(P_i, 2) \geq 0$, $\opt(P_i, 1) - \opt(P_i, 2) \leq \epsilon'$, $\forall i \in [k]$.
\end{proof}

\subsection{Proof of \Cref{lem:2}:}
\begin{proof}
By definition of $k$-clustering, we have
\begin{equation}\label{eq:opt(Pi,1)}
\opt(P_i, 1) = \sum_{p \in P_i} w_p \dist(p, \mu(P_i))^z,
\end{equation}
\begin{align}\label{eq:opt(Pi,2)}
\opt(P_i, 2) \leq \sum_{p \in P_i^1}w_p \dist(p, \mu(P_i))^z + \sum_{p \in P_i^2}w_p \dist(p, p_0)^z
\end{align}
for any $p_0 \in P_i$, where $P_i^1$ is the subset of points in $P_i$ that are closer to $\mu(P_i)$ than $p_0$ (ties broken arbitrarily) and $P_i^2 \coloneqq P_i \setminus P_i^1$. Then subtracting (\ref{eq:opt(Pi,1)}) by (\ref{eq:opt(Pi,2)}), we have
\begin{align}
&\opt(P_i, 1) - \opt(P_i, 2) 
\nonumber \\ &\geq \sum_{p \in P_i^2} w_p ( \dist(p, \mu(P_i))^z - \dist(p, p_0)^z ).
\end{align}
As $\dist(p, \mu(P_i))^z - \dist(p, p_0)^z \geq 0$ for all $p \in P_i^2$,
\begin{align}
&w_p ( \dist(p, \mu(P_i))^z - \dist(p, p_0)^z ) \nonumber \\
&\leq \opt(P_i, 1) - \opt(P_i, 2) \nonumber\\
&\leq \epsilon',~ \forall p \in P_i^2.
\end{align}
In particular, for $ p = p_0$, $w_{p_0} \dist(p_0, \mu(P_i))^z \leq \epsilon'$ and thus $\dist(p_0, \mu(P_i)) \leq (\frac{\epsilon'}{\wmin})^{\frac{1}{z}}$. The proof completes by noting that this holds for any $p_0 \in P_i$.
\end{proof}

\subsection{Proof of \Cref{lem:3}:}
\begin{proof}
Since $\forall x \in \mathcal{X}$ and $p \in P_i$,
\begin{align}
(1-\epsilon)\cost(p, x) &\leq \cost(\mu(P_i), x) \leq (1+\epsilon)\cost(p, x),
\end{align}
multiplying this inequality by $w_p$ and then summing  over $p \in P_i$, we have
\begin{align}
(1-\epsilon)\sum_{p \in P_i}w_p\cost(p, x) \leq \cost(\mu(P_i), x)\sum_{p \in P_i}w_p \nonumber \\
\leq (1+\epsilon)\sum_{p \in P_i}w_p \cost(p, x), ~\forall x \in \mathcal{X}. \label{eq:approx per Pi}
\end{align}
Summing up (\ref{eq:approx per Pi}) over all $i\in [k]$, we have
\begin{align}
(1-\epsilon)\cost(P, x) \leq \cost(S, x) 
\leq (1+\epsilon)\cost(P, x), ~\forall x \in \mathcal{X}.\nonumber
\end{align}
Therefore, $S$ is an $\epsilon$-coreset for $P$ w.r.t. $\cost(P, x)$.
\end{proof}

\subsection{Proof of \Cref{lem:4}}
\begin{proof}
Taking maximum over $p \in P_i$ for (\ref{eq:lem3}), we have that $\forall i \in [k]$ and $x \in \mathcal{X}$,
\begin{align}
(1 - \epsilon) \max_{p \in P_i} \cost(p ,x) \leq \cost(\mu(P_i), x) \nonumber \\
\leq (1 + \epsilon) \max_{p \in P_i} \cost(p, x).  \label{eq:approx per Pi,max}
\end{align}
Taking maximum over $i \in [k]$ for (\ref{eq:approx per Pi,max}): $\forall x\in \mathcal{X}$,
\begin{align*}
(1 - \epsilon) \max_{p \in P} \cost(p ,x) \leq \max_{i \in [k]} \cost(\mu(P_i), x) \\
\leq (1 + \epsilon) \max_{p \in P} \cost(p, x).
\end{align*}
That is, $S$ is an $\epsilon$-coreset for $P$ w.r.t. $\cost(P, x)$.
\end{proof}

\subsection{Proof of \Cref{lem:1 suboptimal}:}
\begin{proof}
By definition and Assumption 2, we have 
\begin{equation}
    \approx(P,k) = \sum_{i=1}^k \approx(\widetilde{P}_i, 1)
\end{equation}
\begin{equation}
    \sum_{i=1}^k \approx(\widetilde{P}_i, 2) \geq \approx(P, 2k)
\end{equation}
Thus, we have
\begin{align}
&\sum_{i=1}^k (\approx(\widetilde{P}_i,1)-\approx(\widetilde{P}_i, 2)) \nonumber \\
&\hspace{5em} \leq \approx(P,k)-\approx(P,2k) \leq \epsilon'. \label{eq:lem1 suboptimal proof 1}
\end{align}
Let $\{\widetilde{P}_i^1, \widetilde{P}_i^2\}$ be the partition of $\widetilde{P}_i$ generated by the algorithm for $2$-clustering of $\widetilde{P}_i$. Under Assumption 1,
\begin{align}
&\approx(\widetilde{P}_i,2) \nonumber \\
&= \sum_{p\in \widetilde{P}_i^1} w_p \dist(p,\mu(\widetilde{P}_i^1))^z  + \sum_{p\in\widetilde{P}_i^2} w_p \dist(p,\mu(\widetilde{P}_i^2))^z \nonumber \\
&\leq \sum_{p\in \widetilde{P}_i^1}w_p \dist(p,\mu(\widetilde{P}_i))^z + \sum_{p\in\widetilde{P}_i^2} w_p \dist(p,\mu(\widetilde{P}_i))^z \nonumber \\
&= \approx(\widetilde{P}_i,1). \label{eq:lem1 suboptimal proof 2}
\end{align}
Combining (\ref{eq:lem1 suboptimal proof 1}, \ref{eq:lem1 suboptimal proof 2}) yields 
\begin{equation}
    \approx(\widetilde{P}_i, 1) - \approx(\widetilde{P}_i, 2) \leq \epsilon', \forall i\in [k].
\end{equation}
\end{proof}

\subsection{Proof of \Cref{lem:2 suboptimal}:}
\begin{proof}
First, by Assumption~1, $\approx(\widetilde{P}_i,1) = \sum_{p\in \widetilde{P}_i}w_p$ $\cdot \dist(p,\mu(\widetilde{P}_i))^z$. Moreover, by Assumption~3,
\begin{align}
\approx(\widetilde{P}_i,2) & \leq \sum_{p\in \widetilde{P}_i}w_p (\min_{x\in \{\mu(\widetilde{P}_i), p^*\}}\dist(p,x))^z \nonumber\\
& \leq \sum_{p\in \widetilde{P}_i\setminus \{p^*\}} w_p \dist(p, \mu(\widetilde{P}_i))^z,
\end{align}
where $p^*\coloneqq \argmax_{p\in \widetilde{P}_i}w_p \dist(p,\mu(\widetilde{P}_i))^z$.
Thus, we have
\begin{align}
&\epsilon' \geq \approx(\widetilde{P}_i, 1)-\approx(\widetilde{P}_i, 2) \nonumber \\
&\geq w_{p^*} \dist(p^*, \mu(\widetilde{P}_i))^z \geq w_p \dist(p, \mu(\widetilde{P}_i))^z,~\forall p\in \widetilde{P}_i.
\end{align}
Therefore, $\dist(p,\mu(\widetilde{P}_i)) \leq ({\epsilon'\over \wmin})^{1\over z}$, $\forall p\in \widetilde{P}_i$.
\end{proof}

\subsection{Proof of \Cref{thm:DRCC_sum-cost_kmedian}:}
\begin{proof}
The proof is based on a sampling lemma from \cite{Balcan13NIPS}:
\begin{lemma}[\cite{Balcan13NIPS}] \label{lem:sampling_lemma}
Let $F$ be a set of nonnegative functions $f: P\to \mathbb{R}_{\geq 0}$. Let $S$ be a set of i.i.d. samples from $P$, where each sample equals $p\in P$ with probability ${m_p\over \sum_{z\in P}m_z}$. Let $u'_q = {\sum_{p\in P}m_p \over m_q |S|}$ for $q\in P$. If for a sufficiently large constant $\alpha$, $|S|\geq {\alpha\over \epsilon^2}(\dim(F,P)+\log{1\over \delta})$, then with probability at least $1-\delta$, $\forall f\in F: \big| \sum_{p\in P}f(p) - \sum_{q\in S}u'_q f(q) \big| \leq \epsilon(\sum_{p\in P}m_p)(\max_{p\in P}{f(p) \over m_p})$.
\end{lemma}

In our case, $S=\bigcup_{j=1}^n S_j$, and $B = \bigcup_{j=1}^n B_j^{k_j}$. Define $f_x(p) \coloneqq w_p(\cost(p,x)-\cost(b_p,x) + \rho \dist(p,b_p))$, where $b_p$ is the center in $B_j^{k_j}$ closest to $p\in P_j$. By the $\rho$-Lipschitz-continuity of $\cost(p,x)$, $f_x(p)\geq 0$ and $f_x(p) \leq 2\rho w_p \dist(p, b_p)$. By lines~\ref{DRCC: allocate t} and \ref{DRCC: sample S_j} in Algorithm~\ref{Alg:distributed RCC}, $k$-median-based DRCC generates $S$ via i.i.d. sampling from $P$ with probabilities proportional to $m_p = w_p \dist(p, b_p)$. Therefore, by \Cref{lem:sampling_lemma}, there exists $|S| = O({1\over \epsilon^2}(\dim(F,P)+\log{1\over \delta}))$ such that with probability at least $1-\delta$, $\forall x\in \mathcal{X}$: 
\begin{align}\label{eq:DRCC proof 1}
\Delta \coloneqq \big|\sum_{p\in P}f_x(p) - \sum_{q\in S}u'_q f_x(q) \big| \leq \epsilon(\sum_{p\in P}m_p)(\max_{p\in P}{f_x(p)\over m_p}).
\end{align}
The righthand side of (\ref{eq:DRCC proof 1}) equals the righthand side of (\ref{eq:DRCC approx error}) as $\sum_{p\in P}m_p = \sum_{j=1}^n c(P_j, B_j^{k_j})$ and ${f_x(p)\over m_p}\leq 2\rho$. 

We will show that the lefthand side of (\ref{eq:DRCC proof 1}) also equals the lefthand side of (\ref{eq:DRCC approx error}). Specifically,  
\begin{align}
\sum_{p\in P} f_x(p) = &\sum_{p\in P}w_p \cost(p,x) - \sum_{b\in B}\cost(b,x)\sum_{p\in P_b}w_p \nonumber \\
&+ \rho \sum_{p\in P} m_p,
\end{align}
and since $u'_q w_q=u_q$ for each $q\in S$ (line~\ref{DRCC: set sample weight} in Algorithm~\ref{Alg:distributed RCC}),
\begin{align}
&\sum_{q\in S}u'_q f_x(q) = \sum_{q\in S}u_q \cost(q,x) \nonumber\\
&-\sum_{b\in B}\cost(b,x)\sum_{q\in P_b\cap S}u_q 
+ \rho \sum_{q\in S}u_q \dist(q,b_q).
\end{align}
Since $\sum_{q\in S}u_q \dist(q,b_q) = \sum_{p\in P}m_p$, 
\begin{align}
\Delta &= \Big|\sum_{p\in P}w_p \cost(p,x) - \sum_{q\in S}u_q \cost(q,x) \nonumber\\
&\hspace{1em} - \sum_{b\in B}\cost(b,x)(\sum_{p\in P_b} w_p - \sum_{q\in P_b\cap S}u_q) \Big|    \nonumber \\
&= \big|\sum_{p\in P}w_p \cost(p,x) -  \sum_{q\in S\cup B}u_q \cost(q,x)\big|,
\end{align}
as $u_b = \sum_{p\in P_b} w_p - \sum_{q\in P_b\cap S}u_q$ (line~\ref{DRCC: set center weight} in Algorithm~\ref{Alg:distributed RCC}).
\end{proof}

\appendixnumbering{B}
\section*{Appendix B: Analysis of Lipschitz Constant}

$\bullet$ For MEB, $\cost(p,x)= \dist(p,x)$, where $x\in \mathbb{R}^d$ denotes the center of the enclosing ball. For any data points $p,p'\in \mathbb{R}^d$, by triangle inequality, we have:
\begin{align}
|\dist(p,x)-\dist(p',x)| \leq \dist(p,p'). 
\end{align}
Hence, its cost function is 1-Lipschitz-continuous ($\rho=1$).

$\bullet$ For $k$-median, $\cost(p,x) = \min_{q\in x}\dist(p,q)$, where $x\subset \mathbb{R}^d$ denotes the set of $k$ centers. For any data points $p,p'\in \mathbb{R}^d$, \emph{without loss of generality (WLOG)}, suppose $\min_{q\in x}\dist(p,q) \geq \min_{q\in x}\dist(p',q) = \dist(p',q')$ for some $q'\in x$. Then
\begin{align}
&|\min_{q\in x}\dist(p,q) -  \min_{q\in x}\dist(p',q)|\nonumber\\
& =  \min_{q\in x}\dist(p,q)  -  \dist(p',q') \nonumber\\
&\leq \dist(p,q') - \dist(p',q') \leq \dist(p,p').
\end{align}
Hence, $\rho=1$ for $k$-median. 

$\bullet$ For $k$-means, $\cost(p,x) = \min_{q\in x}\dist(p,q)^2$, where $x$ denotes the set of $k$ centers. Similar to the above, for any data points $p,p'\in \mathbb{R}^d$,  suppose WLOG that $\min_{q\in x}\dist(p,q) \geq \min_{q\in x}\dist(p',q) = \dist(p',q')$ for some $q'\in x$. Then
\begin{align}
&|\min_{q\in x}\dist(p,q)^2 - \min_{q\in x}\dist(p',q)^2| \nonumber \\
&\leq \dist(p,q')^2 - \dist(p',q')^2 \nonumber\\
& = (\dist(p,q')+\dist(p',q'))(\dist(p,q') - \dist(p',q')) \nonumber \\
&\leq 2\Delta \cdot \dist(p,p'),
\end{align}
where $\Delta$ is the diameter of the sample space (assuming that the centers also reside in the sample space). Hence, $\rho=2\Delta$ for $k$-means. 

$\bullet$ For PCA, $\cost(p,x) = \dist(p,xp)^2$, where $x=W W^T$ is the projection matrix, and $W$ is the matrix consisting of the first $l$ ($l<d$) principle components as columns. For any data points $p,p'\in \mathbb{R}^d$, assuming WLOG that $\dist(p, x p) \geq \dist(p',x p')$, we have
\begin{align}
&|\dist(p, xp)^2 - \dist(p',x p')^2| \nonumber \\
& = (\dist(p, xp)+\dist(p',x p'))  \cdot (\dist(p, xp)-\dist(p',x p')). \label{eq:PCA analysis}
\end{align}
The first factor in (\ref{eq:PCA analysis}) is upper-bounded by $2\Delta$ ($\Delta$: diameter of sample space), as long as the projections $xp,\: xp'$ reside in the sample space. The second factor is upper-bounded by 
\begin{align}
&\hspace{-3em} \dist(p, x p') -\dist(p',x p') + \dist(xp', xp) \nonumber\\
&\leq \dist(p,p') + ||x||_2 \cdot \dist(p,p'),
\end{align}
which is due to triangle inequality and the sub-multiplicative property of $\ell$-2 norm, i.e., $\dist(xp', xp) = ||x(p'-p)||_2\leq ||x||_2 \cdot ||p'-p||_2$. As the principle components are mutually orthogonal unit vectors, we have $||x||_2\leq l$. Plugging these into (\ref{eq:PCA analysis}) gives
\begin{align}
|\dist(p, xp)^2 - \dist(p',x p')^2| 
\leq 2\Delta (l+1) \dist(p,p'),
\end{align}
i.e., $\rho=2\Delta(l+1)$ for PCA. 

$\bullet$ For SVM, the cost function is defined as: $\cost(p,x) = \max(0,1-p_d(p_{1:d-1}^T x_{1:d-1}+x_d))$, where $p_{1:d-1}\in \mathbb{R}^{d-1}$ denotes the numerical portion of a data point $p$, and $p_d\in \{1, -1\}$ denotes its label. Consider two points $p,\: p'$ that are identical with each other except the label, i.e., $p_{1:d-1}=p'_{1:d-1}$ and $p_d = - p'_d$. Suppose that $p_{1:d-1}^T x_{1:d-1}+x_d < -1$, $p_d = 1$, and $p'_d = -1$. Then we have $\dist(p,p') = 2$, $\cost(p',x)=0$, and $\cost(p,x) = 1-(p_{1:d-1}^T x_{1:d-1}+x_d)$. As $\cost(p,x)$ for SVM is unbounded in general, the ratio
\begin{align}
{|\cost(p,x)-\cost(p',x)|\over \dist(p,p')} = {1\over 2}\big[1-(p_{1:d-1}^T x_{1:d-1}+x_d) \big]    
\end{align}
is unbounded. Therefore, $\rho = \infty$ for SVM. 

\appendixnumbering{C}
\section*{Appendix C: Analysis of Dimension of Function Space}

\begin{definition}[\cite{Balcan13NIPS}]\label{def:dimFP}
Let $F\coloneqq \{f_x:\: x\in \mathcal{X}\}$, where each element $f_x: P\to\mathbb{R}_{\geq 0}$ is a function from a set $P$ to nonnegative real numbers. Define $B(x,r)\coloneqq \{p\in P:\: f_x(p)\leq r\}$. The \emph{dimension of the function space $\dim(F,P)$} is the smallest integer $m$ such that 
\begin{align}
|\{S\cap B(x,r):\: x\in \mathcal{X},\: r\geq 0\}| \leq |S|^m,~~~ \forall S\subseteq P.    
\end{align} 
\end{definition}

The dimension of function space is related to the \emph{Vapnik-Chervonenkis (VC) dimension}, defined as follows. We refer to $(P, \mathcal{R})$ as a \emph{range space} if $P$ is a set of points and $\mathcal{R}$ is a family of subsets of $P$. 

\begin{definition}[VC dimension \cite{Har-Peled:11book}]\label{def:VC_dimension}
The VC dimension of a range space $(P,\mathcal{R})$, denoted by $\dVC(P,\mathcal{R})$, 
is the maximum cardinality of a set $S\subseteq P$ that is \emph{shattered} by $\mathcal{R}$, i.e., $\{S\cap R:\: R\in\mathcal{R}\}$ contains all the subsets of $S$. 
\end{definition}

\begin{lemma}[Corollary~5.2.3 \cite{Har-Peled:11book}]\label{lem:dimFP_bound}
If $(P,\mathcal{R})$ is a range space of VC dimension $m$, then for every $S\subseteq P$, we have $|\{S\cap R:\: R\in \mathcal{R}\}|\leq |S|^m$.
\end{lemma}

By \Cref{def:dimFP} and \Cref{lem:dimFP_bound}, we have $\dim(F,P) \leq \dVC(P,\mathcal{R})$ for $\mathcal{R} \coloneqq \{B(x,r):\: x\in\mathcal{X}, r\geq 0\}$. This allows us to bound the dimension of a function space by bounding the VC dimension of the corresponding range space. The VC dimension has an intuitive meaning that it is the number of free parameters to uniquely specify a set in $\mathcal{R}$, e.g., the VC dimension of intervals is $2$, the VC dimension of planar disks is $3$, and the VC dimension of half spaces in $\mathbb{R}^d$ is $d+1$ \cite{Har-Peled:11book}. In our case, we conjecture that the VC dimension is $O(d_{\mathcal{X}})$, where $d_{\mathcal{X}}$ is the number of parameters to uniquely specify an $x\in\mathcal{X}$.

\appendixnumbering{D}
\section*{Appendix D: Additional Evaluations}

\subsection{Error Bound} 

In addition to the cost and accuracy, we have also evaluated the relative error in using the coreset to approximate the cost on the original dataset, defined as $|\cost(P, x_S) - \cost(S, x_S)|/\cost(P, x_S)$, where $x_S$ is the model parameter computed on a coreset $S$.  By Definition~\ref{def:epsilon-coreset}, this error should be upper-bounded by $\epsilon$ if $S$ is an $\epsilon$-coreset. 
In Tables~\ref{tab:epsilon, z=1}--\ref{tab:epsilon, z=2}, we show the maximum relative approximation error over all the Monte Carlo runs for MEB, together with the value of $\epsilon$ computed according to (\ref{eq:epsilon - opt, new}). 
Indeed, the error is always upper-bounded by $\epsilon$. We note that $\epsilon$ tends to grow with the dimension of the dataset and can be greater than one. This is because the bound in (\ref{eq:epsilon - opt, new}) is based on the maximum Euclidean distance between a data point and the corresponding coreset point representing it, which tends to grow with the dimension of the dataset. 
Meanwhile, we also note that the bound tends to be loose, and the actual approximation error can be much smaller than $\epsilon$. This is because 
by Definition~\ref{def:epsilon-coreset}, $\epsilon$ needs to upper-bound the relative approximation error for any model parameter $x\in \mathcal{X}$, and may thus be loose at the computed model parameter $x_S$. 
We note that the bounds for the other machine learning problems ($k$-means, PCA, SVM, NN) equal the bound for MEB scaled by their corresponding Lipschitz constants, which can be large or even infinite (see Table~\ref{tab:cost}).  Nevertheless, our  experiments have shown that the proposed coresets can be used to train these models with competitive performances in both cost and accuracy. 

\begin{table}[tb]
\footnotesize
\renewcommand{\arraystretch}{1.3}
\caption{Error bound $\epsilon$ for MEB ($z=1$)} \label{tab:epsilon, z=1}
\centering
\begin{tabular}{r|c|c|c}
  \hline
  dataset & coreset size &  max relative error & $\epsilon$  \\
  \hline
   Fisher's iris & 20  & 0.0053   & 0.1073   \\
  \hline
  Facebook & 80 & 0.0344 & 1.18  \\
  \hline
  Pendigits &  400 & 0.0026 & 2.1112 \\
  \hline
  MNIST & 400 & 0.0024 & 10.74 \\
  \hline
  HAR & 400 & 5.6054e-05 & 3.9212 \\  
  \hline
\end{tabular}
\end{table}
\normalsize

\begin{table}[tb]
\footnotesize
\renewcommand{\arraystretch}{1.3}
\caption{Error bound $\epsilon$ for MEB ($z=2$)} \label{tab:epsilon, z=2}
\centering
\begin{tabular}{r|c|c|c}
  \hline
  dataset & coreset size &  max relative error & $\epsilon$  \\
  \hline
  Fisher's iris & 20 & 1.1896e-05  & 0.1093   \\
  \hline
  Facebook & 80 & 1.9976e-06 & 1.3711  \\
  \hline
  Pendigits & 400 & 4.3876e-05 & 2.0257 \\
  \hline
  MNIST & 400 & 0.0020 & 8.53 \\
  \hline
  HAR &  400  &  4.9972e-07  &  3.5612  \\
  \hline
\end{tabular}
\end{table}
\normalsize

\subsection{Original Machine Learning Performance}

We provide the performance of applying each machine learning algorithm directly to the original dataset in Table~\ref{tab:original cost}. For unsupervised learning problems (MEB, $k$-means, PCA), these were the denominators in evaluating the normalized costs of models learned on coresets. Each machine learning problem is parameterized as in Section~\ref{sec:Performance Evaluation}. 

\begin{table}[t]
{
\footnotesize
\renewcommand{\arraystretch}{1.3}
\caption{Original Machine Learning Performance (cost for MEB, $k$-means, PCA; accuracy for SVM and NN)} \label{tab:original cost}
\centerline{
\begin{tabular}{r|c|c|c|c|c}
  \hline
  problem & Fisher's & Facebook & Pendigits & MNIST & HAR\\
  \hline
MEB       & 2.05  & 7.65    & 18.06   & 90.43 & 60.26 \\
\hline
$k$-means & 84.54 & 853.67  & 2.00e+05 & 4.98e+07 & 2.86e+06 \\
\hline 
PCA       & 1.94  & 197.19  & 910.79  & 1.73e+04 & 4.46e+05 \\
\hline
SVM/NN\footnotemark       & 100\% & 89.36\% & 99.33\% & 87.01\% & 78.01\% \\
  \hline
\end{tabular}
}
}
 \vspace{.5em}
\end{table}
\footnotetext{As mentioned before, for Fisher's iris, Facebook, and Pendigits, we evaluated SVM accuray; for MNIST and HAR, we evaluated NN accuracy. }

\subsection{Additional Evaluations in Distributed Setting}

To evaluate the impact of the coreset size ($N$) and the maximum number of local centers ($K$), we vary each of these parameters while fixing the data distribution scheme as the `uniform' scheme. Figures~\ref{fig:MNIST_N_test}--\ref{fig:HAR_N_test} are the results of varying $N$ under fixed $K$. As expected, all the algorithms benefit from a larger coreset size, but we observe more significant improvements for the distributed algorithms (CDCC and DRCC). 
Figures~\ref{fig:MNIST_K_test}--\ref{fig:HAR_K_test} show the results of varying $K$ under fixed $N$. Although RCC-kmeans and CDCC do not depend on $K$, we still show them as benchmarks. Overall, increasing $K$ improves the quality of the models trained on the coreset generated by DRCC, as it gives DRCC more space to optimize its configuration. Specifically, we see that increasing $K$ from $1$ to $2$ notably improves the performance in most cases, but further increasing $K$ does not bring notable improvements for the unsupervised learning problems. The accuracy of NN, in contrast, keeps improving as $K$ increases. We note that the small dip at $K=2$ in the $k$-means cost is because this parameter leads to $2$ local centers at most nodes, which coincides with the number of global centers we are computing. 

\begin{figure}[tb]
\begin{minipage}{0.24\textwidth}
\centerline{
\includegraphics[width=1.05\textwidth,height=3.7cm]{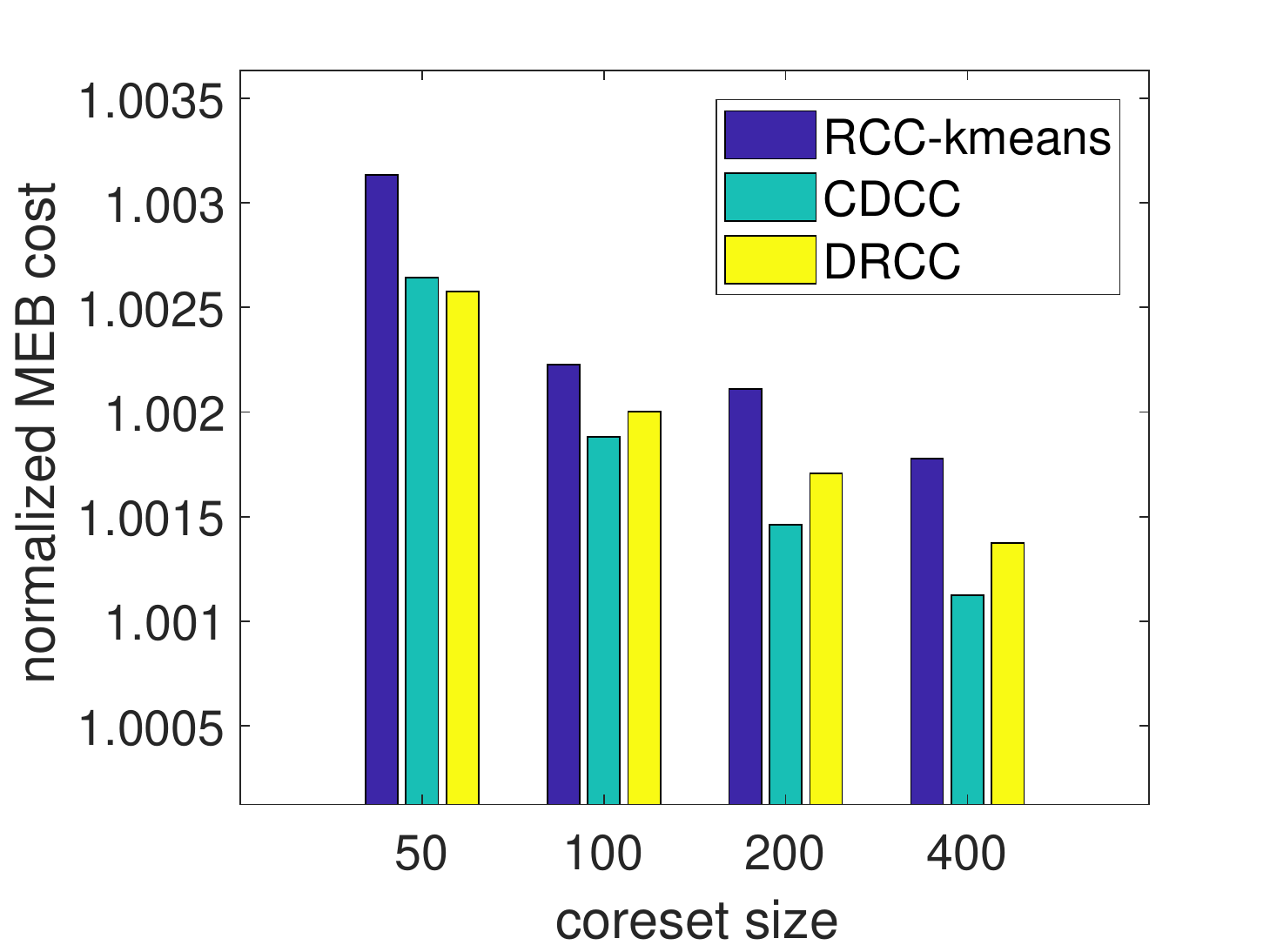}}
\centerline{\scriptsize (a) MEB}
\end{minipage}
\begin{minipage}{0.24\textwidth}
\centerline{
\includegraphics[width=1.05\textwidth,height=3.7cm]{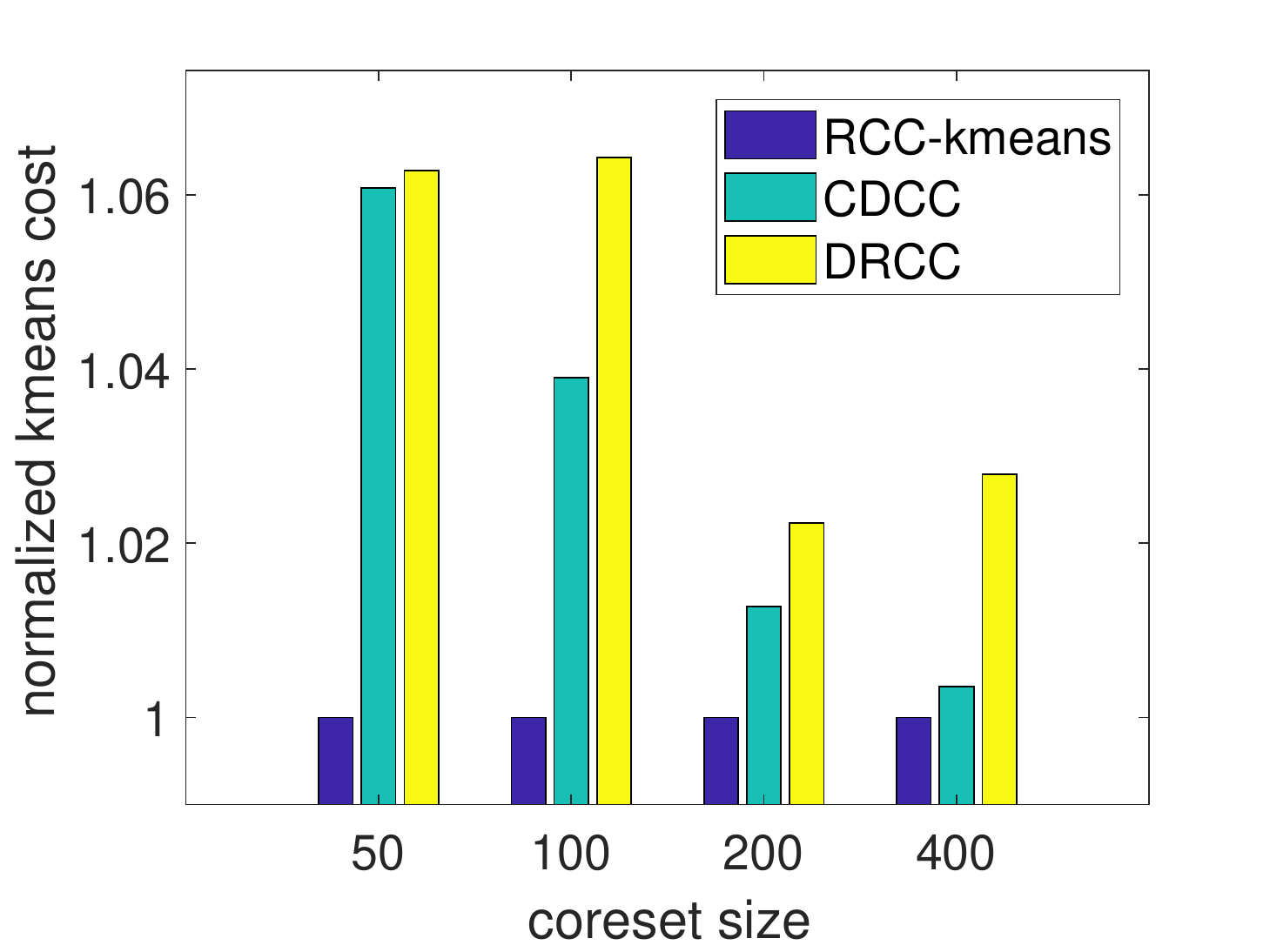}}
\centerline{\scriptsize (b) $k$-means ($k=2$)}
\end{minipage}
  \begin{minipage}{.24\textwidth}
  \centerline{
   \includegraphics[width=1.05\textwidth,height=3.7cm]{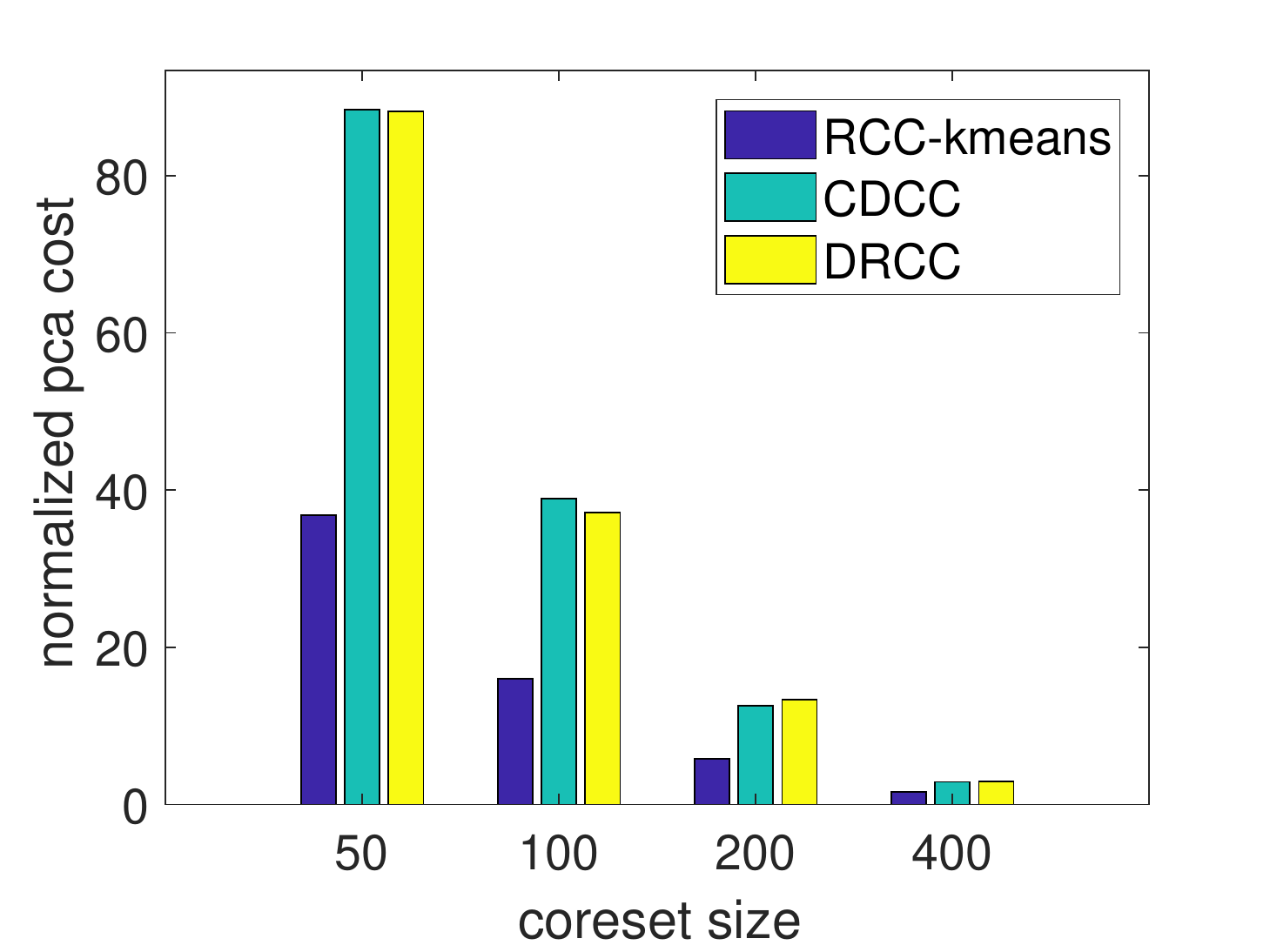}}
    \centerline{\scriptsize (c) PCA (300 components) }
  \end{minipage}
  \begin{minipage}{0.24\textwidth}
    \centerline{
  \includegraphics[width=1.05\textwidth,height=3.7cm]{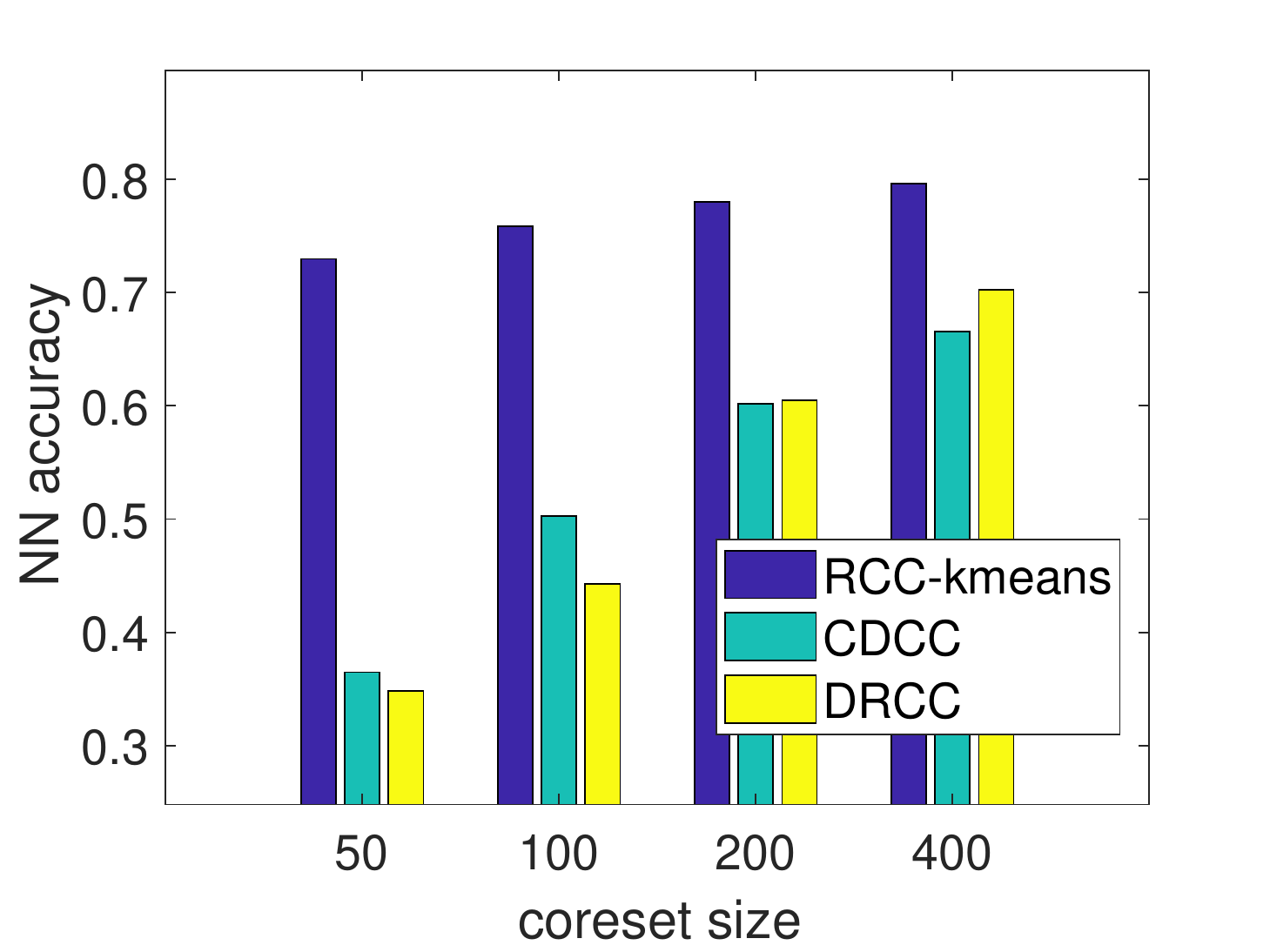}} 
\centerline{\scriptsize (d) NN} 
   \end{minipage}    
\caption{Evaluation on MNIST in distributed setting with varying coreset size ($K=10$). }
\label{fig:MNIST_N_test}
\end{figure}

\begin{figure}[tb]
\begin{minipage}{0.24\textwidth}
\centerline{
\includegraphics[width=1.05\textwidth,height=3.7cm]{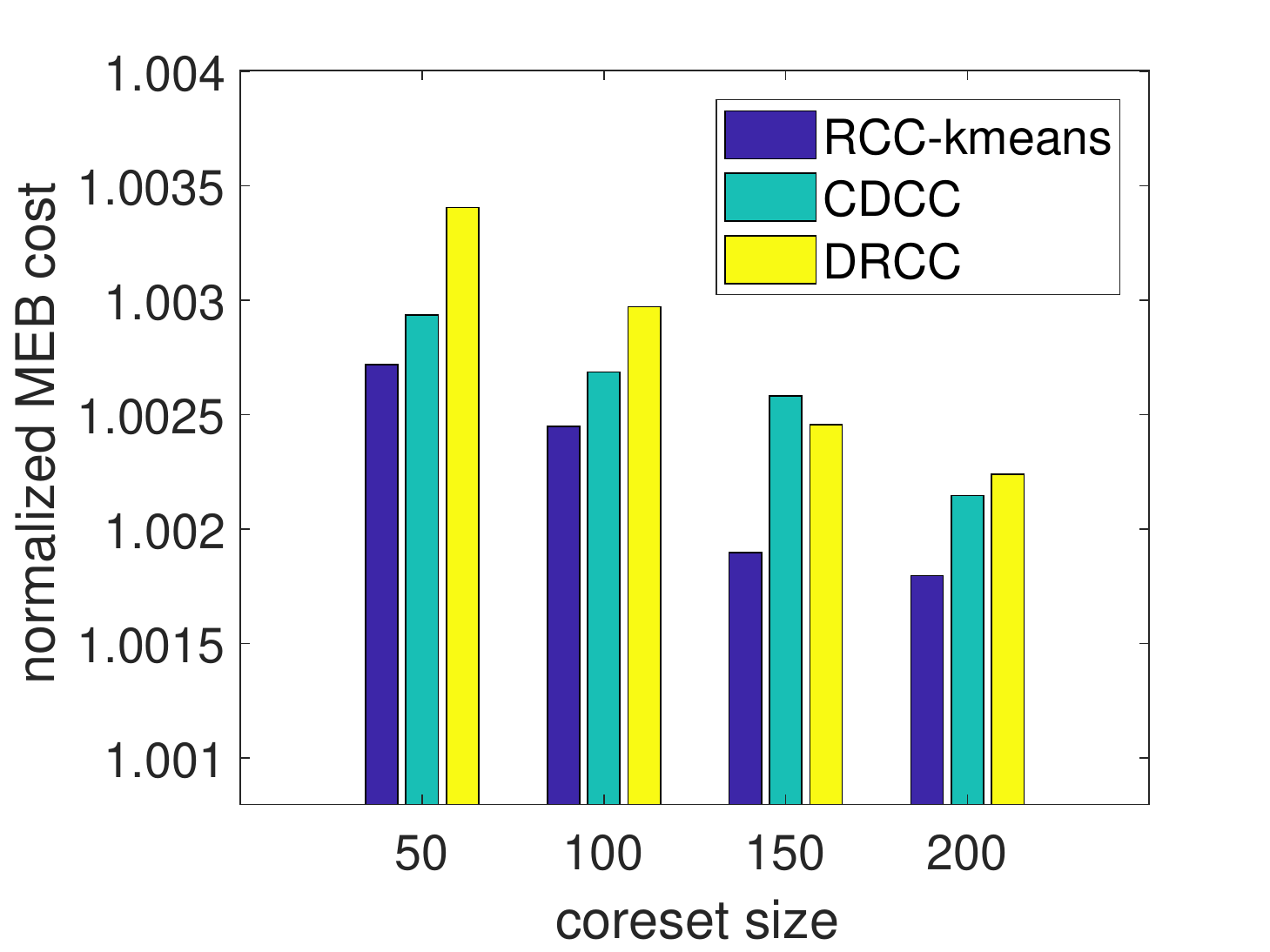}}
\centerline{\scriptsize (a) MEB}
\end{minipage}
\begin{minipage}{0.24\textwidth}
\centerline{
\includegraphics[width=1.05\textwidth,height=3.7cm]{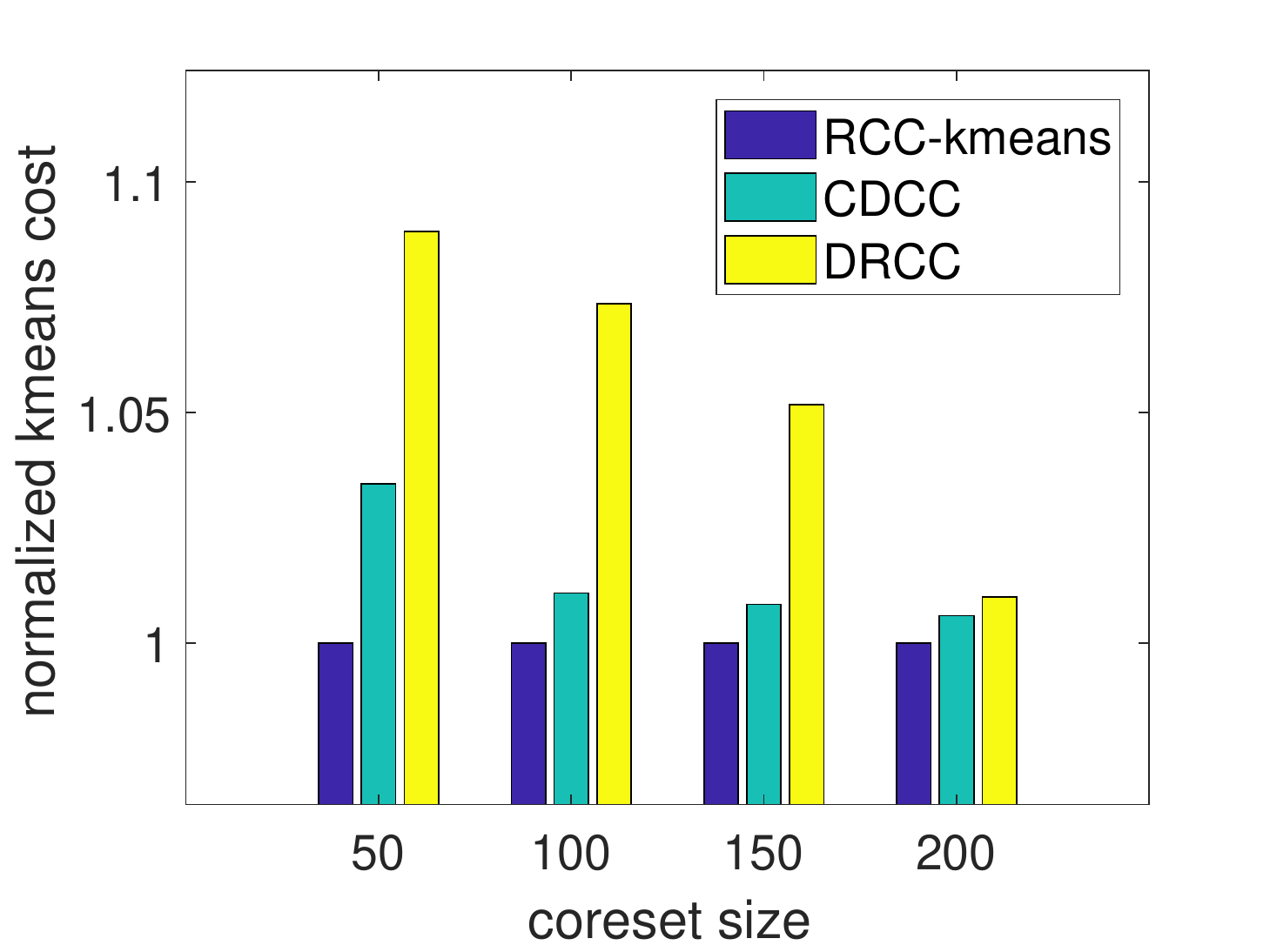}}
\centerline{\scriptsize (b) $k$-means ($k=2$)}
\end{minipage}
  \begin{minipage}{.24\textwidth}
  \centerline{
   \includegraphics[width=1.05\textwidth,height=3.7cm]{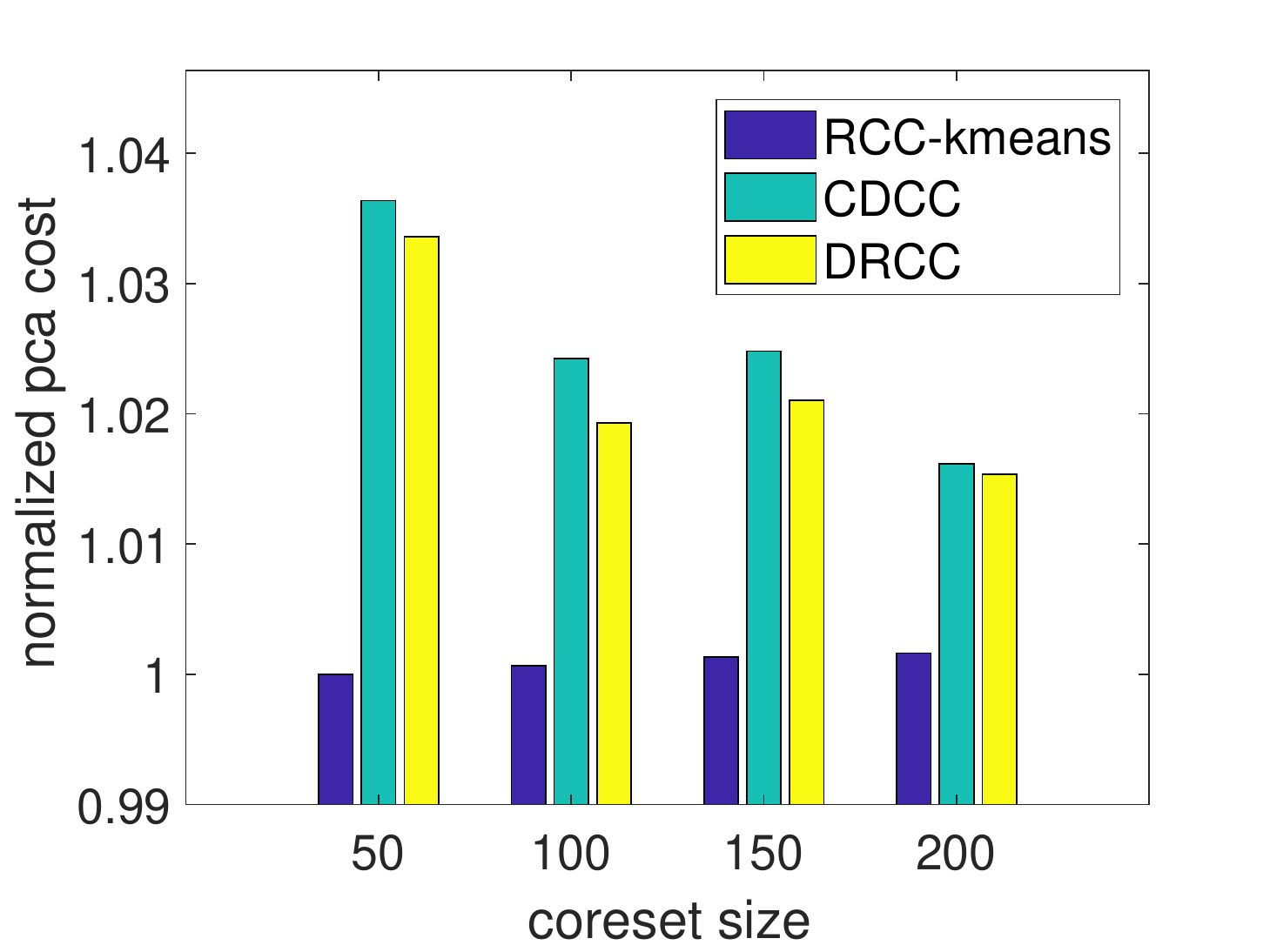}}
    \centerline{\scriptsize (c) PCA (7 components) }
  \end{minipage}
  \begin{minipage}{0.24\textwidth}
    \centerline{
  \includegraphics[width=1.05\textwidth,height=3.7cm]{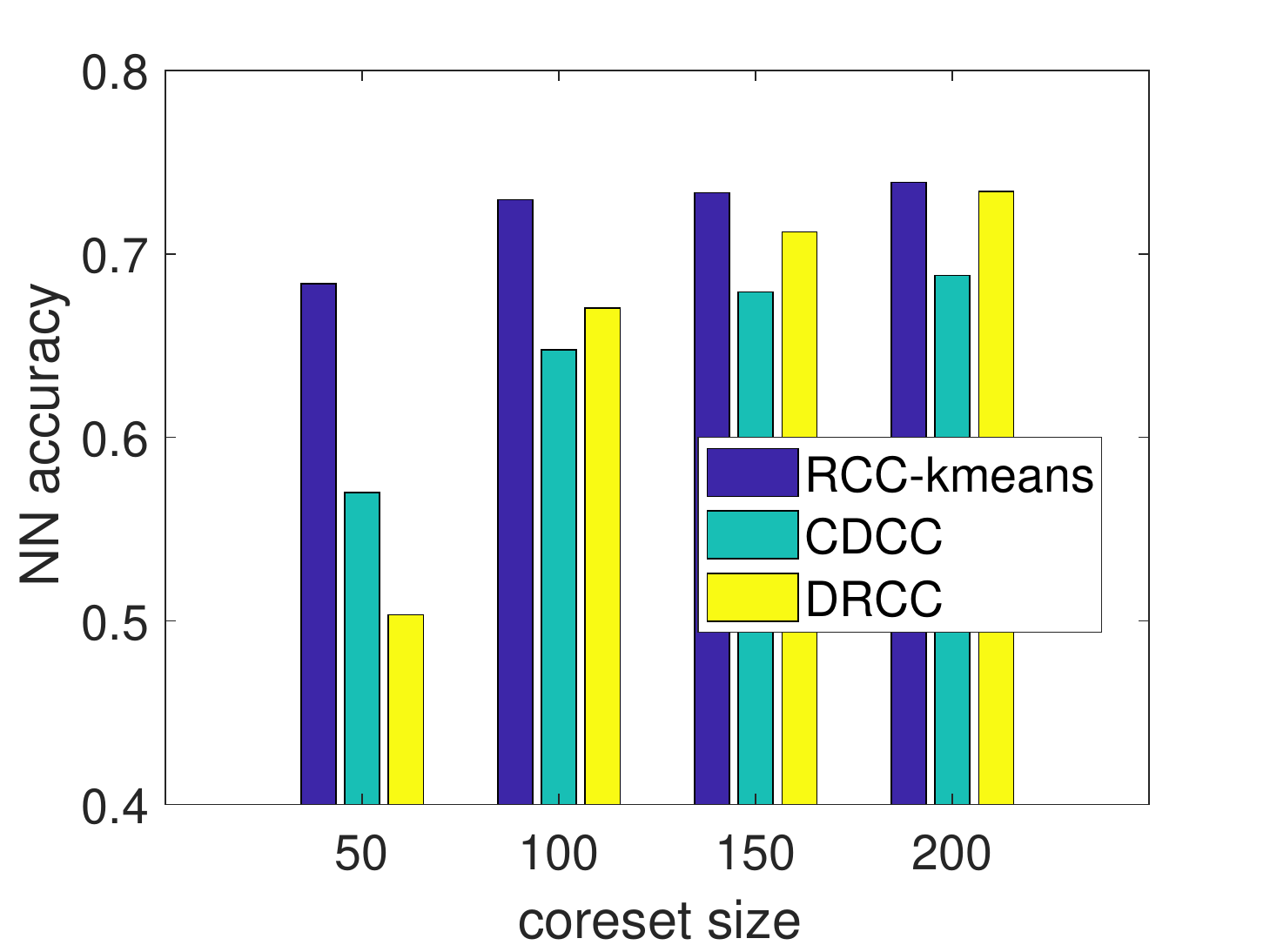}} 
\centerline{\scriptsize (d) NN} 
   \end{minipage}    
\caption{Evaluation on HAR in distributed setting with varying coreset size ($K=10$). }
\label{fig:HAR_N_test}
\end{figure}

\begin{figure}[tb]
\begin{minipage}{0.24\textwidth}
\centerline{
\includegraphics[width=1.05\textwidth,height=3.7cm]{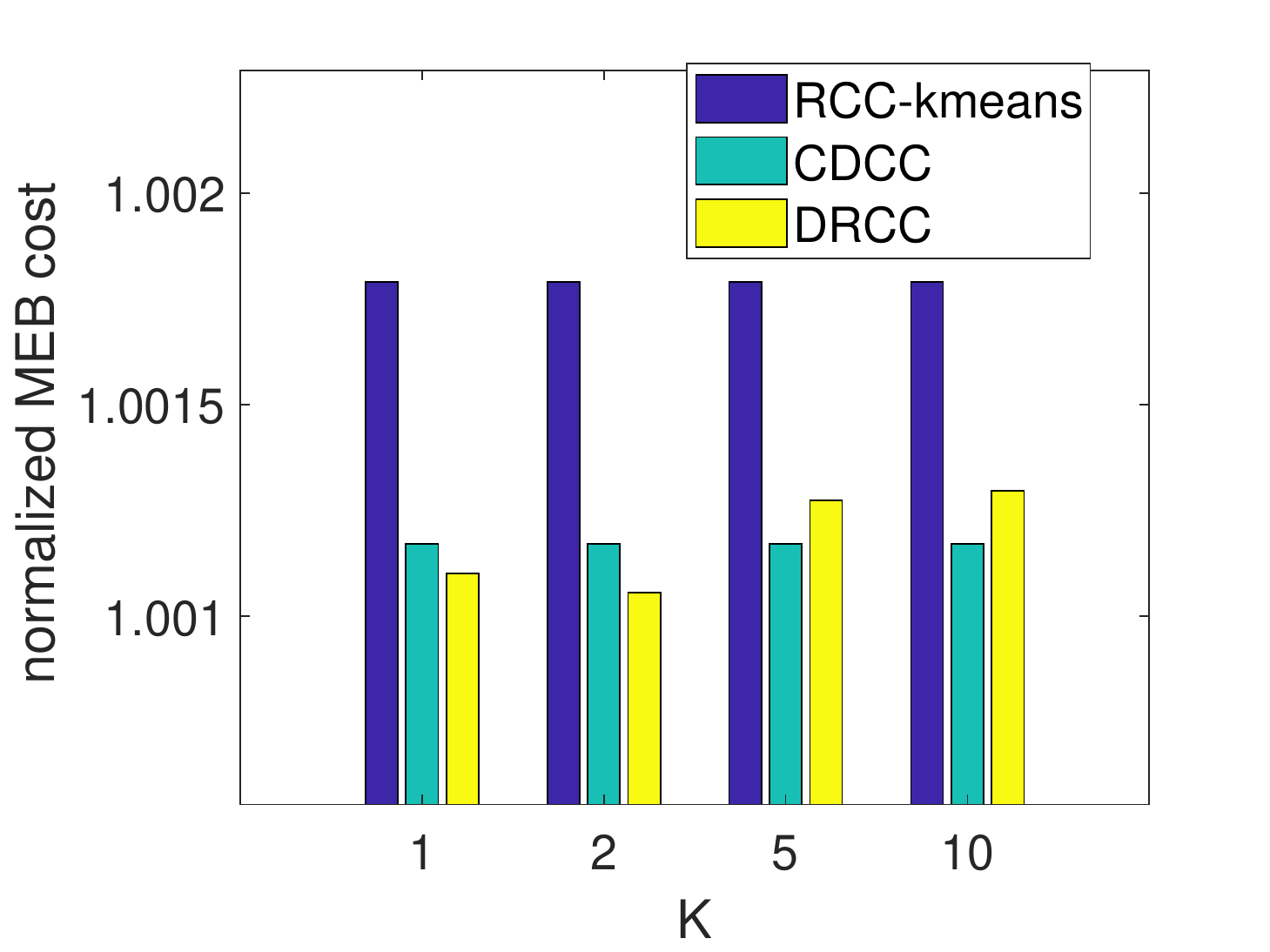}}
\centerline{\scriptsize (a) MEB}
\end{minipage}
\begin{minipage}{0.24\textwidth}
\centerline{
\includegraphics[width=1.05\textwidth,height=3.7cm]{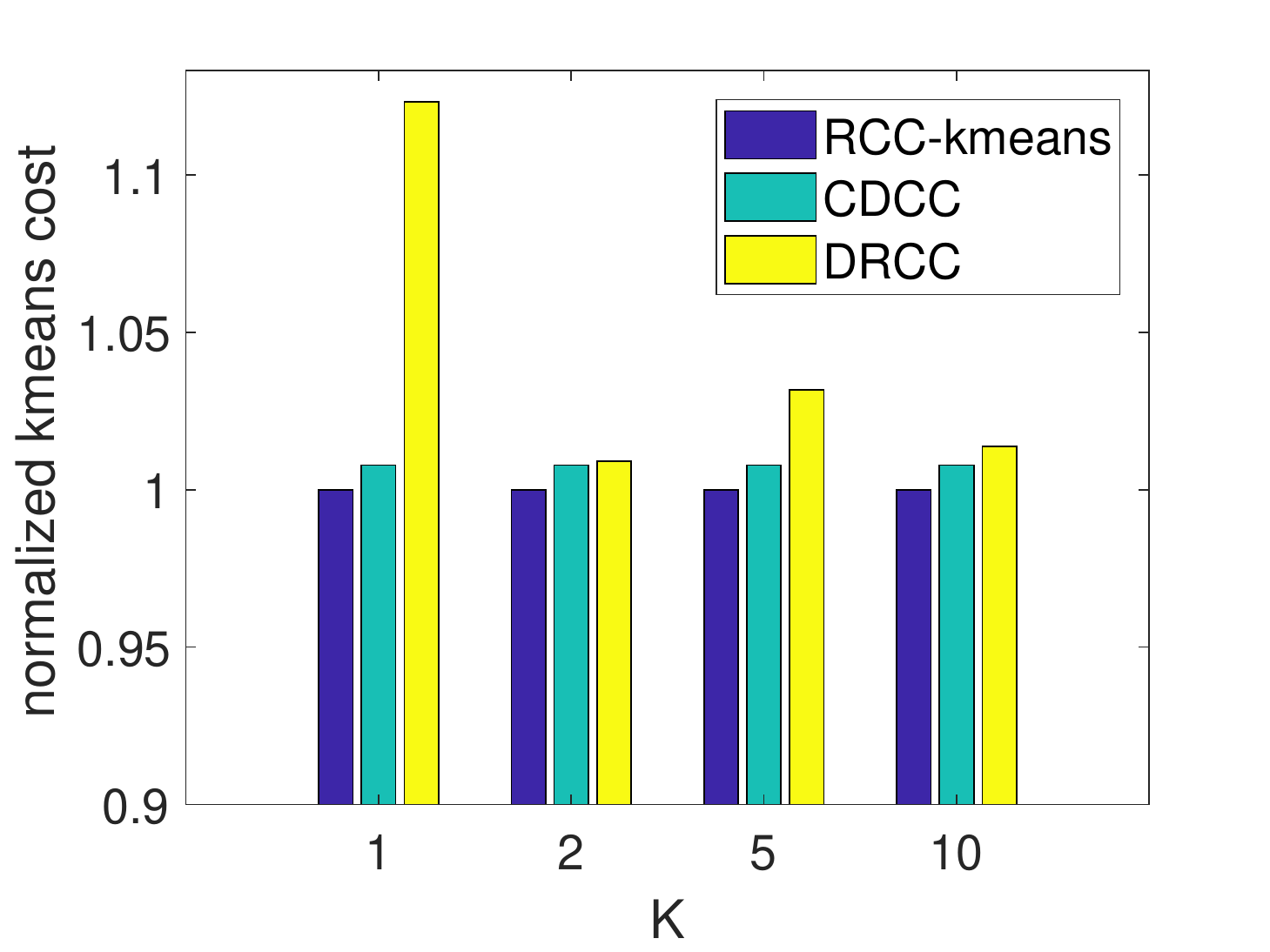}}
\centerline{\scriptsize (b) $k$-means ($k=2$)}
\end{minipage}
  \begin{minipage}{.24\textwidth}
  \centerline{
   \includegraphics[width=1.05\textwidth,height=3.7cm]{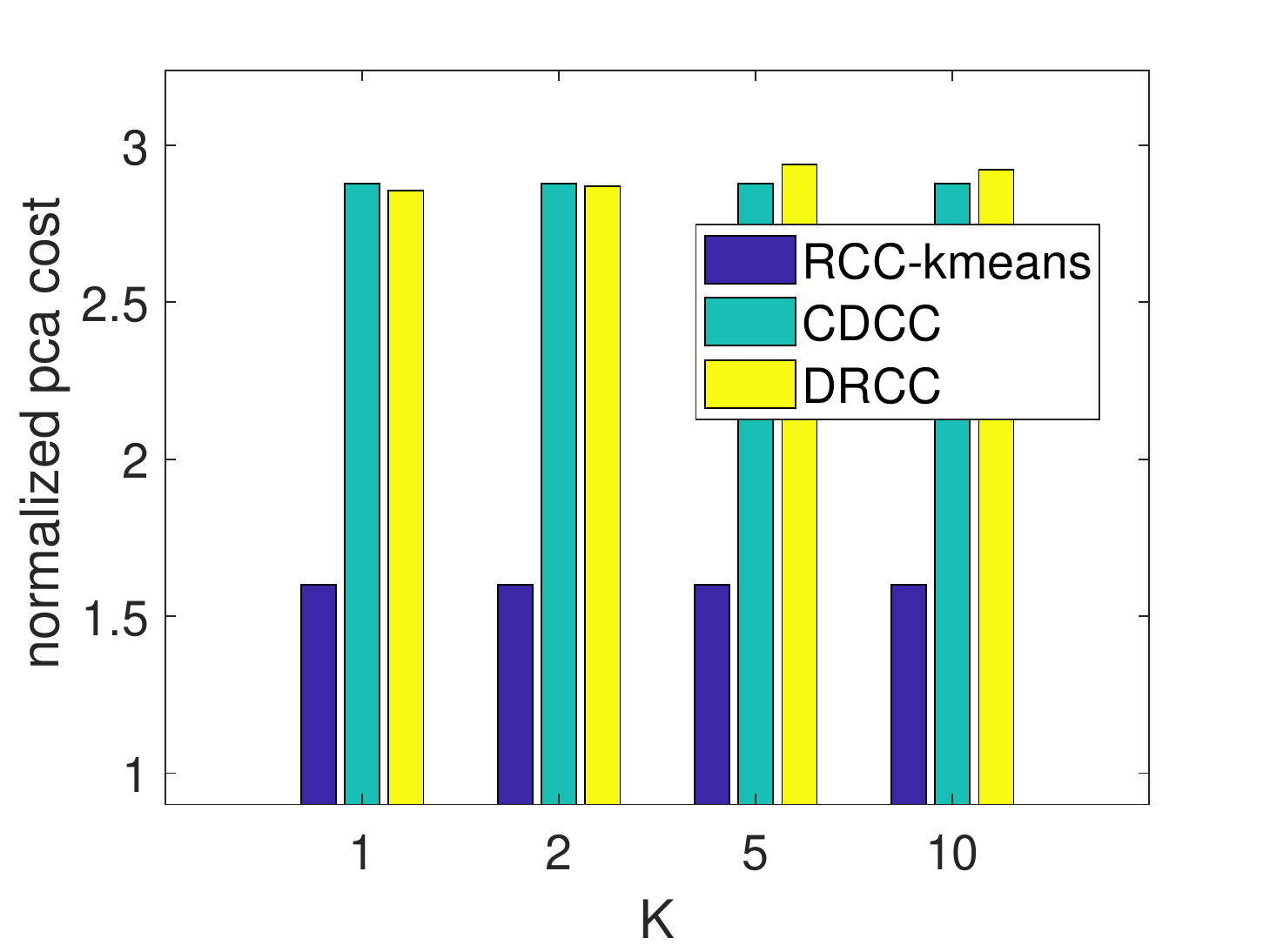}}
    \centerline{\scriptsize (c) PCA (300 components) }
  \end{minipage}
  \begin{minipage}{0.24\textwidth}
    \centerline{
  \includegraphics[width=1.05\textwidth,height=3.7cm]{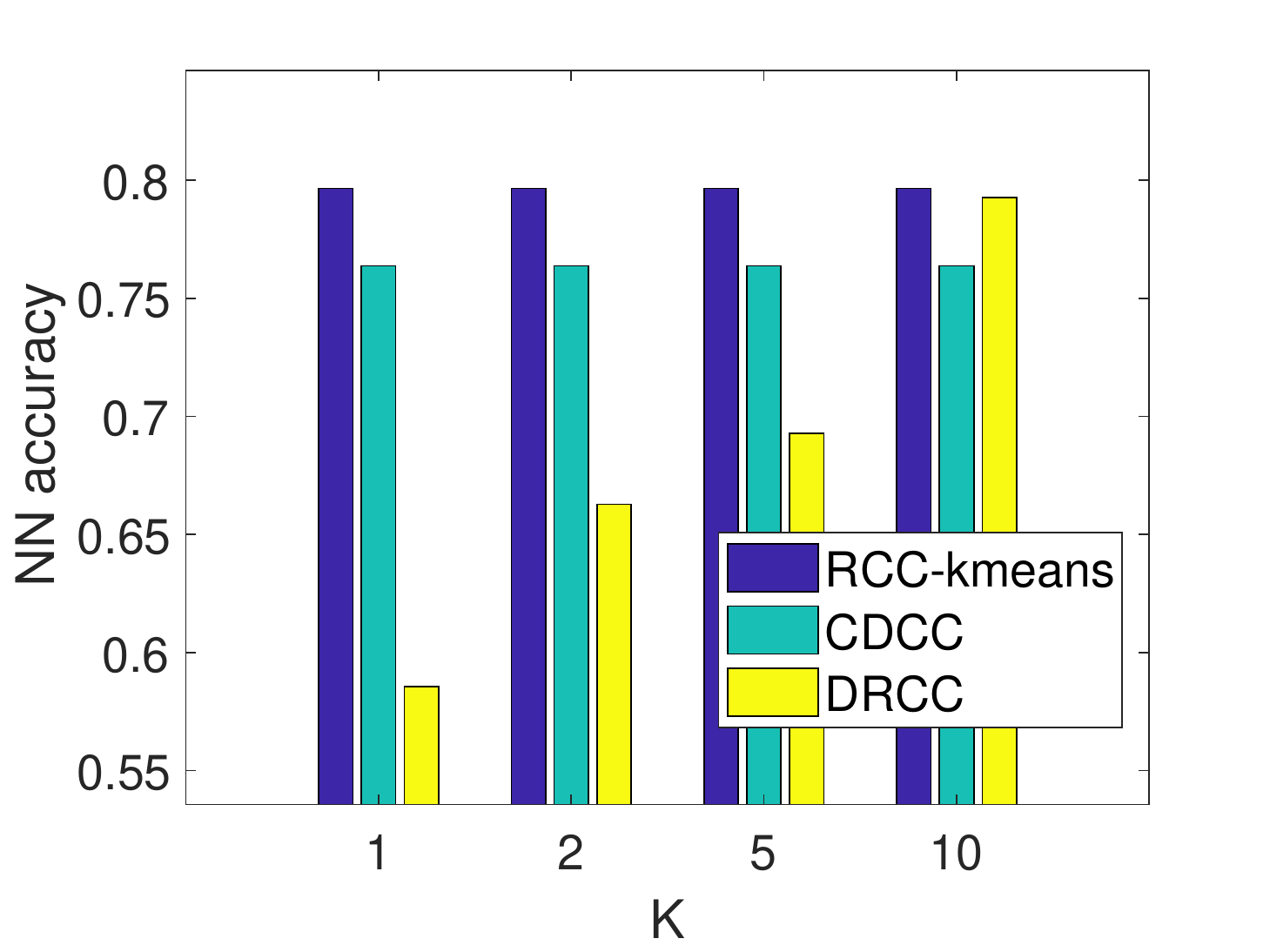}} 
\centerline{\scriptsize (d) NN} 
   \end{minipage}    
\caption{Evaluation on MNIST in distributed setting with varying $K$ ($N=400$). }
\label{fig:MNIST_K_test}
   \vspace{-0.2in}
\end{figure}

\begin{figure}[tb]
\begin{minipage}{0.24\textwidth}
\centerline{
\includegraphics[width=1.05\textwidth,height=3.7cm]{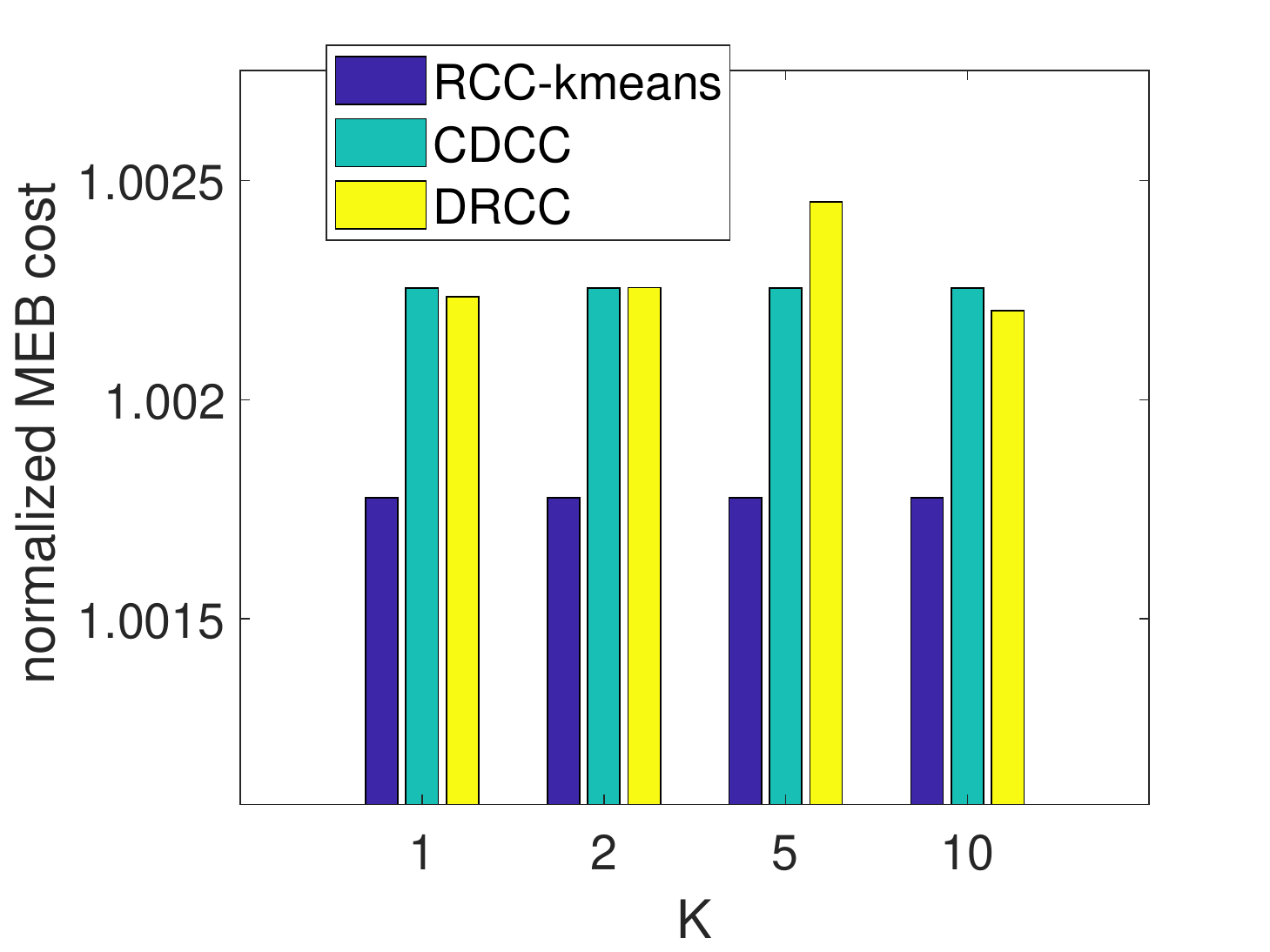}}
\centerline{\scriptsize (a) MEB}
\end{minipage}
\begin{minipage}{0.24\textwidth}
\centerline{
\includegraphics[width=1.05\textwidth,height=3.7cm]{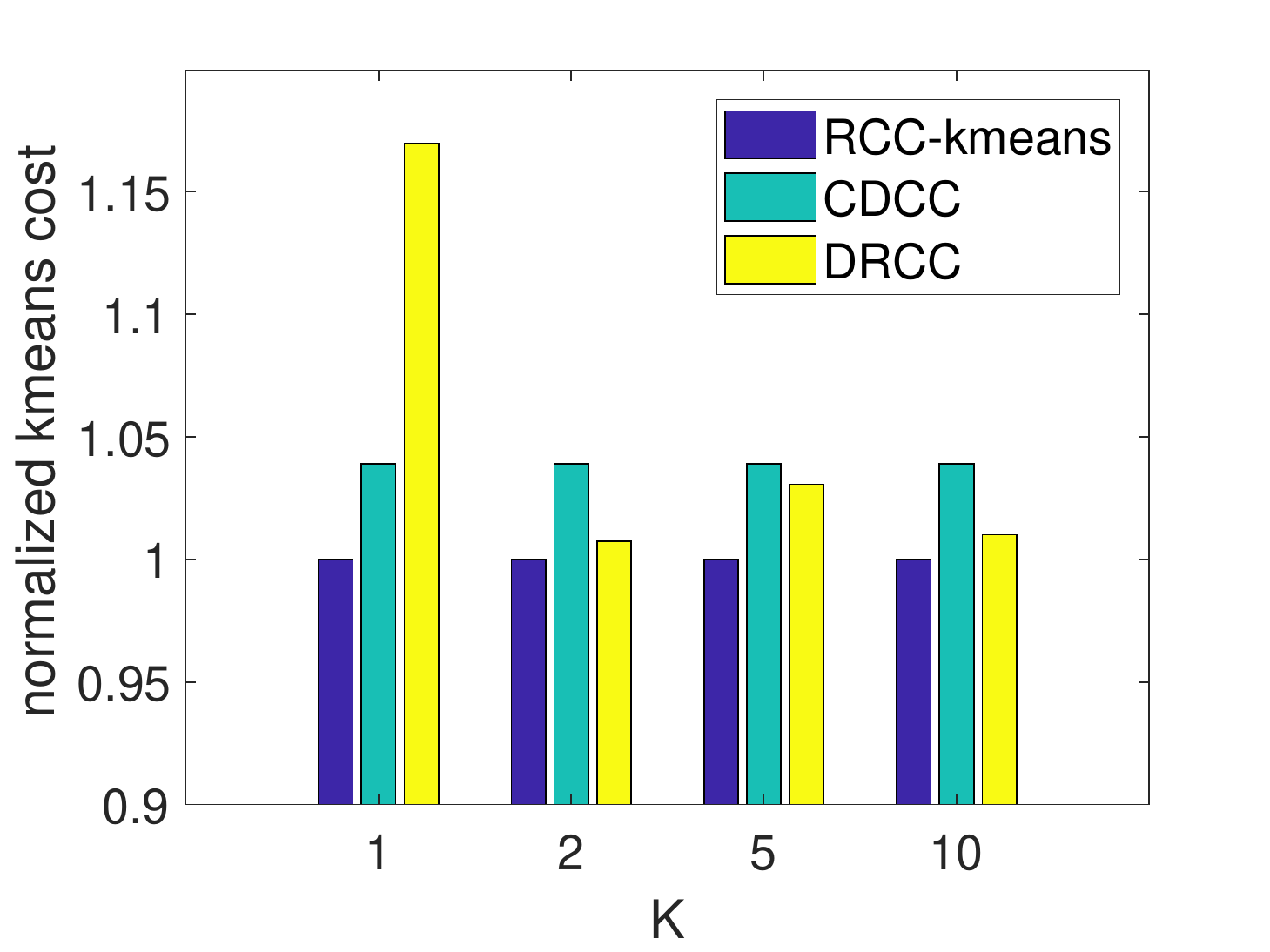}}
\centerline{\scriptsize (b) $k$-means ($k=2$)}
\end{minipage}
  \begin{minipage}{.24\textwidth}
  \centerline{
   \includegraphics[width=1.05\textwidth,height=3.7cm]{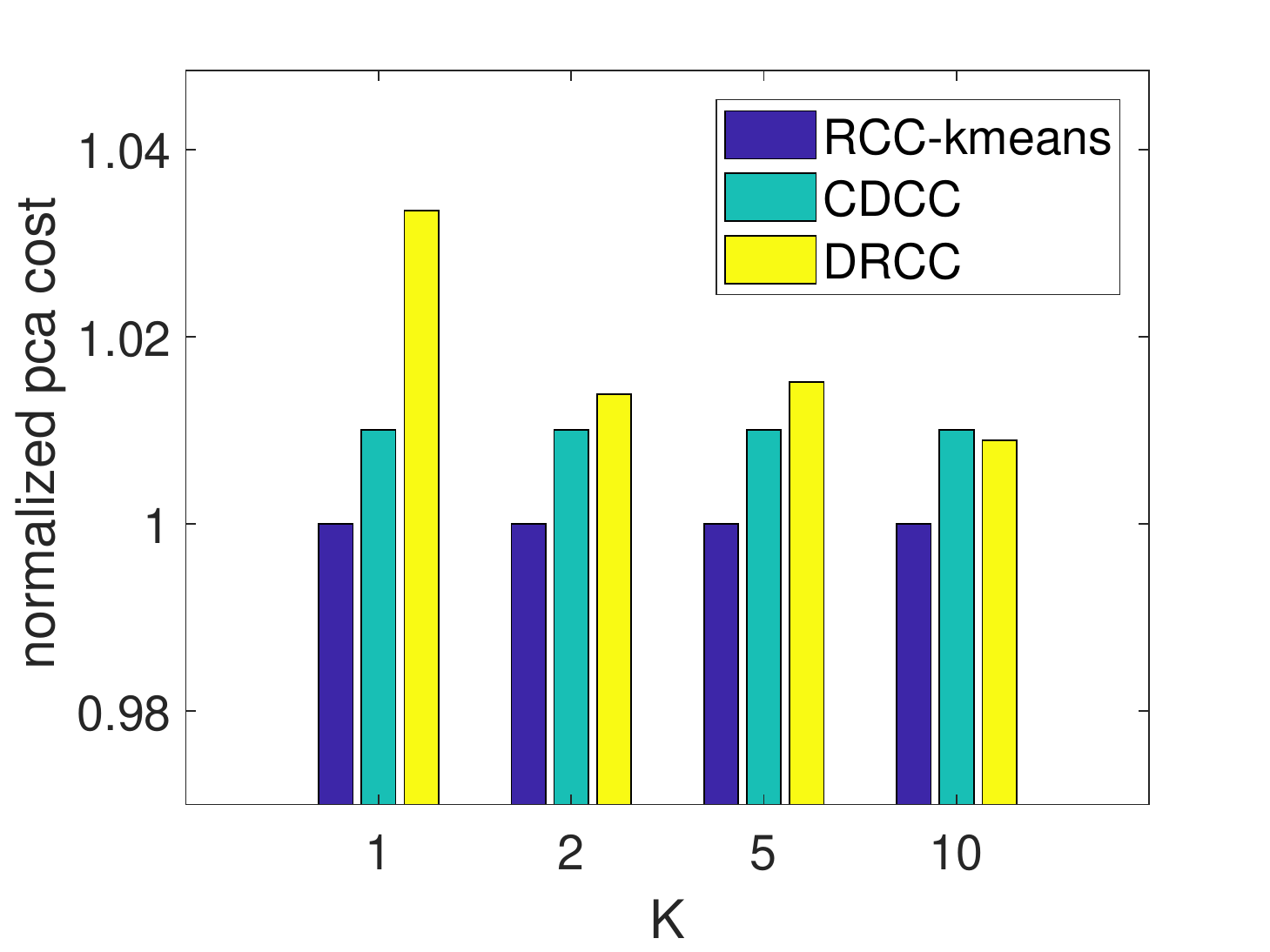}}
    \centerline{\scriptsize (c) PCA (7 components) }
  \end{minipage}
  \begin{minipage}{0.24\textwidth}
    \centerline{
  \includegraphics[width=1.05\textwidth,height=3.7cm]{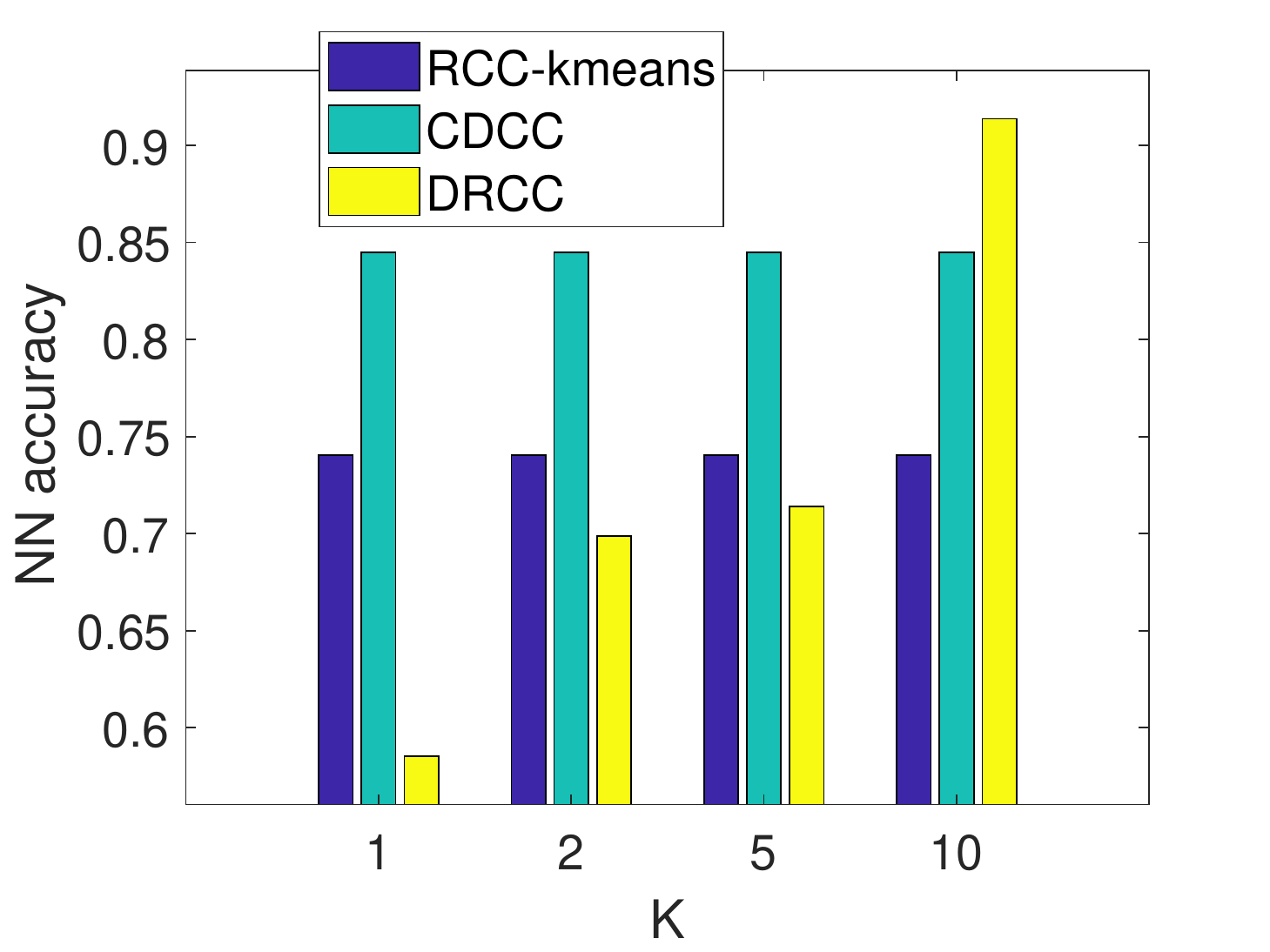}} 
\centerline{\scriptsize (d) NN} 
   \end{minipage}    
\caption{Evaluation on HAR in distributed setting with varying $K$ ($N=200$). }
\label{fig:HAR_K_test}
\end{figure}

\bibliographystyle{IEEEtran}
\bibliography{mybibSimplified}

\begin{thebibliography}{10}
\providecommand{\url}[1]{#1}
\csname url@samestyle\endcsname
\providecommand{\newblock}{\relax}
\providecommand{\bibinfo}[2]{#2}
\providecommand{\BIBentrySTDinterwordspacing}{\spaceskip=0pt\relax}
\providecommand{\BIBentryALTinterwordstretchfactor}{4}
\providecommand{\BIBentryALTinterwordspacing}{\spaceskip=\fontdimen2\font plus
\BIBentryALTinterwordstretchfactor\fontdimen3\font minus
  \fontdimen4\font\relax}
\providecommand{\BIBforeignlanguage}[2]{{%
\expandafter\ifx\csname l@#1\endcsname\relax
\typeout{** WARNING: IEEEtran.bst: No hyphenation pattern has been}%
\typeout{** loaded for the language `#1'. Using the pattern for}%
\typeout{** the default language instead.}%
\else
\language=\csname l@#1\endcsname
\fi
#2}}
\providecommand{\BIBdecl}{\relax}
\BIBdecl

\bibitem{Peteiro-Barral12PAI}
D.~Peteiro-Barral and B.~Guijarro-Berdinas, ``A survey of methods for
  distributed machine learning,'' in \emph{Progress in Artificial
  Intelligence}, November 2012.

\bibitem{McMahan16AISTATS}
H.~B. McMahan, E.~Moore, D.~Ramage, S.~Hampson, and B.~A. Arcas,
  ``Communication-efficient learning of deep networks from decentralized
  data,'' in \emph{AISTATS}, May 2016.

\bibitem{Wang18INFOCOM}
S.~Wang, T.~Tuor, T.~Salonidis, K.~K. Leung, C.~Makaya, T.~He, and K.~Chan,
  ``When edge meets learning: Adaptive control for resource-constrained
  distributed machine learning,'' in \emph{IEEE INFOCOM}, April 2018.

\bibitem{konevcny2016federated}
J.~Kone{\v{c}}n{\`y}, H.~B. McMahan, F.~X. Yu, P.~Richt{\'a}rik, A.~T. Suresh,
  and D.~Bacon, ``Federated learning: Strategies for improving communication
  efficiency,'' \emph{arXiv preprint arXiv:1610.05492}, 2016.

\bibitem{Balcan13NIPS}
M.~F. Balcan, S.~Ehrlich, and Y.~Liang, ``Distributed k-means and k-median
  clustering on general topologies,'' in \emph{NIPS}, 2013.

\bibitem{Kannan14COLT}
R.~Kannan, S.~Vempala, and D.~Woodruff, ``Principal component analysis and
  higher correlations for distributed data,'' in \emph{COLT}, 2014.

\bibitem{Barger16SDM}
A.~Barger and D.~Feldman, ``k-means for streaming and distributed big sparse
  data,'' in \emph{SDM}, 2016.

\bibitem{Badoiu02STOC}
M.~B\={a}doiu, S.~Har-Peled, and P.~Indyk, ``Approximate clustering via
  core-sets,'' in \emph{ACM STOC}, 2002.

\bibitem{lu19GLOBECOM}
H.~{Lu}, M.~{Li}, T.~{He}, S.~{Wang}, V.~{Narayanan}, and K.~S. {Chan},
  ``Robust coreset construction for distributed machine learning,'' in
  \emph{2019 IEEE Global Communications Conference (GLOBECOM)}, 2019.

\bibitem{Chan93AAAI}
P.~K. Chan and S.~J. Stolfo, ``Toward parallel and distributed learning by
  meta-learning,'' in \emph{AAAI Workshop in Knowledge Discovery in Databases},
  1997.

\bibitem{Kittler98IEEE}
J.~Kittler, M.~Hatef, R.~P. Duin, and J.~Matas, ``On combining classifiers,''
  \emph{IEEE Transactions on Pattern Analysis and Machine Intelligence},
  vol.~20, no.~3, pp. 226--239, March 1998.

\bibitem{Wolpert92NN}
D.~Wolpert, ``Stacked generalization,'' \emph{Neural Networks}, vol.~5, no.~2,
  pp. 241--259, 1992.

\bibitem{Tsoumakas02ECAI}
G.~Tsoumakas and I.~Vlahavas, ``Effective stacking of distributed
  classifiers,'' in \emph{ECAI}, 2002.

\bibitem{Guo02LNCS}
Y.~Guo and J.~Sutiwaraphun, ``Probing knowledge in distributed data mining,''
  in \emph{Methodologies for Knowledge Discovery and Data Mining}, 1999.

\bibitem{Chawla04JMLR}
N.~Chawla, L.~Halla, K.~Bowyer, and W.~Kegelmeyer, ``Learning ensembles from
  bites: A scalable and accurate approach,'' \emph{Journal of Machine Learning
  Research}, vol.~5, pp. 421--451, April 2004.

\bibitem{Lazarevic02DPD}
A.~Lazarevic and Z.~Obradovic, ``Boosting algorithms for parallel and
  distributed learning,'' \emph{Distributed and Parallel Databases}, vol.~11,
  no.~2, pp. 203--229, March 2002.

\bibitem{smith2017federated}
V.~Smith, C.-K. Chiang, M.~Sanjabi, and A.~S. Talwalkar, ``Federated multi-task
  learning,'' in \emph{Advances in Neural Information Processing Systems},
  2017, pp. 4424--4434.

\bibitem{Phillips16CoRR}
J.~M. Phillips, ``Coresets and sketches,'' \emph{CoRR}, vol. abs/1601.00617,
  2016.

\bibitem{Munteanu18KI}
A.~Munteanu and C.~Schwiegelshohn, ``Coresets-methods and history: A
  theoreticians design pattern for approximation and streaming algorithms,''
  \emph{KI - K{\"u}nstliche Intelligenz}, vol.~32, no.~1, pp. 37--53, 2018.

\bibitem{wang2017sketching}
J.~Wang, J.~D. Lee, M.~Mahdavi, M.~Kolar, N.~Srebro \emph{et~al.}, ``Sketching
  meets random projection in the dual: A provable recovery algorithm for big
  and high-dimensional data,'' \emph{Electronic Journal of Statistics},
  vol.~11, no.~2, pp. 4896--4944, 2017.

\bibitem{makarychev2019performance}
K.~Makarychev, Y.~Makarychev, and I.~Razenshteyn, ``Performance of
  johnson-lindenstrauss transform for k-means and k-medians clustering,'' in
  \emph{Proceedings of the 51st Annual ACM SIGACT Symposium on Theory of
  Computing}.\hskip 1em plus 0.5em minus 0.4em\relax ACM, 2019, pp. 1027--1038.

\bibitem{Feldman11STOC}
D.~Feldman and M.~Langberg, ``A unified framework for approximating and
  clustering data,'' in \emph{STOC}, June 2011.

\bibitem{Badoiu03SODA}
M.~B\={a}doiu and K.~L. Clarkson, ``Smaller core-sets for balls,'' in
  \emph{SODA}, 2003.

\bibitem{Har-Peled02SCG}
S.~Har-Peled and K.~R. Varadarajan, ``Projective clustering in high dimensions
  using core-sets,'' in \emph{SOCG}, 2002.

\bibitem{Tsang05JMLR}
I.~W. Tsang, J.~T. Kwok, and P.-M. Cheung, ``Core vector machines: Fast {SVM}
  training on very large data sets,'' \emph{The Journal of Machine Learning
  Research}, vol.~6, pp. 363--392, December 2005.

\bibitem{Har-Peled07IJCAI}
S.~Har-Peled, D.~Roth, and D.~Zimak, ``Maximum margin coresets for active and
  noise tolerant learning,'' in \emph{IJCAI}, 2007.

\bibitem{Feldman16NIPS}
D.~Feldman, M.~Volkov, and D.~Rus, ``Dimensionality reduction of massive sparse
  datasets using coresets,'' in \emph{NIPS}, 2016.

\bibitem{Feldman14SOCG}
D.~Feldman, A.~Munteanu, and C.~Sohler, ``Smallest enclosing ball for
  probabilistic data,'' in \emph{SOCG}, 2014.

\bibitem{Clarkson10TALG}
K.~L. Clarkson, ``Coresets, sparse greedy approximation, and the {Frank-Wolfe}
  algorithm,'' \emph{ACM Transactions on Algorithms}, vol.~6, no.~4, August
  2010.

\bibitem{Langberg10SODA}
M.~Langberg and L.~J. Schulman, ``Universal $\epsilon$ approximators for
  integrals,'' in \emph{SODA}, 2010.

\bibitem{Feldman13JMIV}
D.~Feldman, M.~Feigin, and N.~Sochen, ``Learning big (image) data via coresets
  for dictionaries,'' \emph{Journal of Mathematical Imaging and Vision},
  vol.~46, no.~3, pp. 276--291, March 2013.

\bibitem{Molina18AAAI}
A.~Molina, A.~Munteanu, and K.~Kersting, ``Core dependency networks,'' in
  \emph{AAAI}, 2018.

\bibitem{Braverman16CoRR}
V.~Braverman, D.~Feldman, and H.~Lang, ``New frameworks for offline and
  streaming coreset constructions,'' \emph{CoRR}, vol. abs/1612.00889, 2016.

\bibitem{Feldman06FOCS}
D.~Feldman, A.~Fiat, and M.~Sharir, ``Coresets for weighted facilities and
  their applications,'' in \emph{FOCS}, 2006.

\bibitem{Frahling05STOC}
G.~Frahling and C.~Sohler, ``Coresets in dynamic geometric data streams,'' in
  \emph{STOC}, 2005.

\bibitem{Har-Peled04STOC}
S.~Har-Peled and S.~Mazumdar, ``On coresets for k-means and k-median
  clustering,'' in \emph{STOC}, 2004.

\bibitem{Feldman13SODA}
D.~Feldman, M.~Schmidt, and C.~Sohler, ``Turning big data into tiny data:
  Constant-size coresets for k-means, {PCA}, and projective clustering,'' in
  \emph{SODA}, 2013.

\bibitem{Boutsidis13IT}
C.~Boutsidis, P.~Drineas, and M.~Magdon-Ismail, ``Near-optimal coresets for
  least-squares regression,'' \emph{IEEE. Trans. IT}, vol.~59, no.~10, pp.
  6880--6892, October 2013.

\bibitem{Feldman11NIPS}
D.~Feldman, A.~Krause, and M.~Faulkner, ``Scalable training of mixture models
  via coresets,'' in \emph{NIPS}, 2011.

\bibitem{Aloise09ML}
D.~Aloise, A.~Deshpande, P.~Hansen, and P.~Popat, ``{NP}-hardness of
  {Euclidean} sum-of-squares clustering,'' \emph{Machine Learning}, vol.~75,
  no.~2, pp. 245--248, May 2009.

\bibitem{Megiddo84JC}
N.~Megiddo and K.~J. Supowit, ``On the complexity of some common geometric
  location problems,'' \emph{SIAM Journal of Computing}, vol.~13, no.~1, pp.
  182--196, 1984.

\bibitem{Arthur07SODA}
D.~Arthur and S.~Vassilvitskii, ``k-means++: The advantages of careful
  seeding,'' in \emph{SODA}, January 2007.

\bibitem{Cohen16STOC}
M.~Cohen, Y.~T. Lee, G.~Miller, J.~Pachocki, and A.~Sidford, ``Geometric median
  in nearly linear time,'' in \emph{STOC}, 2016.

\bibitem{virmaux2018lipschitz}
A.~Virmaux and K.~Scaman, ``Lipschitz regularity of deep neural networks:
  analysis and efficient estimation,'' in \emph{Advances in Neural Information
  Processing Systems}, 2018, pp. 3835--3844.

\bibitem{FisherIris}
\BIBentryALTinterwordspacing
R.~Fisher, ``Iris data set,''
  \url{https://archive.ics.uci.edu/ml/datasets/iris}, 1936. [Online].
  Available: \url{https://archive.ics.uci.edu/ml/datasets/iris}
\BIBentrySTDinterwordspacing

\bibitem{FacebookMetrics}
\BIBentryALTinterwordspacing
S.~Moro, P.~Rita, and B.~Vala, ``Facebook metrics data set,''
  \url{https://archive.ics.uci.edu/ml/datasets/Facebook+metrics}, 2016.
  [Online]. Available:
  \url{https://archive.ics.uci.edu/ml/datasets/Facebook+metrics}
\BIBentrySTDinterwordspacing

\bibitem{Pendigits}
\BIBentryALTinterwordspacing
E.~Alpaydin and F.~Alimoglu, ``Pen-based recognition of handwritten digits data
  set,''
  \url{https://archive.ics.uci.edu/ml/datasets/Pen-Based+Recognition+of+Handwritten+Digits},
  1996. [Online]. Available:
  \url{https://archive.ics.uci.edu/ml/datasets/Pen-Based+Recognition+of+Handwritten+Digits}
\BIBentrySTDinterwordspacing

\bibitem{MNIST}
\BIBentryALTinterwordspacing
Y.~LeCun, C.~Cortes, and C.~Burges, ``The {MNIST} database of handwritten
  digits,'' \url{http://yann.lecun.com/exdb/mnist/}, 1998. [Online]. Available:
  \url{http://yann.lecun.com/exdb/mnist/}
\BIBentrySTDinterwordspacing

\bibitem{Anguita13ESANN}
D.~Anguita, A.~Ghio, L.~Oneto, X.~Parra, and J.~L. Reyes-Ortiz, ``A public
  domain dataset for human activity recognition using smartphones,'' in
  \emph{ESANN}, 2013.

\bibitem{Har-Peled:11book}
S.~Har-Peled, \emph{Geometric Approximation Algorithms}.\hskip 1em plus 0.5em
  minus 0.4em\relax American Mathematical Society, 2011.

\end{thebibliography}

\begin{IEEEbiography}[{\includegraphics[width=1in,height=1.25in,clip,keepaspectratio]{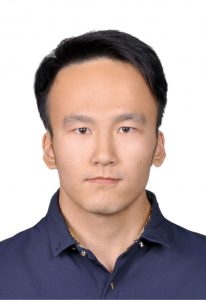}}]{Hanlin Lu}
(S'19) received the B.S. degree in mathematics from Beihang University in 2017. He is a Ph.D. candidate in Computer Science and Engineering at Pennsylvania State University, advised by Prof. Ting He. His research interest includes coreset and distributed machine learning.
\end{IEEEbiography}

\begin{IEEEbiography}[{\includegraphics[width=1in,height=1.25in,clip,keepaspectratio]{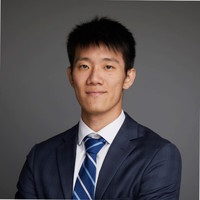}}]{Ming-Ju Li} is a Master’s student in Electrical and Computer Engineering as well as a research assistant in Network Science Research Group under the instruction of Dr. Ting He at The Pennsylvania State University. His research interests are Computer Network and Machine Learning.
\end{IEEEbiography}

\begin{IEEEbiography}[{\includegraphics[width=1in,height=1.25in,clip,keepaspectratio]{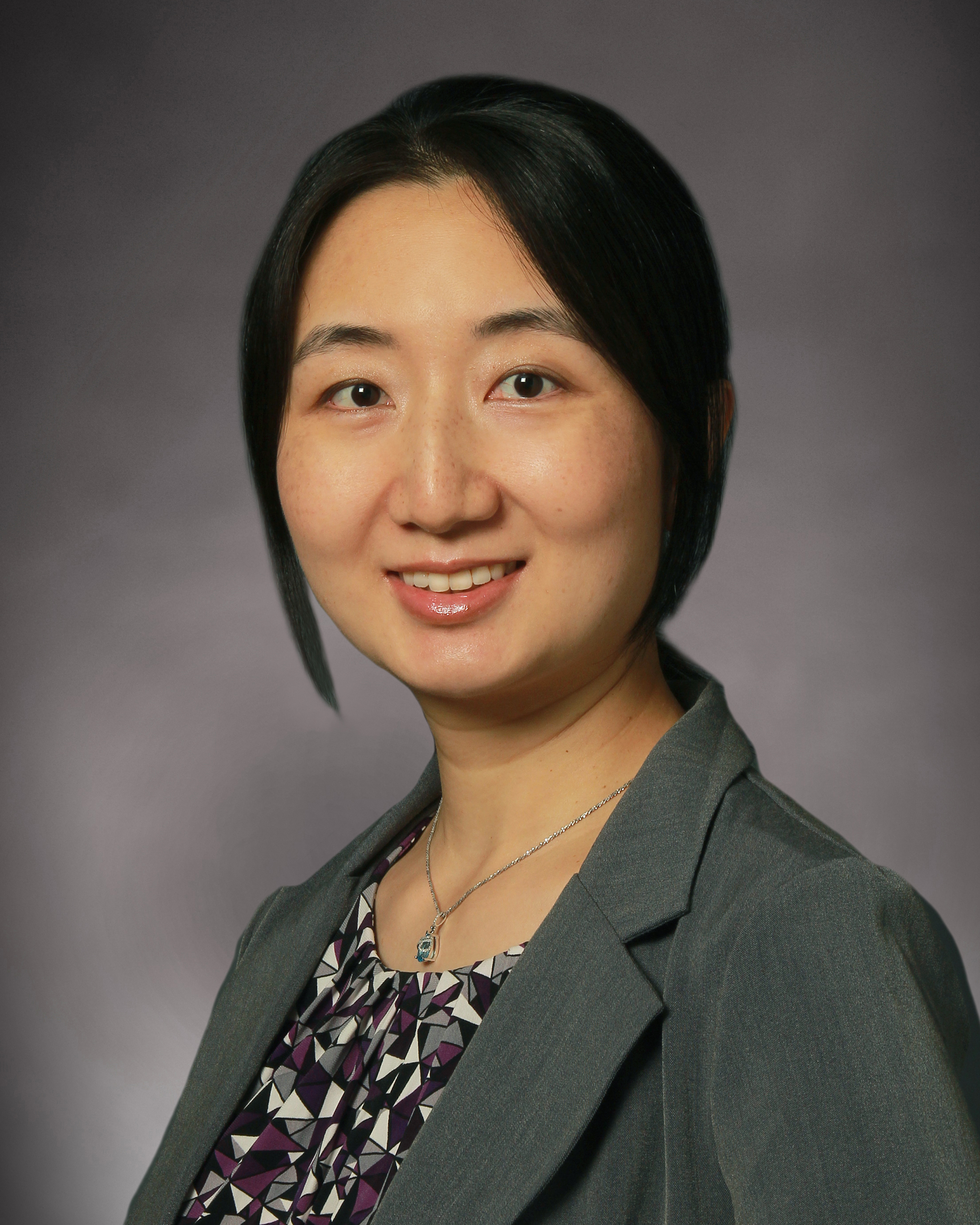}}]{Ting He}
(SM'13) received the B.S. degree in computer science from Peking University, China, in 2003 and the Ph.D. degree in electrical and computer engineering from Cornell University, Ithaca, NY, in 2007. Dr. He is an Associate Professor in the School of Electrical Engineering and Computer Science at Pennsylvania State University, University Park, PA. 
Her work is in the broad areas of computer networking, network modeling and optimization, and statistical inference. 
Dr. He is an Associate Editor for IEEE Transactions on Communications (2017-2020) and IEEE/ACM Transactions on Networking (2017-2021). 

\end{IEEEbiography}

\begin{IEEEbiography}[{\includegraphics[width=1in,height=1.25in,clip,keepaspectratio]{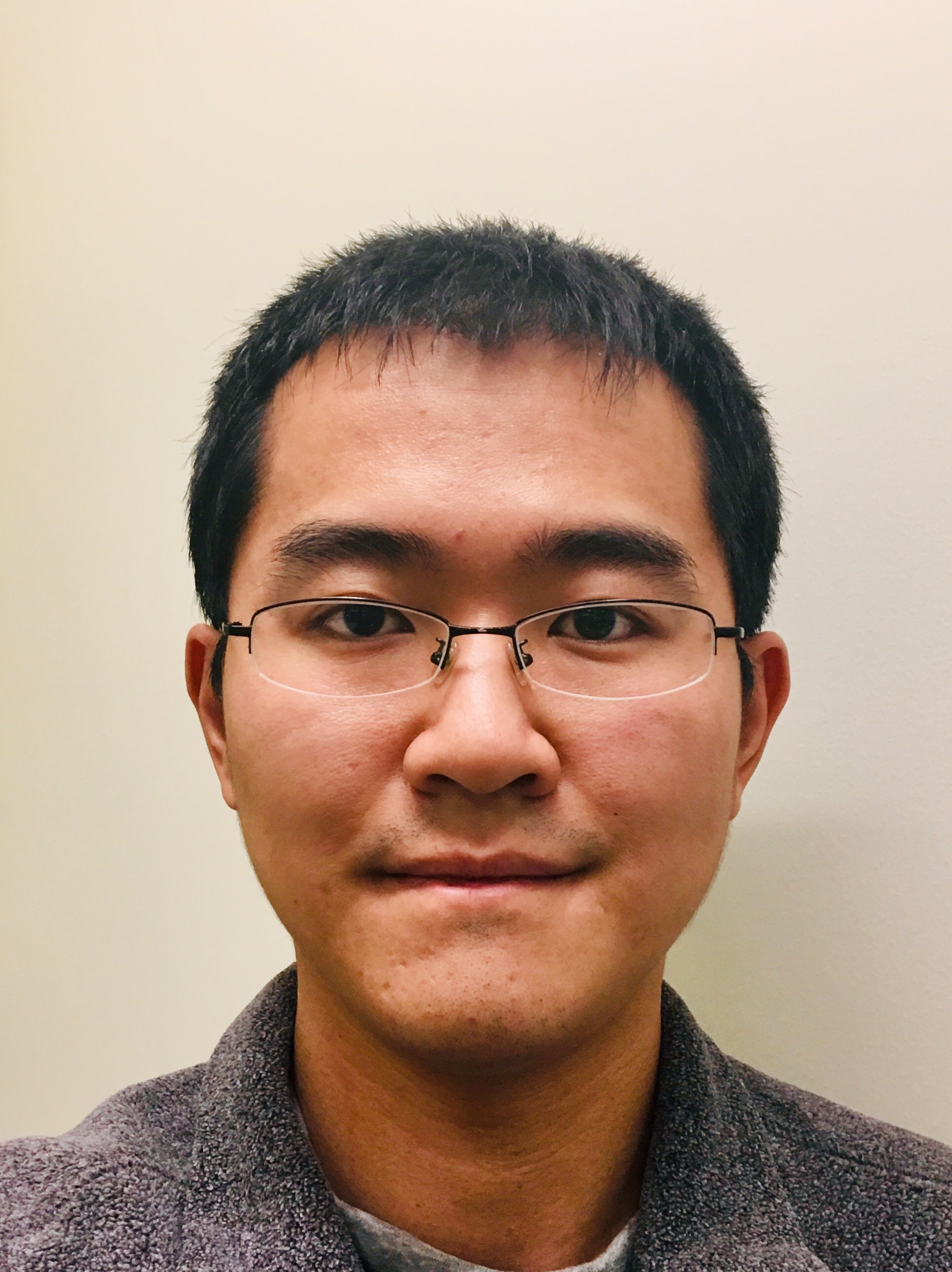}}]{Shiqiang Wang} (S’13–M’15) received his Ph.D. from the Department of Electrical and Electronic Engineering, Imperial College London, United Kingdom, in 2015. Before that, he received his master's and bachelor's degrees at Northeastern University, China, in 2011 and 2009, respectively. He is currently a Research Staff Member at IBM T. J. Watson Research Center. 
His current research focuses on theoretical and practical aspects of mobile edge computing, cloud computing, and machine learning. Dr. Wang serves as an associate editor of IEEE Access. He served as a technical program committee (TPC) member of several international conferences including IEEE ICDCS, IJCAI, WWW, IFIP Networking, IEEE GLOBECOM, IEEE ICC, and as a reviewer for a number of international journals and conferences. He received the IBM Outstanding Technical Achievement Award (OTAA) in 2019, multiple Invention Achievement Awards from IBM since 2016, Best Paper Finalist of the IEEE International Conference on Image Processing (ICIP) 2019, and Best Student Paper Award of the Network and Information Sciences International Technology Alliance (NIS-ITA) in 2015.
\end{IEEEbiography}

\begin{IEEEbiography}[{\includegraphics[width=1in,height=1.25in,clip,keepaspectratio]{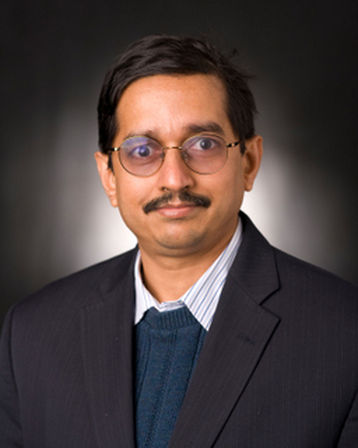}}]{Vijaykrishnan Narayanan} (F'11) is the Robert Noll Chair of Computer Science and Engineering and Electrical Engineering at The Pennsylvania State University. His research interests are in Computer Architecture, Power-Aware Computing, and Intelligent Embedded Systems. He is a Fellow of IEEE and ACM.
\end{IEEEbiography}

\begin{IEEEbiography}[{\includegraphics[width=1in,height=1.25in,clip,keepaspectratio]{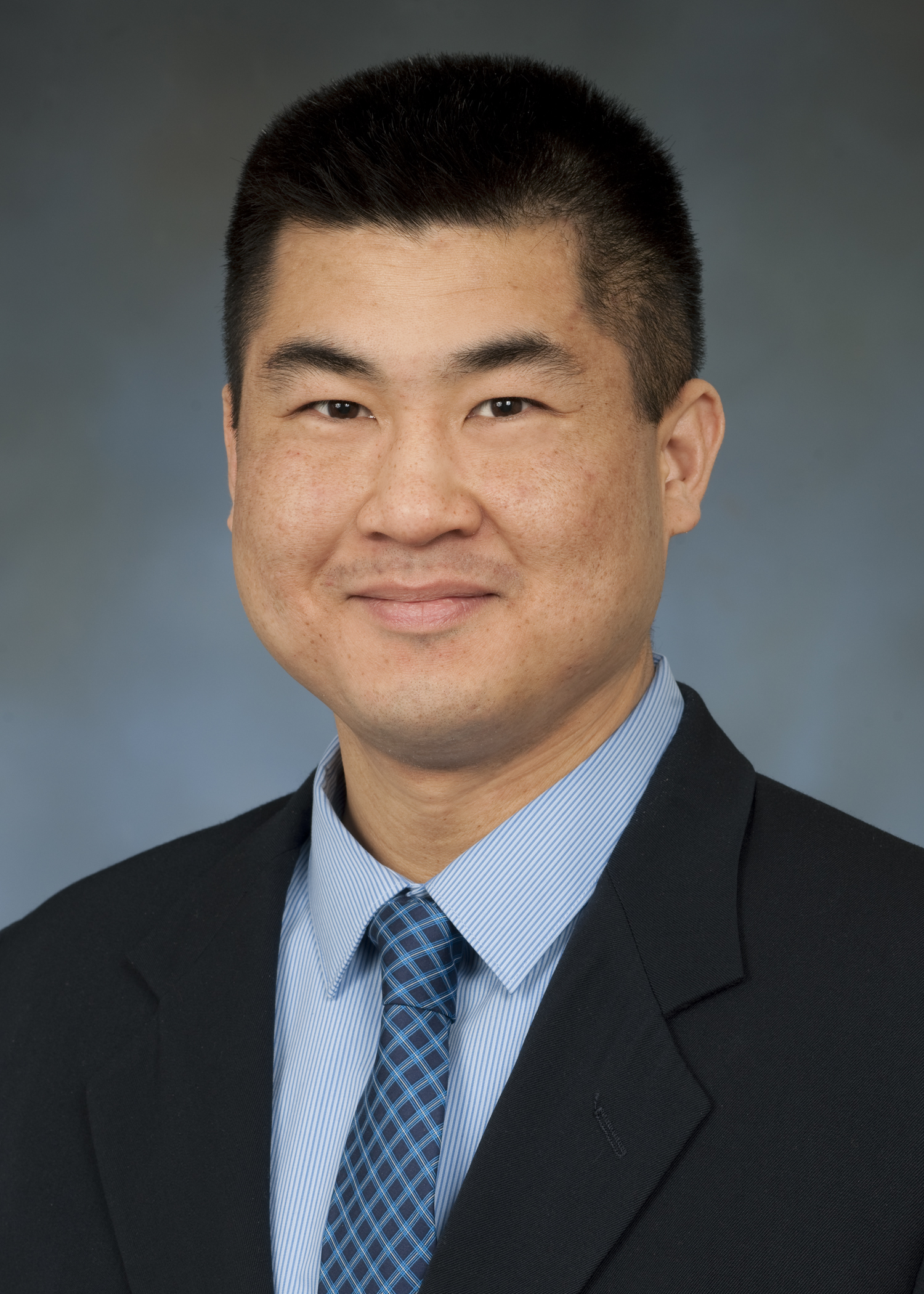}}]{Kevin Chan} (S’02–M’09–SM’18) received the B.S. degree in electrical and computer engineering and engineering and public policy from Carnegie Mellon University, Pittsburgh, PA, USA, in 2001, and the M.S. and Ph.D. degrees in electrical and computer engineering from the Georgia Institute of Technology, Atlanta, GA, USA, in 2003 and 2008, respectively. He is currently a Research Scientist with the Computational and Information Sciences Directorate, U.S. Army Combat Capabilities Development Command, Army Research Laboratory, Adelphi, MD, USA. He is actively involved in research on network science, distributed analytics, and cybersecurity. He has received multiple best paper awards and the NATO Scientific Achievement Award. He has served on the Technical Program Committee for several international conferences, including IEEE DCOSS, IEEE SECON, and IEEE MILCOM.  He is a Co-Editor of the IEEE Communications Magazine Military Communications and Networks Series.\looseness=-1
\end{IEEEbiography}

\end{document}